\documentclass{article} 
\usepackage{arxiv_template}

\usepackage{amsmath,amsfonts,bm}

\def\eqref#1{equation~\ref{#1}}

\def\1{\bm{1}}

\def\eps{{\epsilon}}

\DeclareMathAlphabet{\mathsfit}{\encodingdefault}{\sfdefault}{m}{sl}
\SetMathAlphabet{\mathsfit}{bold}{\encodingdefault}{\sfdefault}{bx}{n}

\DeclareMathOperator*{\argmin}{arg\,min}

\usepackage{microtype}

\usepackage{hyperref}
\usepackage{url}
\usepackage{graphicx}

\usepackage{pgfplots}
\usepackage{pgfplotstable}
\pgfplotsset{compat=1.3}
\usepackage{tikz}
\usetikzlibrary{arrows.meta}
\usetikzlibrary{pgfplots.groupplots}
\usepackage{ragged2e}
\definecolor{mydarkblue}{rgb}{0,0.08,0.85}
\definecolor{mylightblue}{rgb}{0.06,0.56,1.0}
\definecolor{mylightorange}{rgb}{1.0,0.62,0.12}
\definecolor{mylightred}{rgb}{0.99,0.00,0.04}
\definecolor{mygreen}{HTML}{2F9E44}
\definecolor{myred}{HTML}{E03131}
\definecolor{myblue}{HTML}{1971C2}

\usepackage{subcaption}
\usepackage{booktabs}
\usepackage{wrapfig}
\usepackage{changes}
\definecolor{myred}{HTML}{E03131}
\makeatletter
\@namedef{Changes@AuthorColor}{myred}
\colorlet{Changes@Color}{myred}
\makeatother
\colmfinalcopy

\def\l{\left}
\def\r{\right}

\title{Towards Principled Evaluations of Sparse Autoencoders for\\
Interpretability and Control}

\author{Aleksandar Makelov$^*$\\
\texttt{aleksandar.makelov@gmail.com}
\And 
Georg Lange$^*$\\
\texttt{mail@georglange.com}\\
\And
Neel Nanda  \\
\texttt{neelnanda27@gmail.com} \\
}

\usepackage{soul}
\usepackage{amsthm}
\newtheorem{theorem}{Theorem}[section]
\newtheorem{lemma}[theorem]{Lemma}

\begin{document}
\maketitle
\begin{abstract}

Disentangling model activations into meaningful features is a central problem in
interpretability. However, the absence of ground-truth for these features in
realistic scenarios makes validating recent approaches, such as sparse
dictionary learning, elusive. To address this challenge, we propose a framework for
evaluating feature dictionaries in the context of specific tasks, by comparing
them against \emph{supervised} feature dictionaries. First, we demonstrate that
supervised dictionaries achieve excellent approximation, control, and
interpretability of model computations on the task. Second, we use the
supervised dictionaries to develop and contextualize evaluations of unsupervised
dictionaries along the same three axes.

We apply this framework to the indirect object identification (IOI) task using
GPT-2 Small, with sparse autoencoders (SAEs) trained on either the IOI or
OpenWebText datasets. We find that these SAEs capture interpretable features for
the IOI task, but they are less successful than supervised features in
controlling the model. Finally, we observe two qualitative phenomena in SAE
training: feature occlusion (where a causally relevant concept is robustly
overshadowed by even slightly higher-magnitude ones in the learned features),
and feature over-splitting (where binary features split into many smaller,
less interpretable features). We hope that our framework will provide a useful
step towards more objective and grounded evaluations of sparse dictionary
learning methods.

\end{abstract}
\begingroup
\renewcommand{\thefootnote}{}
\footnotetext{$^*$: Joint contribution. Correspondence to: \texttt{aleksandar.makelov@gmail.com}.}
\addtocounter{footnote}{-0}
\endgroup

\section{Introduction}
\label{sec:intro}

While large language models (LLMs) have demonstrated impressive
\citep{vaswani2017attention, bert, radford2019language, brown2020language, gpt4}
results, the mechanisms behind their successes and failures
largely remain a mystery \citep{olah2023interpretability}. A central problem in
this area is how to \emph{disentangle} internal model representations into
meaningful concepts or \emph{features}. If successful at scale, this research
could provide significant scientific and practical value, enabling enhanced
model robustness, controllability, interpretability, and debugging
\citep{gandelsman2023interpreting,nanda2023emergent,marks2024sparse}.

A leading hypothesis for how LLMs represent and use features is the \emph{linear
representation hypothesis}
\citep{Mikolov2013LinguisticRI,grand2018semantic,li2021implicit,abdou2021can,nanda2023emergent}.
A strong version of this hypothesis posits that individual activations of a model can be
decomposed into sparse linear combinations of features from a large,
shared \emph{feature dictionary}. Recently, a series of works has proposed
applying the (unsupervised) \emph{sparse autoencoder} (SAE) framework to find
such dictionaries
\citep{Olshausen1997SparseCW,Faruqui2015SparseOW,goh2016decoding,arora2018linear,Yun2021TransformerVV,cunningham2023sparse,bricken2023monosemanticity}.

\begin{figure}[ht]
    \centering
    \includegraphics[width=\textwidth]{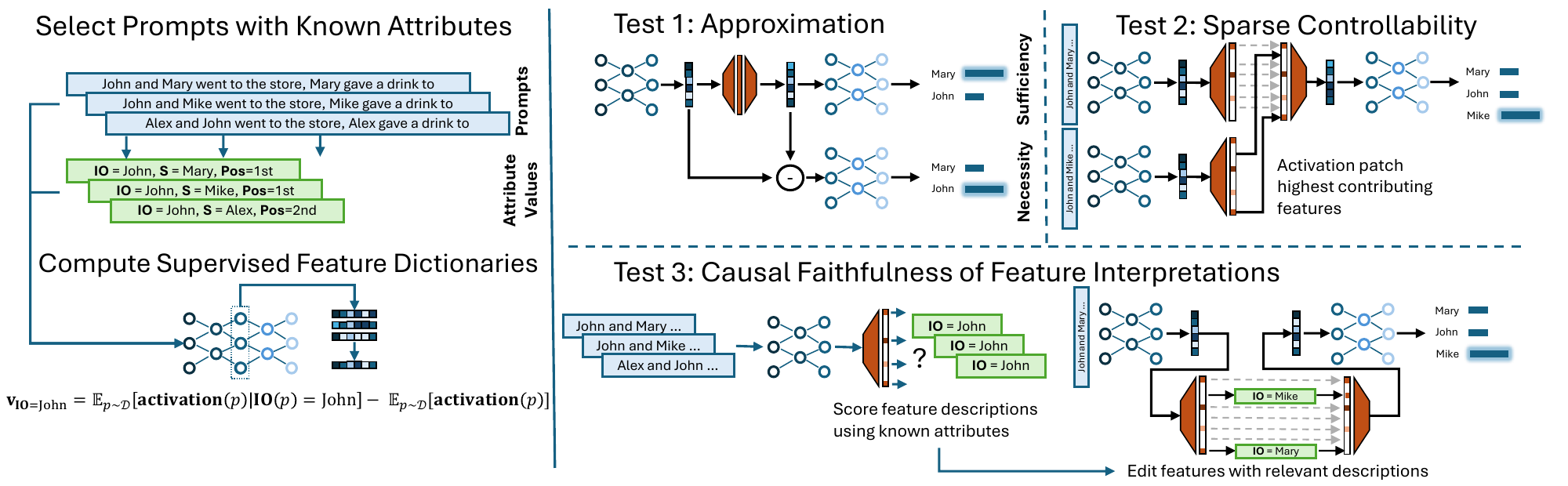}
    \caption{Overview of our evaluation pipeline. We begin by selecting a
    specific model capability and then disentangling model activations into
    capability-relevant features using supervision. Then, we evaluate a given
    feature dictionary w.r.t.\ this capability, using the supervised features as
    a benchmark. We test the extent to which (1) the feature dictionary's
    reconstructions of the activations are necessary and sufficient for the
    capability, (2) the features can be used to edit capability-relevant
    information in internal model representations (agnostic of feature
    interpretations), and (3) the features can be interpreted w.r.t. the
    capability in a manner consistent with their causal role.}
    \label{fig:graphical-abstract}
\end{figure}

While there are promising initial results \citep{bricken2023monosemanticity},
this research area faces a key obstacle: we cannot directly evaluate the
usefulness of features learned by an SAE, as we do not know the hypothetical
`true' features to begin with; indeed, finding them is the reason we use
SAEs in the first place. This has two consequences: (1) the training objective
of an SAE -- balancing $\ell_2$ reconstruction with an $\ell_1$ penalty on
feature coefficients -- is only a substitute for the true goal of finding meaningful features, and (2)
the metrics used to evaluate SAEs are indirect, relying on proxies for the
features, success in toy models, or non-trivial assumptions on SAE learning
\citep{elhage2022toy,bricken2023monosemanticity,sharkey2023taking}. This hampers
progress in the field, as there is no simple, objective and direct way to
compare trained SAEs. In this paper, we take steps towards addressing this
challenge. Our contributions are outlined as follows:
\begin{itemize}
\item We propose a principled method for computing sparse feature dictionaries to
decompose language model activations on realistic tasks, using supervision from
attributes describing the prompts.
\item We apply this method to the IOI task, demonstrating that these dictionaries exhibit
three desirable properties in the task's context: (1) sufficiency and necessity of
activation reconstructions, (2) sparse controllability of model behavior via
feature editing, and (3) interpretability of the features consistent with their
causal role\footnote{
As a by-product, our supervised feature dictionaries also demonstrate that
task activations can be usefully disentangled in a way that adheres to
the linear representation hypothesis and exhibits superposition. To the best of our knowledge, this is the
first time such a disentanglement has been achieved in a realistic LLM task.}.
\item We use these feature dictionaries to design and contextualize evaluations
of \emph{unsupervised} feature dictionaries along the same three axes.
Importantly, we aim to develop evaluations that are agnostic to whether the
unsupervised dictionaries use the same concepts as the supervised
ones. 
\item We apply this methodology to feature dictionaries learned by SAEs trained on
either the IOI dataset (task-specific SAEs) or the LLM's pre-training dataset
(full-distribution SAEs). We find that both types of SAEs find 
interpretable features for the task, but task-specific SAEs allow us to edit
attributes by changing fewer features compared to full-distribution SAEs.
However, both kinds of SAEs fall short of supervised dictionaries.  
\item Finally, we dive deeper into some qualitative phenomena observed in
task-specific SAEs: feature occlusion (where SAEs have a tendency to learn only
the higher-magnitude of two task-relevant attributes, even when the attributes
have slightly different magnitudes), and feature over-splitting (where SAEs tend
to learn $\gg2$ features without a clear interpretation for binary concepts). We
reproduce both phenomena in simple toy models, suggesting they may generalize
beyond our specific setting.
\end{itemize}

Our results underscore the need for more principled training and evaluation
methods in this active area and suggest that supervised feature dictionaries
can be a valuable tool for automating aspects of this process.

\section{Preliminaries}
\label{sec:prelims}

\textbf{The linear representation hypothesis and sparse autoencoders.} A
central hypothesis in interpretability is the \emph{linear representation
hypothesis}. A strong variant of this hypothesis posits that model activations can be decomposed into meaningful features
using a \emph{sparse feature dictionary}: given a location in the model (e.g., an
attention head output), there exists a set of vectors
$\{\mathbf{u}_i\}_{i=1}^{m}$ such that each activation $\mathbf{a}$ at this location can
be approximated as a sparse linear combination of the $\mathbf{u}_j$ with
non-negative coefficients. In particular, recent work suggests that
$n$-dimensional activations $\mathbf{a}\in \mathbb{R}^n$ may be best described by $m\gg n$ such features in
\emph{superposition} \citep{elhage2022superposition,gurnee2023finding}.
Recently, SAEs have been proposed as a way to disentangle these features. Following the
setup of \citet{bricken2023monosemanticity} here and in the rest of this work, a sparse autoencoder (SAE) is an
unsupervised model which learns to reconstruct activations
$\mathbf{a}\in\mathbb{R}^n$ as a weighted sum of $m$ features
with non-negative weights. Specifically, the autoencoder computes a hidden
representation
\begin{align*}
    \mathbf{f} = \operatorname{ReLU}\l(W_{enc} \l(\mathbf{a}-\mathbf{b}_{dec}\r) +
    \mathbf{b}_{enc}\r)
\end{align*}
and a reconstruction
\begin{equation}
\label{eq:sae-reconstruction}
    \widehat{\mathbf{a}} = W_{dec} \mathbf{f} + \mathbf{b}_{dec} = \sum_{j=1}^{m} \mathbf{f}_j (W_{dec})_{:,j} + \mathbf{b}_{dec}
\end{equation}
where $W_{enc}\in \mathbb{R}^{m\times n}$, $W_{dec}\in \mathbb{R}^{n\times m},
\mathbf{b}_{dec} \in \mathbb{R}^n, \mathbf{b}_{enc}\in \mathbb{R}^m$ are learned
parameters. The rows of $W_{enc}$ are the \emph{encoder directions}, and the columns
of $W_{dec}$ are the \emph{decoder directions}. 
Similarly, $\mathbf{b}_{enc}$ is the encoder bias and $\mathbf{b}_{dec}$ is the decoder bias. 
The decoder directions determine the features we decompose the activations into,
while the encoder directions compute the coefficients of these features for a 
given activation. The decoder directions are constrained to have unit norm:
$\l\|(W_{dec})_{:,i}\r\|_2=1$. The training objective over examples
$\{\mathbf{a}^{(k)}\}_{k=1}^N$ is the sum of the MSE between the
activations $\mathbf{a}^{(k)}$ and their reconstructions
$\widehat{\mathbf{a}}^{(k)}$, and the $\ell_1$ regularization term
$\lambda\sum_{k=1}^{N} \left\|\mathbf{f}^{(k)}\right\|_1$, where $\lambda$ is
the $\ell_1$ regularization coefficient.

\textbf{The IOI task.}
In \citet{wang2022interpretability}, the authors analyze how the decoder-only
transformer language model GPT-2 Small \citep{radford2019language} performs the
Indirect Object Identification (IOI) task. 
In this task, the model is required to complete sentences of the form `When Mary
and John went to the store, John gave a book to' (with the intended
completion in this case being ` Mary'). We refer to the repeated name (John) as
\textbf{S} (the subject) and the non-repeated name (Mary) as \textbf{IO} (the
indirect object). For each choice of the \textbf{IO} and \textbf{S} names, there
are two patterns the sentence can have: one where the \textbf{IO} name comes
first (we call these `ABB examples'), and one where it comes second (we call
these `BAB examples'). We refer to this binary attribute as the \textbf{Pos}
attribute (short for position). Additional details on the data distribution,
model and task performance are given in Appendix \ref{app:ioi-dataset-details}.

\citet{wang2022interpretability} discover several classes of attention heads in GPT2-Small that
collectively form the \emph{IOI circuit} solving the IOI task. Specifically,
\citet{wang2022interpretability} argue that the circuit implements the
algorithm: 
\begin{enumerate}
\item detect the (i) position in the sentence and (ii) identity of the repeated
name in the sentence (i.e., the \textbf{S} name). This information is computed
and moved by \emph{duplicate token}/\emph{induction} and \emph{S-Inhibition}
heads;
\item based on the two signals (i) and (ii), exclude this name from the
attention of the \emph{name mover heads}, so that they copy the remaining name
(i.e., the \textbf{IO} name) to the output.
\end{enumerate}
We refer the reader to Appendix \ref{app:ioi-circuit-details} and Appendix
Figure \ref{fig:ioi-circuit} for more details on the IOI circuit.

\emph{The logit difference metric.}
To discover the circuit, \citet{wang2022interpretability} used the logit
difference: the difference in log-probabilities assigned by the model to the
\textbf{IO} and \textbf{S} names. This metric is more sensitive than accuracy,
which makes possible the detection of individual model components with a
consistent but non-pivotal role in the task. Accordingly, we also use the logit
difference throughout this work to evaluate the causal effect of fine-grained
model interventions.

\section{Methods for Evaluating Feature Dictionaries}
\label{section:methods}

\subsection{Overview and Motivation}
\label{subsection:}
In this section, we describe and motivate our methodology for evaluating sparse
feature dictionaries in the context of a specific task an LLM can perform.
Throughout, let $\mathcal{D}$ be a distribution over input prompts for the task.

\textbf{Step 1: Parametrize inputs via (task-relevant) attributes.} First, we
choose attributes $a_i: \operatorname{support}(\mathcal{D})\to S_i$ taking values in some finite sets
$S_i$.  For example, in the IOI task we will focus on the attributes
\textbf{IO}, \textbf{S}, and \textbf{Pos} described in Section
\ref{sec:prelims}, with \textbf{IO} and \textbf{S} taking values in the set of
names in our dataset, and \textbf{Pos} taking values in the set $\{\text{ABB},
\text{BAB}\}$.

\textbf{Step 2: Compute high-quality supervised dictionaries using the attributes.}
The second step is to learn and validate \emph{supervised} dictionaries computed
using the attribute labels $a_i(p),p\sim \mathcal{D}$ (methodology described in
Section \ref{section:manual-sae}). In particular, not all attribute
configurations will result in dictionaries with good approximation, control and
interpretability for the task\footnote{
To get high-quality supervised dictionaries, we should choose attributes that
completely capture the task-relevant information in the prompt, and are moreover
compatible with the intermediate states of the model's computation on the task.
Validating that our attributes satisfy this is a non-trivial but key
prerequisite for our evaluation to be meaningful. We do this for the IOI task in
Section \ref{section:manual-sae}, where we show the supervised dictionaries
score highly on all tests described in the current section.
}.

\textbf{Step 3: Evaluate (unsupervised) dictionaries using the supervised
dictionaries as a benchmark.} Finally, we evaluate the usefulness of a given
feature dictionary for the chosen task, using the supervised features as a
reference point in various ways (described in the next subsections). The
supervised dictionaries are a necessary part of this evaluation because they
allow us to \emph{contextualize} performance differences between feature
dictionaries for the chosen task. 

\textbf{Why are supervised dictionaries a meaningful benchmark?} On the one
hand, supervised dictionaries' features are hard-coded to
correspond to the `human-salient' attributes $a_i$. So, our evaluations risk
overlooking useful feature dictionaries that do not align with this human-chosen
ontology. On the other hand, the saliency of the attributes is precisely what
makes them desirable targets for controlling and interpreting the model; for
example, in the IOI task it is natural to want to change e.g. the representation
of the \textbf{IO} name in the model's computation. 
Thus, our evaluations aim to balance two opposing goals: establish properties
useful from a human standpoint, while still being agnostic to the exact concepts
represented in the unsupervised dictionary.

\subsection{Test 1: Sufficiency and Necessity of Dictionary Reconstructions for the Task}
\label{subsection:test-1-sufficiency-necessity}

\begin{figure}[ht]
    \centering
    \includegraphics[width=0.48\textwidth]{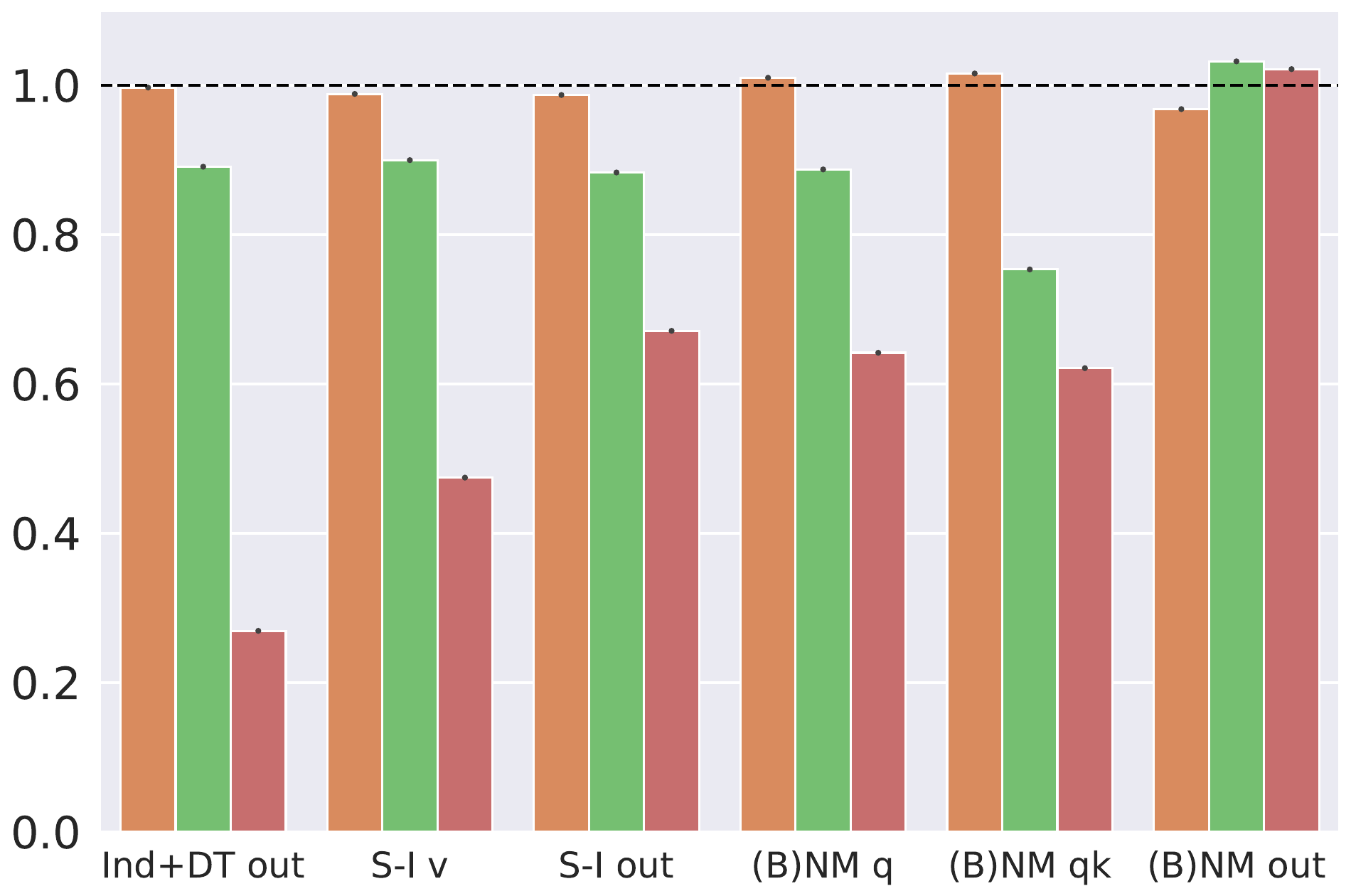}
    \hfill
    \includegraphics[width=0.48\textwidth]{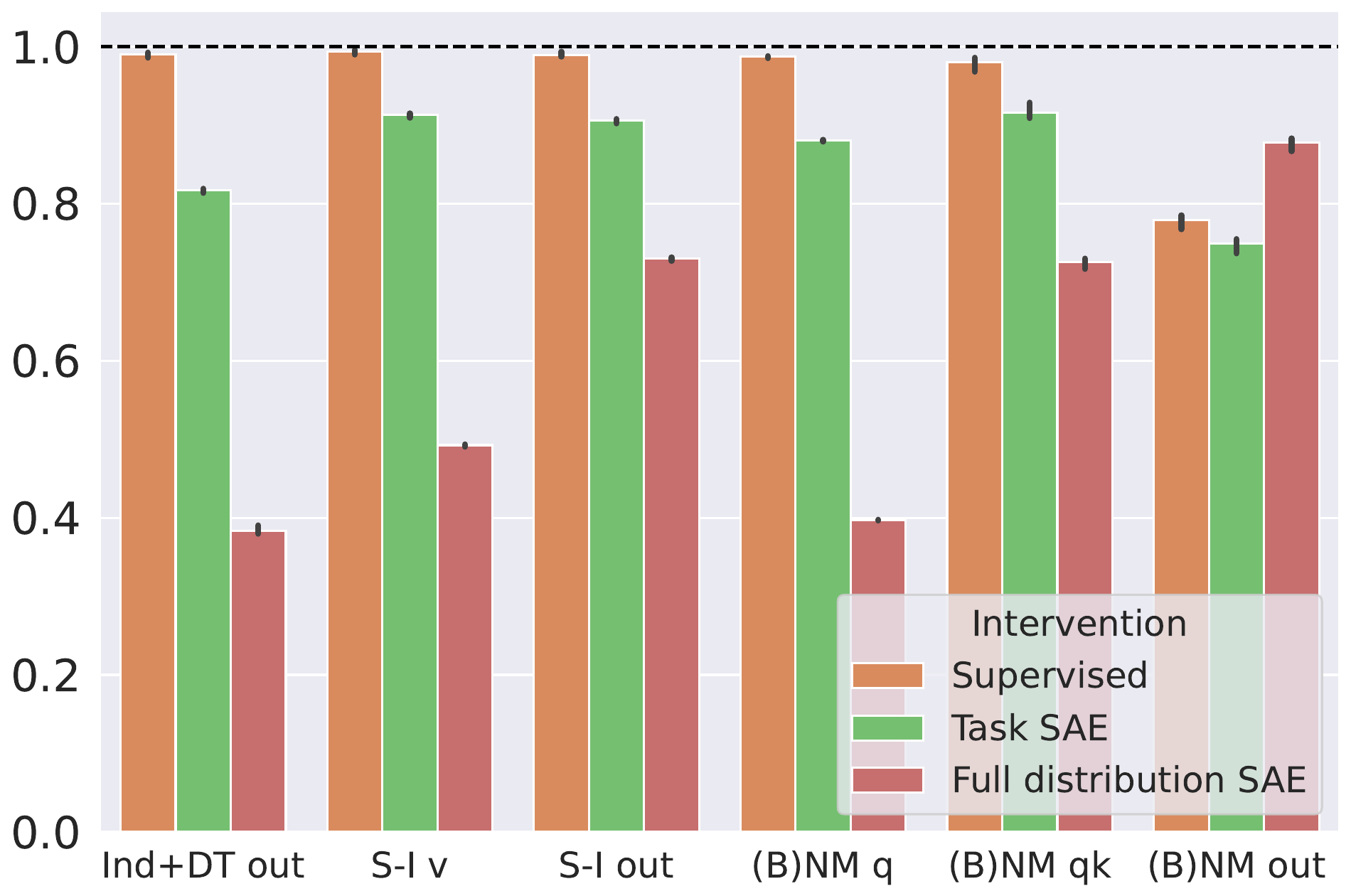}
    \caption{
        Sufficiency (left) and necessity (right) evaluations of reconstructions
        of cross-sections of the IOI circuit computed using supervised feature
        dictionaries, task- and full-distribution SAEs. 
        \textbf{Left}: average logit difference when replacing activations in
        cross-sections of the IOI circuit with their reconstructions, normalized
        by the average logit difference over the data distribution in the
        absence of intervention (a $y$-axis value of $1$ is best).
        \textbf{Right}: drop in logit difference when deleting reconstructions,
        normalized by the respective drop when performing mean-ablation, and
        linearly rescaled so that values close to 1 are best. See Appendix 
        \ref{app:ioi-supervised-details} for details.}
    \label{fig:logitdiff-faithfulness-completeness}
\end{figure}

Our first test checks if activation reconstructions $\widehat{\mathbf{a}}$ \emph{as
a whole} are sufficient and necessary for the model to perform the
task\footnote{This test is analogous to the faithfulness/completeness test from
\citet{wang2022interpretability}; the sufficiency test is also widely used in
the literature to evaluate feature dictionaries.}.  
To evaluate \textbf{sufficiency}, we intervene by replacing
internal activations $\mathbf{a}$ with their reconstructions
$\widehat{\mathbf{a}}$, and measure the drop in performance on the task. To evaluate
\textbf{necessity}, we intervene by replacing activations
$\mathbf{a}$ with $\mathbb{E}_{\mathcal{D}}\l[\mathbf{a}\r] + \l(\mathbf{a} -
\widehat{\mathbf{a}}\r)$, and compare the resulting drop in performance to the
effect of replacing activations with $\mathbb{E}_{\mathcal{D}}\l[\mathbf{a}\r]$
alone (also known as \emph{mean ablation}; see Appendix
\ref{app:methods-details} for motivation)
\footnote{
This test has very few assumptions: it is independent of the attributes we have
chosen to describe inputs with, and of any interpretation assigned to the
features. However, it only measures the quality of the dictionary as a whole,
and not how effectively and sparsely it disentangles any hypothetical concepts
in the activations. We include this test because poor performance indicates a
fundamental failure of the feature dictionary relative to the task.
}. 

\subsection{Test 2: Sparse Controllability of Attributes}
\label{subsection:test-2-sparse-controllability}
With this test, we want to measure the degree to which the feature dictionary
can be used to control the model's behavior on the task by editing intermediate
representations of attribute values.
To use a feature dictionary for editing, we write an activation as a linear combination of dictionary vectors, and try to remove/add dictionary elements to the combination in order to achieve the desired change.
To evaluate such an intervention, we consider both (1) the number of features in the dictionary that need to be changed
in order to achieve a given change in outputs (editing is trivial if e.g.\
activations are dense sums of random features and we change many of them), and
(2) the geometric magnitude of the edit (editing is trivial if you are willing
to throw away the entire activation vector and replace it with the target). Jump ahead to Figure \ref{fig:pareto-combined} to see what trade-offs for these parameters look like in our experiments.

\textbf{Being agnostic to feature interpretations.}
Importantly, we require this test to be agnostic to any human interpretation of
the feature dictionary. This is valuable for several reasons: (1) even if our
attributes are compatible with model computations, the feature dictionary may
not be directly interpretable in terms of these
attributes (see Appendix \ref{app:possible-decompositions} for hypothetical
examples in the IOI task); (2) human-generated feature interpretations, even
with meticulous methodology, may be subjective; (3) sparse
control may be possible to a useful degree even if the features are inherently
non-interpretable from a human perspective. 
Thus, this test should evaluate the degree to which the feature dictionary
sparsely disentangles the attributes, but is independent of whether the features
are interpretable in terms of the attributes (or at all)
\footnote{Note that, to meaningfully conclude disentanglement of attributes from
this test, the task must be one where at least some of the activations represent
multiple attributes simultaneously. This is true in the IOI task, where many
circuit locations represent the \textbf{S} and \textbf{Pos} information
simultaneously. We also find that one particular location, the queries of the L10H0 name mover head, represent all three attributes \textbf{IO}, \textbf{S} and \textbf{Pos}.}.

\textbf{Implementation design space.}
To achieve independence of interpretations, we frame the problem as a
combinatorial optimization problem over the feature dictionary: given an
activation $\mathbf{a}_s$ for prompt $p_s$ and a \textbf{counterfactual} activation
$\mathbf{a}_t$ for a prompt $p_t$ which differs from $p_s$ only in the values of
the attributes we wish to edit, we can optimize over subsets of features active
in $\widehat{\mathbf{a}_s}$ to subtract from $\mathbf{a}_s$, and subsets of
features active in $\widehat{\mathbf{a}_t}$ to add to $\mathbf{a}_s$. There are
multiple optimization objectives possible, such as making the edited activation
similar to $\mathbf{a}_t$ geometrically, or making model behavior on the edited
activation similar to model behavior on $\mathbf{a}_t$.
The optimization may be constrained by the number of features changed and/or the
magnitude of edits
\footnote{
There are many possible ways to instantiate such an optimization; in Section
\ref{section:sae-evaluation} we present one such way based on a greedy algorithm
minimizing $\ell_2$ distance in activation space. Further methodological notes, 
such as how we implement edits in multiple model components at once, are given
in Appendix \ref{app:methods-details}.
}.
Importantly, this optimization proceeds on a
\textbf{per-prompt} basis, which means that it can select different features to
edit depending on the prompt being edited. This means it will not disadvantage
dictionaries that do not dedicate a uniform w.r.t.\ all prompts set of features
to each attribute (which would be a major assumption to impose).

\textbf{Evaluation.} To measure how well an edit changed model behavior, we
compare against the `ground truth' change in behavior that would be achieved by
intervening on the model to replace $\mathbf{a}_s$ with $\mathbf{a}_t$ directly.
To contextualize the magnitude of the edit, we can compare the contribution of
the changed features to the reconstruction against the analogous quantity for
our supervised feature dictionary (see Appendix \ref{app:methods-details} for
details), as we do later in Section \ref{section:manual-sae}.

\subsection{Test 3: Interpretability}
\label{subsection:test-3-interpretability}
With our final test, we want to assess the degree
to which the feature dictionary can be interpreted in terms of the task-relevant
attributes and the features are causally relevant for the model's
behavior in a way consistent with their interpretations. While failing this test
does not necessarily mean that the feature dictionary is not useful (e.g., it
could still be useful for control, or interpretable with respect to different
attributes), passing it would increase our confidence in the feature
dictionary's utility and our understanding of the model's computation. More
broadly, passing this test on a wide range of important tasks would represent a
qualitative improvement over control alone in applications such as auditing,
debugging and verification.

\textbf{Assigning interpretability scores.} Given a feature $\mathbf{u}_j$
from the dictionary, its \emph{active set} $F\subset
\operatorname{supp}\l(\mathcal{D}\r)$ is the subset of the support of
$\mathcal{D}$ where $\mathbf{f}_j >0$. Given
a binary property of inputs $P:\mathcal{D}\to\{0,1\}$, following
\citet{bricken2023monosemanticity}, we say that $P$ is a good interpretation of
the feature if $F$ has high precision and recall relative to $P$
\footnote{Clearly, the best interpretation of $F$ according to these metrics
alone is the indicator $\mathbf{1}_F$ itself; for $P$ to be a useful
interpretation from a human perspective, it must in addition be `simple' enough,
a property harder to formalize.}.
We combine the precision and recall metrics into a single number using the $F_1$
score. 

For example, each attribute defines $\left|S_i\right|$ binary properties
$\mathbf{1}_{a_i(p)=v}$ for $v\in S_i$ that we can use to try to interpret the
feature dictionary. When considering a set of binary properties $\{P_j\}_{j\in
J}$ as possible interpretations for a feature, we pick the one with the highest
$F_1$ score and assign this as the interpretation of the feature. Other
interpretability measures are possible; for a discussion of the limitations of
the $F_1$ score, see Appendix \ref{app:methods-details}.

\textbf{Which interpretations to consider?} As with controllability, we should
not assume that the features correspond 1-to-1 with the attributes we have
chosen. Thus, evaluating the $F_1$ scores only for the binary properties
$\mathbf{1}_{a_i(p)=v}$ corresponding to a single value for a single attribute
may be subjective. On the other hand, when the chosen attributes result in
a good supervised dictionary (required by Step 2 of our method), the attribute
values will be sufficient to deduce the causally important information in model
activations. This suggests that causally-relevant features in an arbitrary
feature dictionary will correspond to particular subsets of the cartesian
product $\prod_i S_i$ of the values the attributes can take.

This motivates us to look for interpretations that are expressible as
intersections and unions of the indicator sets $\mathbf{1}_{a_i(p)=v}$ of the
chosen attributes. We use heuristics to navigate this search space. We present
one possible implementation of this in Section \ref{section:sae-evaluation} for
the IOI task.  As we will see, this choice, while possibly arbitrary/subjective,
is to some extent validated empirically: we find that many features can be
interpreted in this way for a significantly high $F_1$ score threshold. We further discover some
interpretable structure in the ways attribute values group together in unions.

\textbf{Causal evaluation of interpretations.} Finally, while our previous
interpretability methods look for correlations with properties of the input
prompts, we also want to know if the interpretations we assign to features in
the dictionary are consistent with their causal role in the model's computation.
There are two increasingly demanding ways to check this that mirror our first
and second tests (Subsections \ref{subsection:test-1-sufficiency-necessity} and \ref{subsection:test-2-sparse-controllability}):
\begin{itemize}
\item \textbf{interpretation-aware sufficiency/necessity of reconstructions}:
similar to our first test (Subsection
\ref{subsection:test-1-sufficiency-necessity}), we can (1) subtract
non-interpretable features from activations and see if the model is still able
to perform the task; (2) subtract highly-interpretable features and see if
the model's performance degrades to the same extent as with mean ablation. This
evaluates whether our interpretability method as a whole flags the important
features for the task;
\item \textbf{interpretation-aware sparse controllability}: like our second test
(Subsection \ref{subsection:test-2-sparse-controllability}), but explicitly
using highly-interpretable features (with respect to a given attribute) to
remove/add in order to edit an activation. This evaluates whether
interpretations that specifically relate features to attributes find the
attribute-relevant features.
\end{itemize}

\section{Computing and Validating Supervised Feature Dictionaries}
\label{section:manual-sae}

We first present our methods and results for computing \emph{supervised} feature
dictionaries, in which features correspond 1-to-1 with the possible values of
attributes we have chosen. This is a key prerequisite for our evaluation of SAEs
later on, as it (1) verifies that the attributes we have chosen are compatible
with the model's internal states; (2) demonstrates the existence of high-quality
sparse feature dictionaries for the task; and (3) provides an `ideal'
reference for SAE evaluation.
Specifically, we will establish that feature dictionaries for all locations in
the IOI circuit exist with the following properties:
\begin{itemize}
\item they approximate activations well using as few as 3 active features per
activation;
\item any of the \textbf{IO}, \textbf{S}, and \textbf{Pos} attributes can be
edited precisely by replacing only 1 active feature with a different one;
\item the features in the dictionary are by construction interpretable w.r.t.\
the attributes, and furthermore, interactions in the IOI circuit exhibit 
sparsity when decomposed into feature-to-feature interactions.
\end{itemize}

\subsection{Computing Supervised Feature Dictionaries}
\label{subsection:}
\textbf{Algorithms.} Motivated by the linear representation hypothesis, we
conjecture that given activations $\mathbf{a}\l(p\r)\in \mathbb{R}^d$ of a given
model component (e.g., outputs of an attention head) for prompts $p\sim
\mathcal{D}$, there exists a choice of attributes $\{a_i:\mathcal{D}\to
S_i\}_{i\in I}$ such that
\begin{equation}
    \label{eq:manual-sparse-codes}
    \mathbf{a}(p) \approx \mathbb{E}_{p\sim \mathcal{D}}\l[\mathbf{a}(p)\r] + \sum_{i\in I}\mathbf{u}_{a_i(\cdot)=v} := \widehat{\mathbf{a}}
\end{equation}
where $\widehat{\mathbf{a}}$ is the \emph{reconstruction} of $\mathbf{a}$, and
$\mathbf{u}_{a_i(\cdot)=v}\in \mathbb{R}^d$ is a feature corresponding to the
$i$-th attribute having value $v\in S_i$
\footnote{This formulation is less expressive
than SAE reconstructions, as it effectively requires a given feature to always
appear with the same coefficient in reconstructions; we discuss this choice in
Appendix \ref{app:ioi-supervised-details}.}. 

Given a dataset of prompts $\{p_k\}_{k=1}^N$ with associated activations
$\{\mathbf{a}\l(p_k\r)\}_{k=1}^N$, how should we compute `good' values for the
vectors $\mathbf{u}_{a_i(\cdot)=v}$?  While we considered several ways to do so,
our best method for the IOI task is simple: we average all activations for
prompts for which $a_i(p)=v$. Formally,
\begin{align*}
    \mathbf{u}_{a_i(\cdot)=v} := \frac{1}{\left|\{k: a_i(p_k)=v\}\right|} \sum_{k: a_i(p_k)=v}^{} \mathbf{a}\l(p_k\r) - \overline{\mathbf{a}}
\end{align*}
where $\overline{\mathbf{a}} = \frac{1}{N}\sum_{k=1}^{N} \mathbf{a}\l(p_k\r)$ is
the empirical mean activation. We refer to these dictionaries as \textbf{mean
feature dictionaries}. One can prove that, in the limit of infinite data, the
mean features for an attribute not linearly detectable in the activations will
converge to zero (Appendix \ref{app:mean-codes-properties}).

The main alternative method we considered was \textbf{MSE feature dictionaries},
which use a least-squares linear regression to predict the activations from the
attribute values. We note that mean feature dictionaries work well in our
setting because the attributes we choose in the IOI task are
probabilistically independent in the IOI distribution; we recommend using MSE
dictionaries in general (and see Appendix \ref{app:mse-math} for more details on
MSE dictionaries and comparisons to mean dictionaries).

\textbf{Choosing attributes for the IOI task.} Not every set of attributes
will result in a good approximation of the model's internal activations; in
fact, we find that the choice of attributes is a crucial modeling decision.
Recall that, according to \citet{wang2022interpretability}, each IOI prompt $p$
is described by three properties influencing how $p$ is processed in the IOI
circuit: the subject (\textbf{S}) and indirect object (\textbf{IO}) names, and
their relative position (\textbf{Pos}). 
Motivated by this,
we chose the attributes \textbf{S}, \textbf{IO}, and \textbf{Pos} to describe
each prompt. We experimented with other choices of attributes, but did not find
them to be more successful in our tests (see Appendix
\ref{app:alternative-ioi-parametrizations} for details\footnote{ In particular,
we found that another set of attributes, though more expressive in principle,
learns to approximate the features we get using the \textbf{S}, \textbf{IO} and
\textbf{Pos} attributes through a change-of-variables-like transformation.}). We
emphasize that there are many other imaginable choices of attributes; see
Appendix \ref{app:possible-decompositions} for further discussion.

\begin{figure}[ht]
    \centering
    \includegraphics[width=\linewidth]{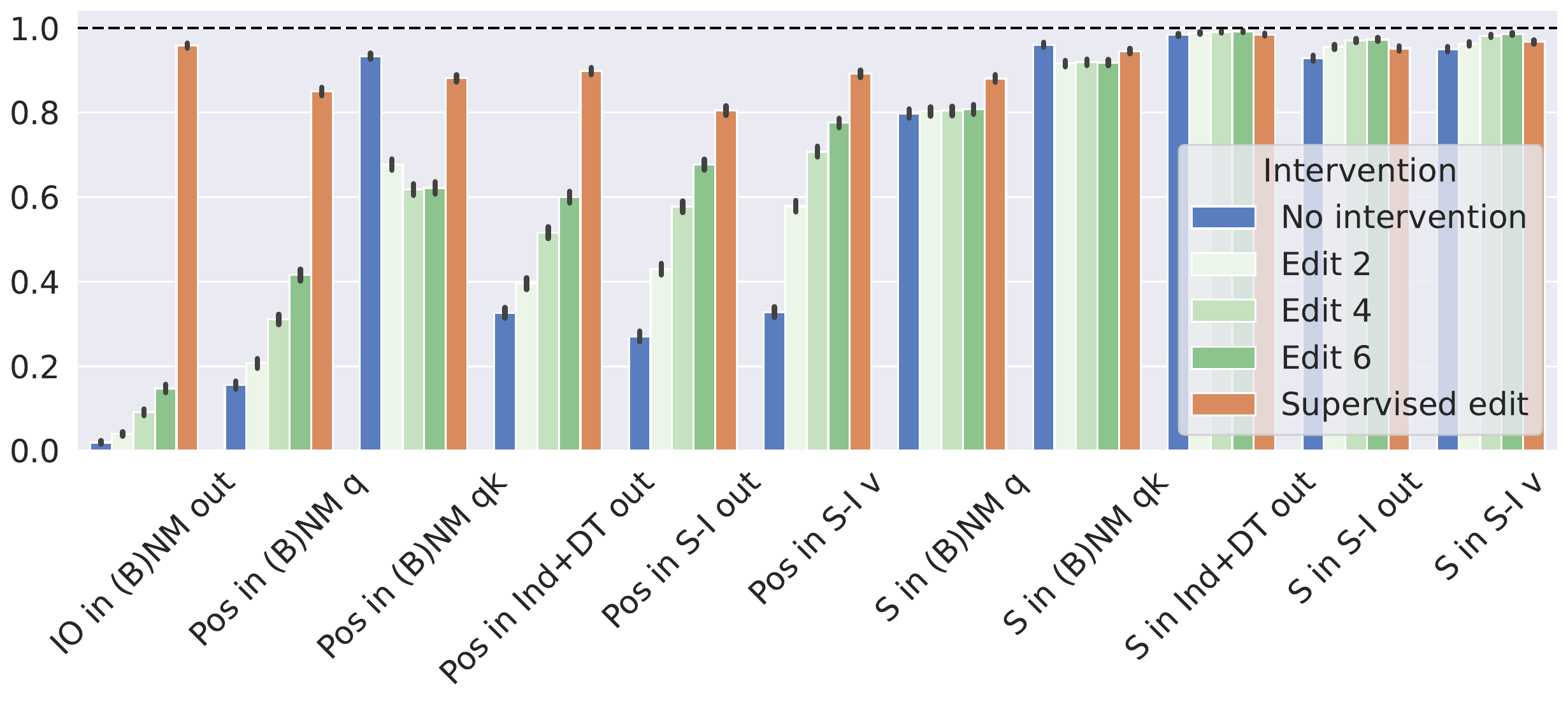}
    \caption{Accuracy when editing \textbf{IO}, \textbf{S} and \textbf{Pos} for
    circuit cross-sections using our supervised feature dictionaries and
    task-specific SAEs; the outcome in the absence of intervention is shown in
    blue for reference. When using task-specific SAEs, we edit either 2, 4 or 6
    features (which means we in total add and/or remove up to that many features
    from activations). For comparison, supervised edits always involve removing
    1 feature and adding 1 feature. Accuracy is measured as the proportion of
    examples for which the model's prediction agrees with the ground-truth
    prediction for the edit; see Section
    \ref{subsection:sae-evaluation-methodology} and Appendix
    \ref{app:ioi-supervised-details} for details.}
    \label{fig:edit-accuracy}
\end{figure}

\subsection{Evaluation Results}
\label{subsection:}

Our evaluations are carried out at the main cross-sections of the IOI circuit.
These cross-sections, as well as the information processing they perform and the
expected effect of editing attributes in them, are described in detail in
Appendix \ref{app:ioi-supervised-details}. These considerations guide our choice
of the attributes to edit in each given cross-section.

Sufficiency/necessity plots using the mean feature dictionaries are shown in
Figure \ref{fig:logitdiff-faithfulness-completeness} (orange bars); we find that
our supervised dictionaries are quite successful. For controllability, note
that attribute editing can naturally be defined in closed form, because we have
a 1-to-1 correspondence between attribute values and features by construction.
We find that simple feature arithmetic works quite well. Formally, given prompt
$p$ with $a_i(p)=v$, we can edit the attribute $a_i$ to have value $v'$ via
$\mathbf{a}_{edited}(p) := \mathbf{a}(p) - \mathbf{u}_{a_i(\cdot)=v} +
\mathbf{u}_{a_i(\cdot)=v'}$. Results for editing in cross-sections of the IOI
circuit are shown in Figure \ref{fig:edit-accuracy} (orange bars), where we
report the fraction of the time our intervention predicts the same token as the
ground-truth intervention that substitutes counterfactual activations in the
corresponding cross-section (a value of 1 is best). We observe that, when
editing the \textbf{S} attribute, not performing any intervention often already
mostly agrees with the ground-truth edit, effectively reducing the resolution of
our evaluation results for this attribute. Otherwise, we find our supervised
dictionaries to perform well, always achieving $>80\%$ agreement with the
ground-truth intervention.

Finally, the supervised feature dictionaries tautologically pass the
interpretability test, as they were defined to have a single feature activating
for each possible attribute value, achieving perfect $F_1$ scores. Accordingly, we
performed a more demanding test of interpretability: decomposing internal model
computations in terms of interactions between individual features. We find that
pre-softmax attention scores and composition between heads can be decomposed in
terms of feature-level interactions, such that many interactions are close to
zero, and the few non-zeros correspond to those expected based on the high-level
IOI circuit description from \citet{wang2022interpretability}; see Appendix
\ref{app:feature-level-mech-methodology} for details.

\section{Evaluating Task-Specific and Full-Distribution Sparse Autoencoders}
\label{section:sae-evaluation}

\subsection{Methodology}
\label{subsection:sae-evaluation-methodology}

\textbf{SAE training on the task-specific and full pretraining distributions.}
We trained SAEs on all IOI circuit locations, using activations from either the
IOI dataset (`task SAEs') or \textsc{OpenWebText} distribution
 \citep{Gokaslan2019OpenWeb} (`full-distribution SAEs'). While performance varied
strongly across circuit locations, most full-distribution SAEs had an
$\ell_0$-loss between 2 and 12 and a recovered loss fraction (against a mean ablation baseline)
between $0.4$ and $0.9$ (both measured on \textsc{OpenWebText}). Similarly, most
task-SAEs had an $\ell_0$-loss below $25$ and a recovered logit difference
fraction against mean ablation $>0.8$ (both measured on the IOI dataset)
\footnote{Importantly, we did not perform exhaustive hyperparameter tuning to
train these SAEs, as our main goal was to evaluate the methodology and how it
can distinguish between different classes of feature dictionaries, rather than
to achieve state-of-the-art performance. Thus it is possible that significantly
better performance could be achieved with more tuning. Indeed, it is our hope
that the methods we present here will be useful for tuning SAEs in the future.}.
Further details are given in Appendix \ref{app:sae-training-details}.

\textbf{Sparse controllability implementation in the IOI task.}
We instantiate the sparse controllability test from Subsection
\ref{subsection:test-2-sparse-controllability} as follows. Suppose our SAE has a dictionary of decoder
vectors $\{\mathbf{u}_j\}_{j=1}^m$, and the original and counterfactual
activations $\mathbf{a}_s, \mathbf{a}_{t}$ have reconstructions respectively
\begin{align*}
    \widehat{\mathbf{a}}_s = \sum_{i\in S} \alpha_i\mathbf{u}_i + \mathbf{b}_{dec}, \quad 
    \widehat{\mathbf{a}}_{t} = \sum_{i\in T} \beta_i\mathbf{u}_i + \mathbf{b}_{dec}
\end{align*}
for $S, T \subset \{1, \ldots, m\}$ and $\alpha_i, \beta_i > 0$. Consider the
optimization problem
\begin{align*}
    \min_{R\subset S, A\subset T, \left|R\cup A\right|\leq k} \l\|\mathbf{a}_s - \sum_{i\in R}^{}\alpha_i \mathbf{u}_i + \sum_{i\in A}^{} \beta_i \mathbf{u}_i - \mathbf{a}_t\r\|_2
\end{align*}
In words, this problem asks for at most $k$ features to remove ($R$) from and/or 
add ($A$) to the original activation\footnote{We also experimented with using the reconstructions instead of the actual activations in this algorithm, but results were worse.} to bring it as close as possible to the
counterfactual activation, where the features to add are taken directly from the
counterfactual one. In general, this problem has no polynomial-time (in
$k,\left|S\right|,\left|T\right|$) solution, as the NP-hard problem
\textsc{SubsetSum} reduces to it; instead, we use a greedy algorithm to find a
solution. 

\emph{Measuring the magnitude of edits.}
To measure the magnitude of the edit, we compare the contribution of the changed
features to the reconstruction against the analogous quantity for our `ideal'
supervised feature dictionary. Namely, for each summand in the reconstruction 
we assign a measure of its contribution $ \operatorname{weight}(i) =
(f_i\mathbf{u}_i)^\top
\l(\widehat{\mathbf{a}} - \mathbf{b}_{dec}\r)/\left\|\widehat{\mathbf{a}} - \mathbf{b}_{dec}\right\|_2^2$ so that
$\sum_{i=1}^k \operatorname{weight}(i) = 1$\footnote{
While weights can in general take any real value, we find
that in practice they are almost always approximately in $[0, 1]$; see Appendix
\ref{app:sae-evaluation-details} for empirical details.}. Note that weights are
additive in the features, so that the sum of weights of some subset of features
is the weight for these features' total contribution to the reconstruction. We
then measure the magnitude of an edit by the total weight of the features
removed during the edit. 

\textbf{Interpretability implementation for the IOI task.} We instantiate the
interpretability test from Section \ref{section:methods} as follows. Starting
with our primary attributes \textbf{IO}, \textbf{S}, and \textbf{Pos}, we
consider intersections of each of \textbf{IO} and \textbf{S} with the
\textbf{Pos} attribute. Then, for each resulting attribute that represents
variation over the \textbf{IO} or \textbf{S} name, we also consider unions of up
to 10 attribute values (30 for \textsc{OpenWebText}). Finally, we add some other
attributes of interest, such as the gender commonly associated with names in our
dataset. Details on all the possible interpretations we consdered are given in
Appendix
\ref{app:sae-interp-methodology}.
Additionally, we evaluate the causal role of highly interpretable features as
described in Subsection \ref{subsection:test-3-interpretability}:
\begin{itemize}
\item \textbf{sufficiency/necessity of interpretable features}: for
sufficiency, we intervene by subtracting from
activations only the features for which \emph{no} interpretation has an $F_1$
score above some threshold. To test necessity, we instead subtract the features
for which their chosen interpretation has an $F_1$ score above a threshold. 
This
experiment is directly comparable with the sufficiency/necessity of
reconstructions experiments, and uses the same baselines.
\item \textbf{interpretability-aware sparse control}: we attempt to edit a given
attribute by removing/adding only $\leq k$ features with highest $F_1$ score for
the value of this attribute in the original/counterfactual prompt.  
This experiment is directly comparable to the (interpretation-agnostic) sparse
control experiment and uses the same baseline. 
\end{itemize}

Additional details are given in Appendix \ref{app:sae-interp-methodology}.

\begin{figure}[ht]
    \centering
    \includegraphics[width=0.5\linewidth]{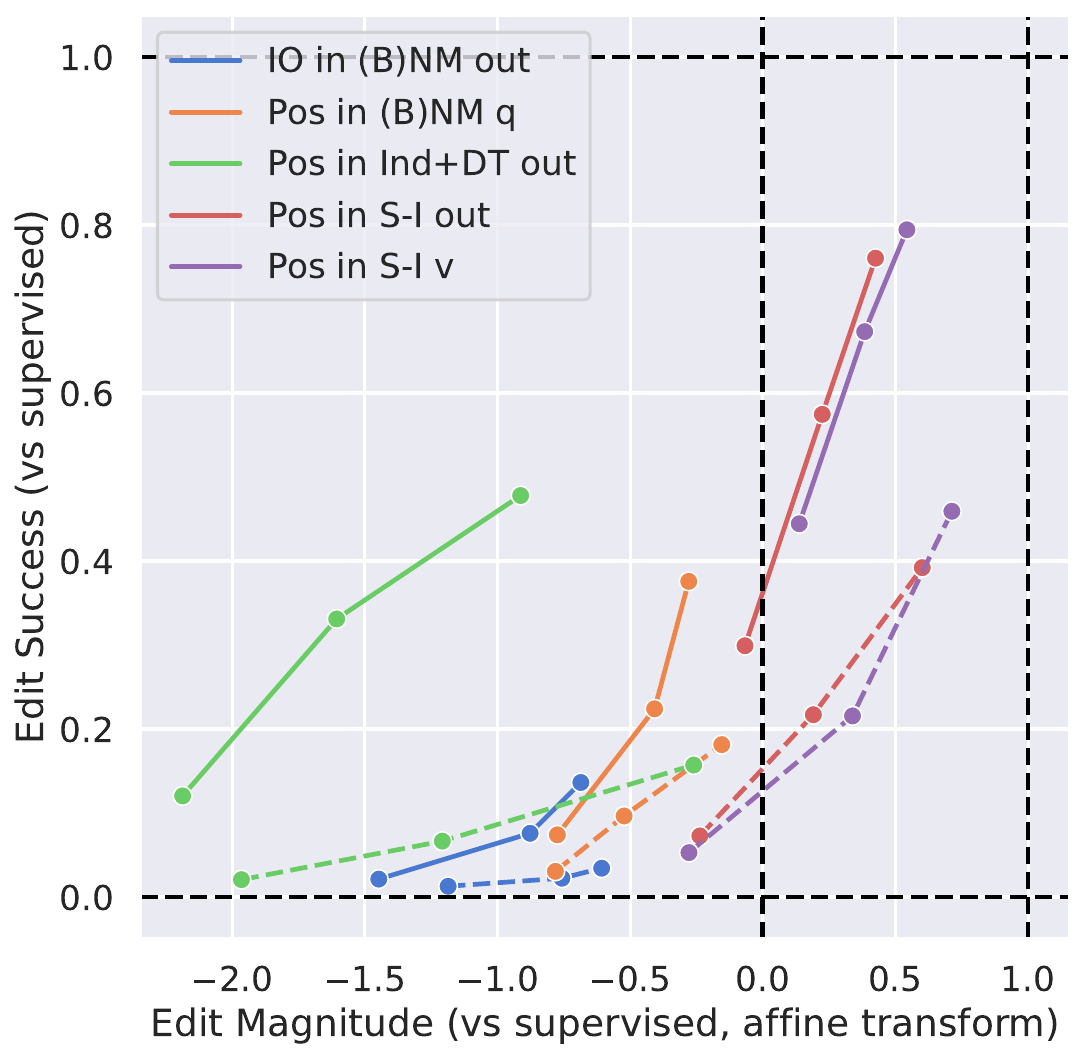}
    \caption{Trade-offs between edit magnitude and edit success for attribute
    editing using task-specific SAEs for select IOI circuit cross-sections. The
    x-axis measures the weight (see Subsection
    \ref{subsection:sae-evaluation-methodology})  of the features removed by the
    edit (features added are not reported in this plot), averaged over the
    attention heads in the cross-section. This metric is affine-transformed so
    that a value of 0 indicates the weight removed by the corresponding
    supervised edit, and a value of 1 indicates that the edit removed all
    features in the reconstruction. The y-axis is an affine transform of the
    fraction of examples for which the edit results in the same next-token
    prediction as the ground-truth edit, with a value of $0$ corresponding to no
    intervention, and a value of $1$ corresponding to the supervised edits.
    Results for our interpretation-agnostic/interpretation-aware editing methods
    are shown as thick/dashed lines respectively. For both methods, we edit 2, 4
    or 6 features (a higher magnitude score indicates editing more features).}
    \label{fig:pareto-combined}
\end{figure}

\subsection{Results}
\label{subsection:}

\textbf{Test 1: Sufficiency/necessity of reconstructions.} The
sufficiency/necessity plots for the task-specific and full-distribution SAEs are
shown in Figure \ref{fig:logitdiff-faithfulness-completeness} (green and red).
We find that the task-specific SAEs offer a significantly worse, but not
catastrophically bad, approximation compared to the supervised feature
dictionaries. Meanwhile, the full-distribution SAEs fare notably worse than the
task SAEs.

\textbf{Test 2: (Interpretation-agnostic) sparse controllability.} 

\emph{Task-specific SAEs.} Results for sparse controllability using our
task-specific SAEs are shown in Figure \ref{fig:edit-accuracy} as well as Figure \ref{fig:pareto-combined} for $k=2,4,6$. We
find that our SAEs can edit the \textbf{Pos} attribute and, to a small extent,
the \textbf{IO} attribute, even though this requires changing more features
compared to the supervised dictionaries. For the \textbf{S} attribute, the
results are less clear, because the range of performance between `no
intervention' and the supervised edit is often within the margin of error; the
exception is the queries of the name mover heads, where results indicate failure
of the SAE features. Regarding the magnitude of edits, the most successful edits
introduce higher-magnitude changes as measured by the weight (recall Subsection
\ref{subsection:sae-evaluation-methodology}) than the corresponding supervised
edits (results in Figure \ref{fig:pareto-combined} as well as Appendix Figure \ref{fig:removed-weight-sae}). On the positive
side, the edits don't overwrite the SAE features entirely.

\emph{Full-distribution SAEs.} Full-distribution SAEs require a significantly
larger number of features to achieve a statistically significant level of
control compared to task-specific ones (often 32 or more); results are shown in 
Appendix Figure \ref{fig:agnostic-editing-accuracy-webtext}. We also found that
the magnitude of these edits surpsasses that of supervised edits significantly,
often up to the point of throwing away total weight approaching $1$.

\emph{A baseline: task-specific SAEs with decoder directions frozen at
initialization.} Do these results demonstrate \emph{any} non-trivial
controllability afforded by the SAE features? To check, we run the same
controllability pipeline on task-specific SAEs which were trained with frozen at
initialization decoder directions; thus, these SAEs are forced to approximate
activations using sparse sums of random features. Results in Appendix Figure
\ref{fig:agnostic-editing-accuracy-freeze-decoder} show that our evaluation
distinguishes between these two types of SAEs: the frozen-decoder task SAEs perform
much worse than the ordinary task SAEs. In fact, we find that the
frozen-decoder task SAEs perform on par with the full-distribution SAEs.

\textbf{Test 3: Interpretability.}

\begin{figure}[ht]
    \centering
    \includegraphics[width=\textwidth]{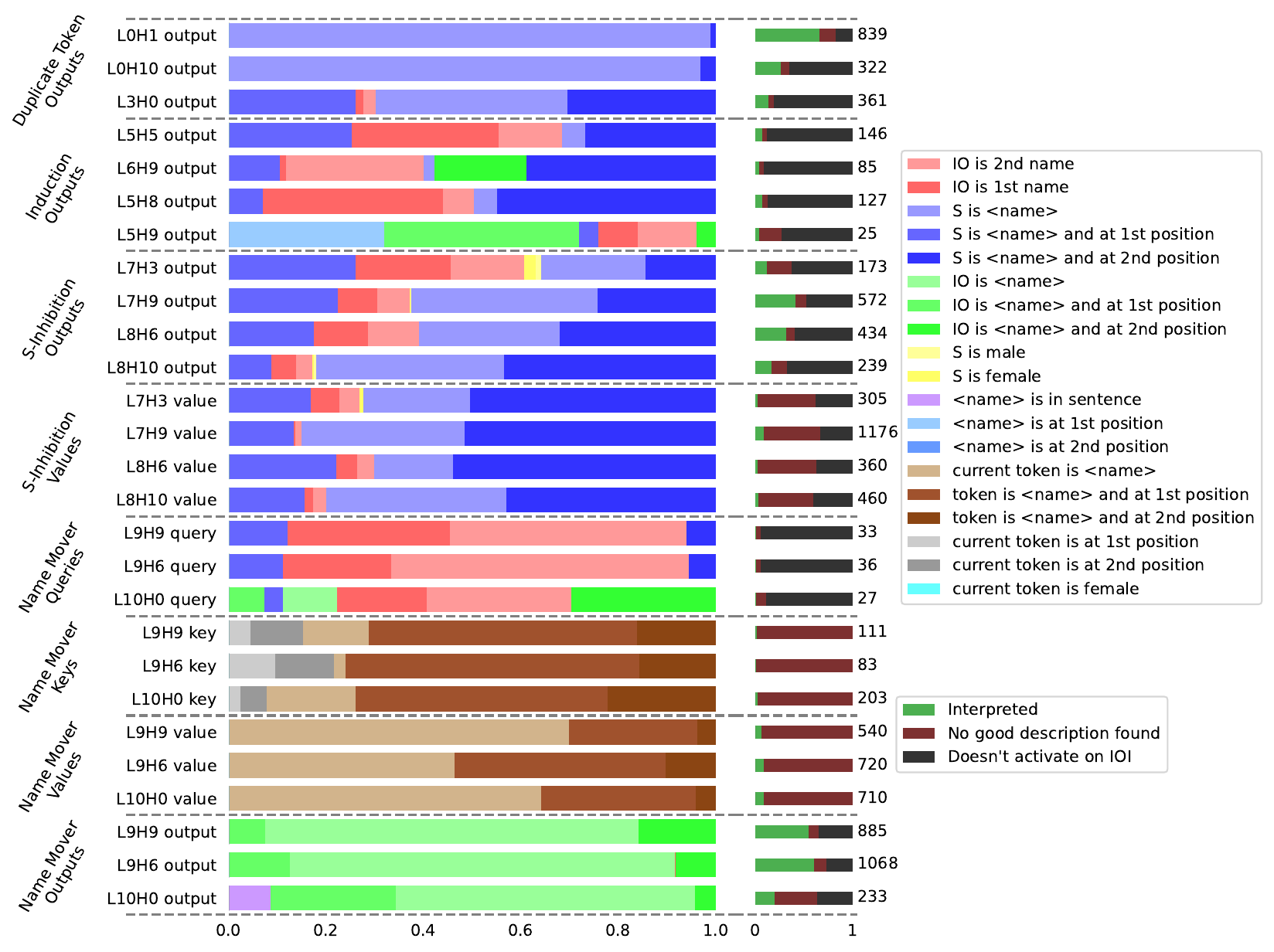}
    \caption{Interpreting the IOI features learned by SAEs trained on
    \textsc{OpenWebText}. For each node in the IOI circuit, we show the
    distribution of interpretations for the features which have any interpretation with $F_1$ score
    above a threshold. The numbers in the right column indicate the number of
    features with an assigned interpretation by our method, and the color bars
    show the overall distribution of the SAE features (conditioned on the
    feature not being dead on the SAE training distribution). See Section
    \ref{section:sae-evaluation} for methodology; details on the
    interpretations considered are given in Appendix
    \ref{app:sae-interp-methodology}.}
    \label{fig:full-sae-interp-most-interp}
\end{figure}

\emph{Correlational evaluation.} Full-distribution SAEs must capture variation
in activations across a large set of text, of which IOI-like prompts are only a
small subset. Consistent with this, we found that only a subset of
full-distribution SAE features activates on IOI prompts, with the number of
features that fire on IOI prompts varying strongly between components. We scored
the features that do fire on IOI prompts and found a significant amount of
feature descriptions with high $F_1$-score. We summarize the number of
high-$F_1$-score features per type in Figure
\ref{fig:full-sae-interp-most-interp}. Remarkably, we find that the
interpretable features in the full-distribution SAEs and the task-specific SAEs
are qualitatively similar; corresponding task-specific results are given in
Appendix Figures \ref{fig:task-sae-interp-most-interp} (showing the most
interpretable SAEs at each node of the IOI circuit) and
\ref{fig:task-sae-interp-best-metrics} (showing the SAEs chosen to optimize the
tradeoff between the $\ell_0$ loss and the logit difference reconstruction, which we use throughout the main body of the paper).

In practice, we want to use feature explanations to get insight into the more
general computation of a component. Thus, we investigated whether the features
found are consistent with the previously established function of the heads from
\citet{wang2022interpretability}. We found that this was true for all heads and
that simply looking at the number of features with a given interpretation draws
a clear picture; examples are provided in Appendix
\ref{app:interp-extra-observations}. 

Our results also suggests several details about the IOI circuit that weren't
reported previously, which we summarize in Appendix
\ref{app:interp-extra-observations}. We were also curious about how the detected
features behave on arbitrary text of the model's training distribution. As
creating a rigorous test for this is difficult, we report some anecdotal
evidence in support of feature generalization in Appendix
\ref{app:interp-extra-observations}.

\emph{Causal evaluation: sufficiency/necessity of interpretable features.}
Results here are encouraging: keeping/removing the features with $F_1$ score
$\geq 0.6$ often goes a long way towards preserving/degrading the model's
performance on the task. Appendix Figures \ref{fig:interp-sufficiency} and
\ref{fig:interp-necessity} show the results of these experiments for task SAEs.

\emph{Causal evaluation: sparse control via interpretable features.} Here,
results are also moderately encouraging. We find that, for the task-specific
SAEs, editing using the high-$F_1$-score features w.r.t.\ a given attribute as a
guide performs not much worse than the interpretation-agnostic editing method. Results
are provided in Figure \ref{fig:pareto-combined} in the form of a comparison
with interpretation-agnostic editing, as well as Appendix Figure
\ref{fig:editing-accuracy-sae-interp}. However, for full-distribution SAEs, we
again need to edit a high number of features to achieve a noticeable effect
(results in Appendix Figure \ref{fig:editing-accuracy-sae-interp-webtext}).

\section{Qualitative Phenomena in SAE Learning}
\label{section:qualitative-phenomena}
\begin{figure}[ht]
    \centering
    \begin{subfigure}{.52\textwidth}
        \centering
        \includegraphics[width=\linewidth]{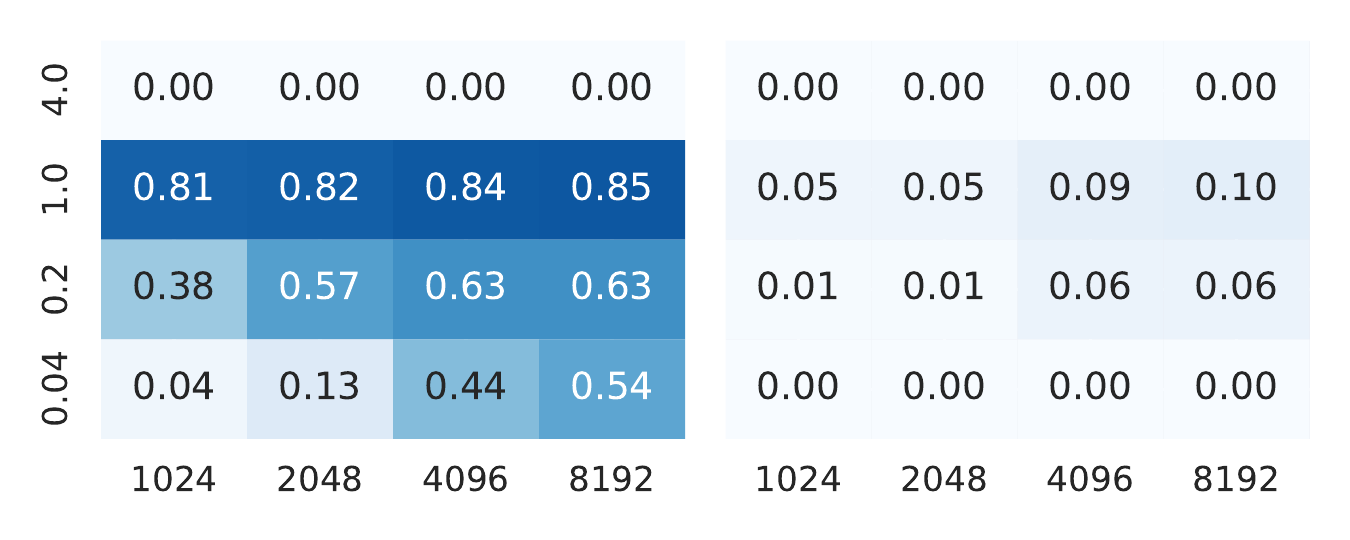}
        \label{subfig:occlusion-io}
    \end{subfigure}%
    \hfill
    \begin{subfigure}{.45\textwidth}
        \centering
        \includegraphics[width=\linewidth]{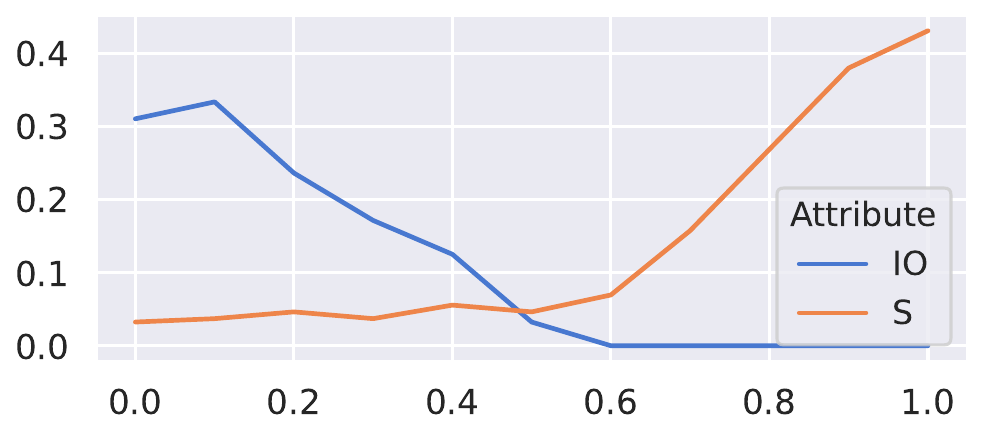}
        \label{subfig:occlusion-s}
    \end{subfigure}%
    \caption{\textbf{Left}: Fraction of \textbf{IO} (left subplot) and
    \textbf{S} (right subplot) names in our dataset for which a feature with
    $F_1$ score $\geq0.5$ is found, as a function of dictionary size (x-axis)
    and effective $\ell_1$ regularization coefficient (y-axis), over a wide
    hyperparameter sweep for the queries of L10H0. 
    \textbf{Right}: fraction of \textbf{IO} and \textbf{S} names in our dataset
    for which a feature with $F_1$ score $\geq 0.5$ is found, as a function of
    $\alpha$ ($x$-axis), the fraction of supervised \textbf{IO} features we
    subtract from the activations.}
    \label{fig:occlusion-main}
\end{figure}

\subsection{Feature Occlusion}
\label{subsection:}
Our experiments suggest that when two (causally relevant)
attributes are represented in the same activation, but one attribute's (supervised) features have overall
higher magnitude, SAEs have a tendency to robustly learn more interpretable
features for the attribute with higher magnitude. We observe this in the queries
of the L10H0 name mover head at the END token in the IOI circuit, where the
\textbf{IO} and \textbf{S} attributes are both represented and causally
relevant (ablating either leads to a change of $\approx0.5$ in logit difference). We find that our task SAEs consistently find features with high $F_1$ score
for individual \textbf{IO} names, but fail to find a significant number of
features for individual \textbf{S} names. In Figure \ref{fig:occlusion-main}
(left), we show interpretability results from training SAEs over a wide grid of
hyperparameters that confirm this observation; more details in Appendix
\ref{app:occlusion}.

\emph{Hypothesis: feature magnitude is a driver of occlusion.} We noticed that
the supervised features for \textbf{IO} and \textbf{S} names in the L10H0
queries have a small but significant difference in norm (see Appendix Figure
\ref{fig:occlusion-extra} (left)). We then hypothesized that feature magnitude
is a factor in this phenomenon. To check this, we surgically reduce the
magnitude of \textbf{IO} features in the activations using our supervised
feature dictionaries, and observe that the number of \textbf{S} features
discovered in these modified activations monotonically increases as we remove
larger fractions of the \textbf{IO} features (Figure \ref{fig:occlusion-main}
(right); see Appendix \ref{app:occlusion} for methodology). We furthermore
constructed a simple toy model based on i.i.d. isotropic random features that
mimic the norms of supervised features in the L10H0 queries, and find that a
similar phenomenon occurs for the distribution of features with high $F_1$ score
(e.g. higher than $0.9$; see Appendix \ref{app:occlusion}).

\subsection{Feature Over-splitting}
\label{subsection:}
We also found in our experiments that task SAEs have a tendency to split a single binary
attribute into multiple features, even when the number of features available
could in principle be spent on other attributes. Note that, while this behavior
may be counter-intuitive from a human standpoint, it does not necessarily mean
that the SAE failed; it may be that the binary attribute is not part of an
optimal sparse description of the model's internal states. We observe this
phenomenon with the \textbf{Pos} attribute in the IOI task (again in the queries
of the L10H0 name mover), which is robustly split into many (e.g.  $\geq10$)
features by our SAEs that activate for small, mostly non-overlapping subsets of
examples sharing the same \textbf{Pos} value that have no clear semantic
interpretation.

\emph{Is over-splitting a form of over-fitting?} To investigate whether this is due
to overfitting, we compared \textbf{Pos} features between (1) different random
seeds for the same training dataset and (2) different training datasets. In both
cases, we found that the \textbf{Pos} features discovered are similar above
chance levels, suggesting that the over-splitting is not due solely to
overfitting to randomness in the training algorithm and/or dataset. 

\emph{Reproducing the over-splitting phenomenon in a simple toy model.} On the
other hand, we show empirical evidence that in a toy model where activations are
a uniform mixture of two isotropic Gaussian random variables, an appropriately
\emph{randomly initialized} SAE with enough hidden features will achieve lower
total loss than an ideal SAE with just two features corresponding to the two
components of the mixture. Such randomized constructions exist for \emph{any}
$\ell_1$ regularization coefficient, even in the limit of infinite training
data. Details are given in Appendix \ref{app:over-splitting}.

\section{Related Work}
\label{sec:related-work}

\textbf{Learning and evaluating SAE features.} The SAE paradigm predates the recent surge of interest in LLMs. Early work
in ML focused on the analysis of word embeddings \citep{DBLP:conf/nips/MikolovSCCD13},
with works such as
\citet{Faruqui2015SparseOW,Subramanian2017SPINESI,arora2018linear} finding
sparse linear structure. \citet{elhage2022superposition} proposed the use
of sparse autoencoders to disentangle features in LLMs.
\citet{sharkey2023taking} used SAEs to learn an over-complete dictionary in a
toy model and in a one-layer transformer, and follow-up work by
\citet{cunningham2023sparse} applied this technique to residual stream
activations\footnote{We adopt the terminology of \citep{elhage2021mathematical}
when discussing internal activations of transformer-based language models.} of a
6-layer transformer from the Pythia family \citep{Biderman2023PythiaAS}.
\citet{bricken2023monosemanticity} trained SAEs on the hidden MLP activations of
a 1-layer language model, and performed several thorough evaluations of the
resulting features. Similarly to us, \citet{Gould2023SuccessorHR} trained SAEs
on a \emph{narrow} data distribution instead of internet-scale data. \citet{tamkin2023codebook} incorporated sparse feature
dictionaries (without per-example weights for the features) into the model
architecture itself, and fine-tuned the model on its pre-training distribution
to learn the dictionaries. More recently, \citet{gpt2_attention_saes} used SAEs
on attention layer outputs of GPT-2 Small and found learned features that are
consistent with the IOI circuit from \citet{wang2022interpretability}.

Sparse autoencoders seek to (approximately) decompose
activations into meaningful features, and are thus a stronger form of
interpretability than linear probing for individual concepts
\cite{Alain2016UnderstandingIL}, or finding individual subspaces with causal
effect \citep{geiger2023finding}.

Throughout this line of work, the evaluation of learned SAE features has been a
major challenge. The metrics used so far can be broadly categorized as follows:
\begin{itemize}
\item \textbf{indirect geometric measures}: \citet{sharkey2023taking} proposed
use of the mean maximum cosine similarity (MMCS) between two different SAEs'
learned features to evaluate their quality. However, this metric relies on the
assumption that convergence to the same set of features is equivalent to
interpretability and having found the `true' features.
\item \textbf{auto-interpretability}:
\citet{bricken2023monosemanticity,bills2023language,cunningham2023sparse} used a
frontier LLM to obtain natural-language descriptions of SAE features based on
highly activating examples, and use the LLM to predict feature activations on
unseen text; the prediction quality is then used as a measure of
interpretability. However, the use of maximum (or even stratified by activation
value) activating examples has been criticized as potentially giving an illusory
and subjective sense of interpretability \citep{bolukbasi2021interpretability}.
\item \textbf{manually crafted proxies for ground truth}:
\citet{bricken2023monosemanticity}  manually formed hypotheses about a handful
of SAE features and defined computational proxies for the ground truth features
based on these hypotheses. This method may be less prone to blind spots than
auto-interpretability, but still relies on the correctness of the computational
proxy.
\item \textbf{toy models}: \citet{sharkey2023taking} used a toy model where
ground-truth features are explicitly defined; however, it is unclear whether toy
models miss crucial aspects of real LLMs. Similar objections apply to manually
injecting ground-truth features into a real model.

\item \textbf{direct logit attribution}: \citet{bricken2023monosemanticity}
additionally considered the direct effect of a feature on the next-token
distribution of the model; this method is valuable because it tells us about the
causal role of a feature, but it cannot detect its indirect effects via other
features.
\end{itemize}

Beyond the evaluation challenges, there is debate about whether SAEs find
computationally non-trivial, compositional features, or merely clusters of
similar examples in the data \citet{olah2024circuits}.

\textbf{Mechanistic interpretability and circuit analysis.} Mechanistic interpretability (MI) aims to reverse-engineer the internal workings
of neural networks \citep{olah2020zoom,elhage2021mathematical}. In particular,
MI frames model computations as a collection of \emph{circuits}: narrow,
task-specific algorithms \citep{olah2020zoom}. So far, circuit analyses of
LLMs have focused on the component level, mapping circuits to collections of
components such as attention heads and MLP layers
\citep{wang2022interpretability,docstring}. 

However, the linear representation hypothesis suggests that component
activations can be broken down further into (sparse) linear combinations of
meaningful feature vectors; thus, the eventual goal of MI is to give a precise,
\emph{subspace-level} understanding of the model's circuits. Initial steps in
this direction have been taken using methods distinct from SAEs.
\citet{geiger2023finding} propose finding meaningful subspaces using an
optimization-based method; \citet{nanda2023emergent} discover linear subspaces
in emergent world-models on a toy task; \citet{tigges2023linear} discover linear
subspaces corresponding to sentiment in a LLM. However, while these works focus
on finding individual subspaces representing specific concepts, the SAE paradigm
is more ambitious, as it aims to fully decompose activations as a sum over
meaningful features. This is a stronger property than identifying individual 
meaningful subspaces, and would accordingly provide a more exhaustive form of 
interpretability.

More broadly, MI has found applications in several downstream tasks: removing
toxic behaviors from a model \citep{li2023circuit}, changing factual knowledge
encoded by models \citep{meng2022locating}, improving the truthfulness of LLMs
at inference time \citep{li2023inference}, studying the mechanics of gender bias
in language models \citep{vig2020causal}, and reducing spurious correlations by 
intervening on model internals \citep{gandelsman2023interpreting}.

\section{Limitations and Conclusion}
\label{section:}

We have taken steps towards more principled and objective evaluations of the
usefulness of sparse feature dictionaries for disentangling LLM activations. In
particular, we have demonstrated that:
\begin{itemize}
\item simple supervised methods can be used as a principled way to compute
high-quality feature dictionaries in a task-specific context;
\item these dictionaries can be used as `skylines' to evaluate and contextualize
the performance of unsupervised methods, such as SAEs.
\end{itemize}

\textbf{Limitations.}
The central conceptual limitation of our work is that our method relies on
supervision in the form of a potentially subjective choice of variables used to
parametrize task-relevant information in model inputs. We mitigate this to some
extent by requiring this parametrization to be \emph{consistent} with the
internal computations of the model, as quantified by our tests for
approximation, control and interpretability of model computations on the task.
However, in principle there could be many parametrizations that are just as
consistent, but fundamentally different (recall the discussion in Appendix
\ref{app:alternative-ioi-parametrizations} and
\ref{app:possible-decompositions}). Thus, we risk making the proverbial `judging
a fish by its ability to climb a tree' mistake. We have mitigated this problem
further by devising evaluations that are, when possible, agnostic to the precise
features in a dictionary, as long as they allow us to disentangle and control
our chosen variables in a sparse manner. 

Finally, we find it likely that features used by LLMs in tasks of practical
interest will be quite complex from a human standpoint, and we believe it is
useful to be able to assess the degree to which these features can be harnessed
towards controlling model behavior along more top-down, concise and
human-understandable concepts. Our evaluations provide such an assessment.

Our work is also limited in that we only consider a single task, and a single
language model. We hope that in future work we will reduce these limitations.

\textbf{Conclusion.} Sparse dictionary learning, and SAEs in particular,
represent an interesting and promising avenue for disentangling internal
representations of LLMs. However, in order to measure progress in this area, it
is imperative to have nuanced conceptual understanding of the goals and
downstream applications of such disentangling, and benchmarks faithful to this
understanding. We hope that our work will inspire further research in this
direction, and that our methods will be useful for practitioners in the field.

\section*{Acknowledgements}
We thank Sonia Joseph and Can Rager for feedback on earlier versions of this
manuscript. We thank the SERI MATS program for making this collaboration possible. We
also thank Atticus Geiger for valuable discussions and supervision during some
early versions of this project.  AM was supported by grants from AI Safety
Support and Effective Ventures. GL was supported by grants from AI Safety
Support. 

\section*{Author Contributions}

AM designed and ran the algorithms for computing supervised feature dictionaries
and observed their theoretical properties; trained task-specific SAEs; designed
the tests for sufficiency/necessity of reconstructions, sparse controllability
and causal faithfulness of feature descriptions; ran these tests on full-distribution and task-specific SAEs; observed the occlusion/oversplitting phenomena in task-specific SAEs and developed the toy models for them; wrote all sections of the paper unless otherwise specified.

GL trained full-distribution SAEs; designed and ran the automatic feature scoring algorithms for them, and ran most of the interpretability evaluations for full-distribution SAEs; contributed to the sections on interpretability of full-distribution SAE features, created Figures \ref{fig:graphical-abstract} and \ref{fig:full-sae-interp-most-interp}.

NN provided guidance on research goals and prioritization, as well as writing, throughout the entire project; proposed the initial idea of finding supervised feature dictionaries for the IOI task, and using them to evaluate (unsupervised) feature dictionaries; refined the search for algorithms to compute the supervised dictionaries.

\bibliography{example_paper}
\bibliographystyle{arxiv_template}

\newpage
\appendix
\onecolumn
\section{Appendix}

\subsection{Additional details on the IOI circuit}
\label{app:ioi-circuit-details}

\textbf{Circuit structure.} 
To refer to individual token positions within the sentence, we use the notation
of \citet{wang2022interpretability}: IO denotes the position of the \textbf{IO}
name, S1 and S2 denote respectively the positions of the first and second 
occurrences of the \textbf{S} name (with S1+1 being the token position after
S1), and END denotes the last token in the sentence (at the word `to').

\citet{wang2022interpretability} suggest the
model uses the algorithm `Find the two names in the sentence, detect the
repeated name, and predict the non-repeated name' to do this task. Specifically,
they discover several classes of heads in the model, each of which performs a
specific subtask of this overall algorithm. A simplified version of the circuit
involves the following three classes of heads and proceeds as follows:
\begin{itemize}
\item \textbf{Duplicate token heads}: these heads detect the repeated name in
the sentence (the \textbf{S} name) and output information about both its
position and identity to the residual stream\footnote{We follow the conventions of
\citet{elhage2021mathematical} when describing internals of transformer models.
The residual stream at layer $k$ is the sum of the output of all layers up to
$k-1$, and is the input into layer $k$.}
\item \textbf{S-Inhibition heads}: these heads read the identity and position of
the \textbf{S} name from the residual stream, and output a signal to the effect
of `do not attend to this position / this token identity' to the residual stream
\item \textbf{Name Mover heads}: these are heads that attend to names in the
sentence. Because the signal from the S-Inhibition heads effectively removes the
\textbf{S} name from the attention of these heads, they read the identity of the
\textbf{IO} name from the input prompt, and copy it to the last token position
in the residual stream.
\end{itemize}
In reality, the circuit is more nuanced, with several other classes of heads
participating: previous token heads, induction heads \citep{olsson2022context},
backup name mover heads, and negative name mover heads. 
In particular, the
circuit exhibits \emph{backup behavior} \citep{mcgrath2023hydra} which poses
challenges for interpretability methods that intervene only on single model
components at a time. We refer the reader to Figure \ref{fig:ioi-circuit} for a
schematic of the full circuit, and to \citet{wang2022interpretability} for a
more complete discussion.

\begin{figure}
    \centering
    \includegraphics[width=0.9\textwidth]{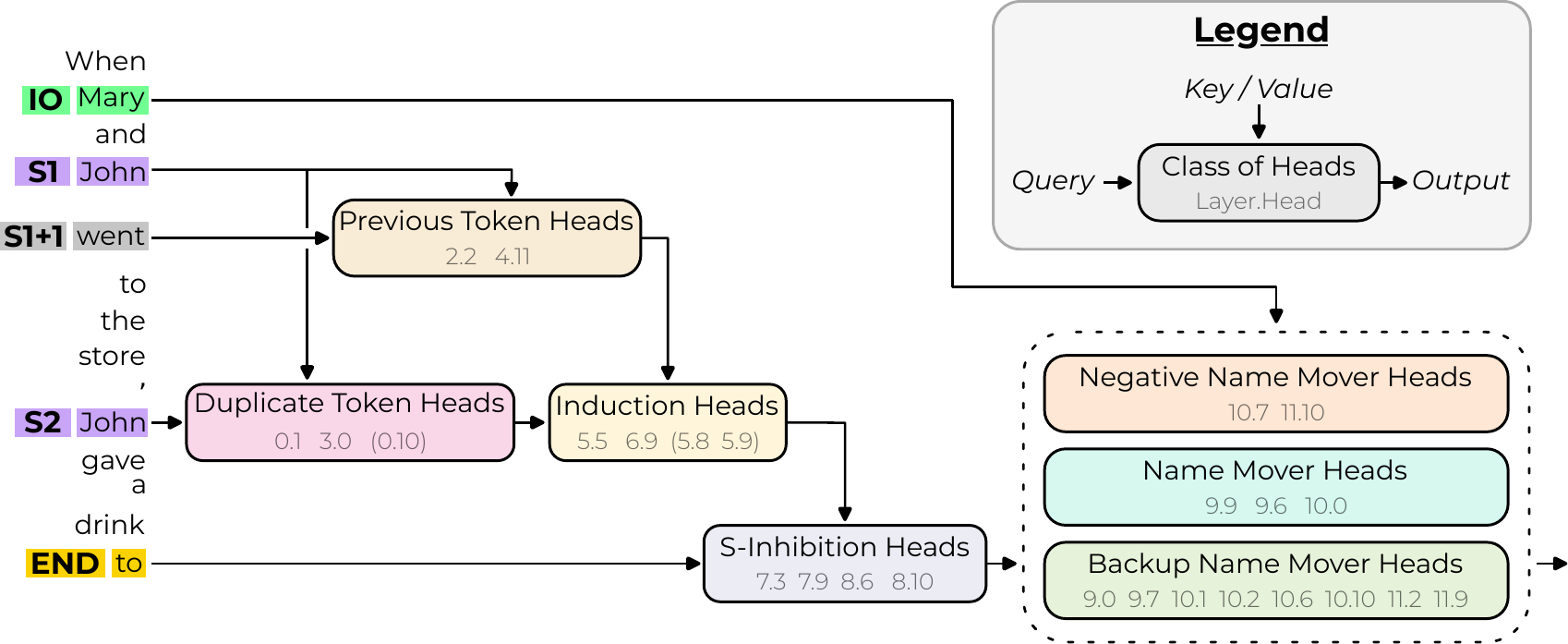}
    \caption{A reproduction of Figure 2 from \citet{wang2022interpretability},
    showing the internal structure of the IOI circuit. Original caption:
    \emph{The input tokens on the left are passed into the residual stream.
    Attention heads move information between residual streams: the query and
    output arrows show which residual streams they write to, and the key/value
    arrows show which residual streams they read from.}}
    \label{fig:ioi-circuit}
\end{figure}

\subsection{Additional details for Section \ref{section:methods}}
\label{app:methods-details}

\textbf{Why do we use mean ablation as a baseline for necessity of
reconstructions?} Using the mean instead of simply a zero value is intended to
not take the model away from the task distribution. The information that is
constant across the task distribution -- such as grammar and syntax -- will
remain in the mean ablation, while task-specific information will be averaged
out. If the residual $\mathbf{a} - \widehat{\mathbf{a}}$ is not task-relevant,
we expect that the model will perform the task as well as with the mean
ablation; conversely, if the reconstructions $\widehat{\mathbf{a}}$ leave out
task-relevant information, we expect that the model will perform better with our
intervention than with the mean ablation.

\textbf{Editing in multiple model locations at once.} We do this by
pre-computing edited activations for each location, then patching all of them
into the model's forward pass at once. This means that we are only editing
computational cross-sections of the model; however, applying the edits
sequentially as the model is run is prone to taking activations
off-distribution, since an edit will propagate to downstream activations, where
making the same edit again may be the wrong intervention; we have observed this
in practice, in the interaction of the name mover heads in layers 9 and 10 in
the IOI circuit.

\textbf{Evaluating edits.} Our edits are intended to only change the values of
specific attributes, and leave all other information in the activations
unchanged. A ground-truth baseline for the effect of an edit would be to take
the model's activation on a \emph{counterfactual} prompt: one which differs from
the original prompt only in the value of the attribute being edited, in the
precise way that the edit changes it. The existence of counterfactual prompts in
the support of $\mathcal{D}$ is not guaranteed in general, but it holds in our
IOI distribution.

\textbf{Precision, recall, and the $F_1$-score.}
Given a set of examples $S$ used for evaluation, a learned feature $f$ active on
a subset $F\subset S$ of examples, and a binary attribute of a prompt which is
true for the subset $A\subset S$, we define
$\operatorname{recall}(F,A)=\left|A\cap F\right| / \left|A\right|$ and
$\operatorname{precision}(F,A)=\left|A\cap F\right| / \left|F\right|$. Following
\citet{bricken2023monosemanticity}, we consider a feature meaningful for a given
property if it has both high recall and high precision for that property, and we
combine them into a single number using the F-score:
\begin{align*}
    F_1(F,A) =
    \frac{2\operatorname{precision}(F,A)\operatorname{recall}(F,A)}{\operatorname{precision}(F,A)+\operatorname{recall}(F,A)}.
\end{align*}
An $F_1$-score of $\alpha$ guarantees that both precision and recall are at least
$\frac{\alpha}{2-\alpha}$. For example, when $\alpha=0.8$ (the value we use in
most evaluations), both precision and recall are at least $0.8/1.2\approx 0.67$.
Requiring a sufficiently high $F_1$ value is important in order to avoid
labeling a trivial feature as meaningful for attributes where $\left|A\right|$
is large, because then a feature active for all examples can have a high
$F_1$-score. 

The $F_1$ score has some limitations in the context of our work:
\begin{itemize}
\item it does not take into account the magnitude of the feature activations;
for instance, a feature that is active for all examples in $S$ but only has high
activation values on the examples in $A$ may have a low $F_1$ score, even though
it is in some sense highly informative for the attribute $A$.
\item it is a very conservative metric, in that it requires both high precision
and high recall to be high. For example, a feature with precision $0.5$ but
recall $0.02$ will have an $F_1$ score of $\sim0.04$, heavily skewed towards the
lower of the two metrics, even though it is in some sense informative for the
attribute $A$.
\end{itemize}
We hope to address these limitations in future work.

\subsection{Additional details for Section \ref{section:manual-sae}}
\label{app:ioi-supervised-details}

\textbf{Why do we use fixed coefficients for our supervised activation
reconstructions?}
Note that the formulation from Equation \ref{eq:manual-sparse-codes} is
\emph{less expressive} than the reconstruction provided by an SAE (Equation
\ref{eq:sae-reconstruction}), as it requires the coefficients of decoder vectors
to be fixed (similar to \citet{tamkin2023codebook}, and unlike in the SAE, where
they are computed from the input via a linear function and a ReLU). This is a
reasonable assumption in settings where we expect the relevant features to
behave as binary, on/off switches (as opposed to having continuous degrees of
activation). The IOI task is an example of such a setting, as we expect that
there is no `degree' to which a given name in the sentence is present or not, or
to which a given name is repeated or not. See also the discussion of `features
as directions' vs `features as points' in \citet{tamkin2023codebook}.

\textbf{Computing and evaluating supervised feature dictionaries.} For each
parametrization and each method to compute feature dictionaries, we use 20,000 prompts sampled
from our IOI distribution (see Appendix \ref{app:ioi-dataset-details}) to
compute feature dictionaries for the query, key, value, and attention output (i.e.,
attention-weighted values) of the relevant token positions of all 26 heads
identified in \citet{wang2022interpretability} (recall Figure
\ref{fig:ioi-circuit}). We use another sample of 5,000 prompts to validate the
quality of the feature dictionaries. 

\textbf{Cross-sections of the circuit.} Based on the understanding of the IOI
circuit from \citet{wang2022interpretability}, we identify several 
cross-sections of the computational graph of the IOI circuit where feature
editing is expected to have effects meaningful for the task:
\begin{itemize}
\item \emph{outputs of (backup) name mover heads at END (\textbf{(B)NM out})}: these activations
encode the \textbf{IO} name and write it to the END token of the residual
stream. We expect that editing the \textbf{IO} name in these activations will
directly affect the model's prediction, while editing other attributes will not
have a significant effect.
\item \emph{queries+keys of (backup) name movers at END (\textbf{(B)NM qk})}: the queries represent the
\textbf{S} name and \textbf{Pos} information, but they are mainly used as
\emph{inhibitory} signals for the model, decreasing the attention to the
\textbf{S} token\footnote{In addition, we later find that the queries of the L10H0 name mover
head also represent the \textbf{IO} attribute, and serve an \emph{inhibitory}
role for it as well, decreasing the attention to the \textbf{IO} token.}. The
keys represent information about the \textbf{IO} and \textbf{S} names: in
particular, the \textbf{S} information combines with the query to inhibit
attention to the \textbf{S} token.

We expect that editing the \textbf{S} and \textbf{Pos} attributes in \emph{both}
the keys and queries will not significantly hurt model performance, because as a
result attention to the \textbf{S} token will again be inhibited. By contrast,
it is unclear what editing the \textbf{IO} name is expected to do, since its
role in the attention computation is not fully described in
\citet{wang2022interpretability}.
\item \emph{outputs of S-Inhibition heads at END (\textbf{S-I out}), values of
S-Inhibition heads at S2 (\textbf{S-I v}), and outputs of duplicate token and
induction heads at S2 (\textbf{Ind+DT out})}: these activations transmit the
inhibitory signal to the name mover heads through the residual stream. We expect
that editing \textbf{S} and \textbf{Pos} in these activations will lower the
model's logit difference by disrupting the inhibitory signal, while editing
\textbf{IO} will have no effect.
\end{itemize}

\textbf{Evaluating necessity of feature reconstructions}: When we intervene on the model by removing
reconstructions from activations in cross-sections of the circuit, model
performance on the IOI task (as measured by the logit difference) goes down from the clean value $\operatorname{logitdiff}_{\mathbf{clean}}$ to a lower value $\operatorname{logitdiff}_{\mathbf{intervention}}$. As
we describe in the main text, the ground-truth intervention for removing the
features from the activations is mean ablation of the corresponding
cross-section, which also results in a lower value of the logit difference, $\operatorname{logitdiff}_{\mathbf{mean\ ablation}}$. We want to measure the degree to which $\operatorname{logitdiff}_{\mathbf{intervention}}$ approximates $\operatorname{logitdiff}_{\mathbf{mean\ ablation}}$, in a way that normalizes for different values of $\operatorname{logitdiff}_{\mathbf{mean\ ablation}}$ across cross-sections of the circuit. We use the following metric to do this:

\begin{align*}
    \operatorname{necessity\ score} = 1 -
    \frac{\left|\operatorname{logitdiff}_{\mathbf{mean\ ablation}} -
    \operatorname{logitdiff}_{\mathbf{intervention}}\right|}{\left|\operatorname{logitdiff}_{\mathbf{mean\ ablation}}
    - \operatorname{logitdiff}_{\mathbf{clean}}\right|}.
\end{align*}

\textbf{Evaluating accuracy of attribute edits.} In our figures on attribute
editing (e.g., Figure \ref{fig:edit-accuracy}), we report the proportion of
examples (in a test set not used to compute feature dictionaries) for which the
model's prediction when intervening via a given edit equals the model prediction
when we intervene by the ground-truth editing intervention described in
\ref{subsection:test-2-sparse-controllability}. This metric's ideal value is
$1$, and its worst value is zero. In many cases, simply not intervening on the
model already achieves a nontrivial (and sometimes very high) value of this
score; this is why we also report the value in the absence of intervention.

\subsection{Dataset, Model and Evaluation Details for the IOI Task}
\label{app:ioi-dataset-details}
We use GPT2-Small for the IOI task, with a dataset that spans 216 single-token names, 144 single-token objects and 75 single-token places, which are split $1:1$ across a training and test set. 
Every example in the data distribution includes (i) an initial clause introducing the indirect object (\textbf{IO}, here `Mary') and the subject (\textbf{S}, here `John'),
and (ii) a main clause that refers to the subject a second time.
Beyond that, the dataset varies in the two names, the initial clause content, and the main clause content.
Specifically, use three templates as shown below:
\begin{center}
    \text{Then, [ ] and [ ] had a long and really crazy argument. Afterwards, [ ] said to}
    \\
    \text{Then, [ ] and [ ] had lots of fun at the [place]. Afterwards, [ ] gave a [object] to}
    \\
    \text{Then, [ ] and [ ] were working at the [place]. [ ] decided to give a [object] to}
\end{center}
and we use the first two in training and the last in the test set. Thus, the test set relies on unseen templates, names, objects and places. We used fewer templates than the IOI paper \cite{wang2022interpretability} in order to simplify tokenization (so that the token positions of our names always align), but our results also hold with shifted templates like in the IOI paper.

On the test partition of this dataset, GPT2-Small achieves an accuracy of
$\approx 91\%$. The average difference of logits between the correct and
incorrect name is $\approx 3.3$, and the logit of the correct name is greater
than that of the incorrect name in $\approx 99\%$ of examples. Note that, while
the logit difference is closely related to the model's correctness, it being
$>0$ does not imply that the model makes the correct prediction, because there
could be a third token with a greater logit than both names.

\subsection{Properties of mean feature dictionaries}
\label{app:mean-codes-properties}

Mean feature dictionaries enjoy several convenient properties:
\begin{itemize}
\item The vectors $\mathbf{u}_{iv}$ for an attribute $a_i$ do not depend on which
other attributes $a_l\neq a_i$ we have chosen to describe the prompt $p$ with. 
\item If an attribute $i$ is not linearly represented in the activations, the
mean code features $\mathbf{v}_{iv}\to 0$ in the limit of infinite data (see
below). In particular, this also holds if the attribute is not represented
\emph{at all} in the activations.
\end{itemize}
This suggests that mean feature dictionaries are robust to the inclusion of irrelevant or
non(-linearly)-represented attributes, which is a desirable property in real
settings where we may not know the exact attributes present in each activation.
However, mean feature dictionaries are \emph{not} robust to the inclusion of redundant
attributes, as the lack of interaction between the attributes means that
redundant attributes cannot `coordinate' to reduce the reconstruction error
$\l\|\mathbf{a}-\widehat{\mathbf{a}}\r\|_2^2$.

\subsubsection{Mean features are zero for non-linearly-represented attributes.}
\label{subsubsection:}
Suppose we have a random vector $\mathbf{x}$ for a $k$-way classification task
with one-hot labels $\mathbf{z}\in\mathcal{Z}=\{\mathbf{z}\in\{0,1\}^k\text{ s.t. }
\l\|\mathbf{z}\r\|_1=1\}$. In Section 3 of \citet{Belrose2023LEACEPL}, it is
shown that the following are equivalent:
\begin{itemize}
\item the expected cross-entropy loss of a linear predictor
$\widehat{\mathbf{z}}=\mathbf{w}^\top \mathbf{x} + \mathbf{b}$ for $\mathbf{z}$
is minimized at a \emph{constant} linear predictor. In other words, the optimal
logistic regression classifier (in the limit of infinite data) is no better than
the optimal constant predictor (which, at best, always predicts the majority
class).  
\item the class-conditional mean vectors
$\mathbb{E}\l[\mathbf{x}|\mathbf{z}=e_i\r]$ are all equal to the overall mean
$\mathbb{E}\l[\mathbf{x}\r]$ of the data.
\end{itemize}

If we translate this to the context of mean feature dictionaries from Subsection
\ref{section:manual-sae}, we have that logistic regression for the value
of an attribute $a_i$ will degenerate to the majority class predictor if and
only if the mean feature dictionaries for all values of this attribute are zero. In the finite
data regime, this gives us some theoretical grounds to expect that the mean
feature dictionaries will be significantly away from zero if and only if the
attribute's values can be non-trivially recovered by a (logistic) linear probe.
As a special case, if an attribute is not represented in the data at all, we
expect the mean feature dictionaries for this attribute to be zero.

\subsection{Definition and Properties of MSE Feature Dictionaries}
\label{app:mse-math}

MSE feature dictionaries compute $\mathbf{u}_{a_i(\cdot)=v}$ by directly minimizing
the $\ell_2$ reconstruction error over the centered activations:
\begin{equation}
\label{eq:mse-codes}
    \{\mathbf{u}_{a_i(\cdot)=v}\}_{i\in I, v\in S_i} = \argmin_{\mathbf{u}_{a_i(\cdot)=v}} \frac{1}{N}\sum_{k=1}^{N} \left\|\l(\mathbf{a}(p_k)- \overline{\mathbf{a}}\r) - \sum_{i\in I}\mathbf{u}_{a_i(\cdot)=a_i(p_k)}\right\|_2^2
\end{equation}
This objective is convex, and is equivalent to a least-squares regression
problem; in fact, the optimal solutions take a form very similar to the mean
feature dictionaries (see below). Furthermore, this objective closely mimics the
SAE objective: here, the sparsity is hard-coded, leaving only the $\ell_2$
objective.

We next discuss some properties of the MSE feature dictionaries. For brevity, in
the remainder of this section we write $\mathbf{u}_{iv}$ instead of
$\mathbf{u}_{a_i(\cdot)=v}$.

\subsubsection{MSE feature dictionaries as a multivariate least-squares regression problem.}
\label{app:mse-math-regression}

Let $S = \sum_{i=1}^{N_A} \left|S_i\right|$ be the total number of possible values for
all attributes. For each attribute $i$, consider the characteristic matrix
$C_i\in \mathbb{R}^{N\times S_i}$ of the dataset for this attribute, where 
\begin{align*}
    C_{kj} = \begin{cases}
        1 & \text{if } a_i(p^{(k)}) = v_j \\
        0 & \text{otherwise}
    \end{cases}
\end{align*}
for some ordering $(v_1, \ldots, v_{\left|S_i\right|})$ of the values in $S_i$,
and let $ C = \begin{bmatrix} C_1 & C_2 & \cdots & C_{N_A} \end{bmatrix} \in 
\mathbb{R}^{N\times S}$ be the concatenation of all characteristic matrices.
Also, let $A\in \mathbb{R}^{N\times d}$ be the matrix of activations with rows 
$\mathbf{a}^{(k)}$. Then the objective function for the MSE feature dictionaries can be written
as the multivariate least-squares regression problem
\begin{align*}
    \min_{U\in \mathbb{R}^{S\times d}} \frac{1}{N}\left\|A - CU\right\|_F^2
\end{align*}
where the rows of $U$ are the vectors $\mathbf{u}_{iv}$ across all $i$ and
$v\in S_i$, with the optimal solution given by 
\begin{equation}
\label{eq:mse-codes-closed-form}
    U^* = \l(C^\top C\r)^{+}C^\top A
\end{equation}

\subsubsection{MSE feature dictionaries as averaging over examples.}
\label{app:mse-math-averaging}
Using the special structure of the
objective, we can also derive some information about the optimal solutions
$\mathbf{u}_{iv}^*$. Namely, at optimality we should not be able to decrease
the value of the objective by changing a given $\mathbf{u}_{iv}^*$ away from
its optimal value. The terms containing $\mathbf{u}_{iv}^*$ in the objective
are
\begin{align*}
    \frac{1}{N}\sum_{k\in P_{iv}}^{}\l\|\mathbf{a}^{(k)} - \sum_{l\neq i}^{}\mathbf{u}_{lv_l^{(k)}}^* - \mathbf{u}_{iv}^* \r\|_2^2 &= 
    \frac{1}{N}\sum_{k\in P_{iv}}^{}\l\|\l(\mathbf{a}^{(k)} - \sum_{l\neq i}^{}\mathbf{u}_{lv_l^{(k)}}^*\r) - \mathbf{u}_{iv}^* \r\|_2^2 
    \\
    &= 
    \frac{1}{N}\sum_{k\in P_{iv}}^{} \l\|\overline{\mathbf{a}}^{(k)} - \mathbf{u}_{iv}^*\r\|_2^2
\end{align*}
where recall that $P_{iv} = \{k \,|\, a_i(p^{(k)}) = v\}$, and
$\overline{\mathbf{a}}^{(k)}$ is the residual of $\mathbf{a}^{(k)}$ after
subtracting the reconstruction using all other attributes $l\neq i$. Since this
value cannot be decreased by changing $\mathbf{u}_{iv}^*$, we have that it
equals the minimizer of this term (holding $\overline{a}^{(k)}$ fixed). In other
words, if we define
\begin{align*}
    f \l(\mathbf{u}\r) = \frac{1}{N}\sum_{k\in P_{iv}}^{}
    \l\|\overline{\mathbf{a}}^{(k)} - \mathbf{u}\r\|_2^2 
\end{align*}
we have that $\mathbf{u}_{iv}^* = \argmin_{\mathbf{u}} f\l(\mathbf{u}\r)$. Since
$f$ is a sum of convex functions, it is itself convex, and so the first-order
optimality condition is also sufficient for optimality. We have
\begin{align*}
    \nabla f\l(\mathbf{u}\r) \propto \sum_{k\in
    P_{iv}}^{}\l(\overline{\mathbf{a}}^{(k)} - \mathbf{u}\r) \propto
    \frac{1}{\left|P_{iv}\right|}\sum_{k\in
    P_{iv}}^{}\overline{\mathbf{a}}^{(k)} - \mathbf{u}
\end{align*}
and so
\begin{equation}
\label{eq:mean-codes-optimality}
    \mathbf{u}_{iv}^* = \frac{1}{\left|P_{iv}\right|}\sum_{k\in P_{iv}}^{}
    \overline{\mathbf{a}}^{(k)}
\end{equation}
Note that this is very similar to the definition of mean feature dictionaries,
but also importantly different, because the optimal $\mathbf{u}_{iv}^*$ depends
on the optimal values of the feature dictionaries for the other attributes.

\subsubsection{MSE feature dictionaries with independent attributes.}
\label{app:mse-math-independent}
Finally, we can prove that,
under certain conditions, attributes for which $\mathbb{E}\left[\mathbf{a}|
a_i(p) = v_i\right] = \mathbb{E}\left[\mathbf{a}\right]$, i.e. the conditional
mean of activations over values of the attribute is the same as the overall mean
(assuming both means exist), will have (approximately) constant MSE feature dictionaries
$\mathbf{u}_{iv}=\mathbf{u}_i \forall v$. This is a counterpart to the result
from Appendix \ref{app:mean-codes-properties} for MSE feature dictionaries:

\begin{lemma}
\label{lemma:mse-codes-constant}
Suppose that all conditional means $\mathbb{E}_{p\sim \mathcal{D}}\left[\mathbf{a}| a_i(p) =
v\right]$ exist for all $i, v\in S_i$. Let $a_i$ be an attribute such its values
appear independently from the values of all other attributes, i.e.
\begin{align*}
    \mathbb{P}_{p\sim \mathcal{D}}\left[a_i(p) = v_i, a_l(p) = v_l\right] =
    \mathbb{P}_{p\sim \mathcal{D}}\left[a_i(p) = v_i\right]\mathbb{P}_{p\sim
    \mathcal{D}}\left[a_l(p) = v_l\right]\quad \forall v_i\in S_i, v_l\in S_l, l\neq i
\end{align*}
Then, in the limit of infinite training data, the conditional means
$\mathbb{E}\left[\mathbf{a}| a_i(p) = v\right]$ are all equal to the overall
mean $\mathbb{E}\left[\mathbf{a}\right]$ if and only if the optimal MSE feature dictionaries
$\mathbf{u}_{iv}^*$ for this attribute are constant with respect to the value
$v$ of the attribute, i.e. $\mathbf{u}_{iv}^* = \mathbf{u}_i$ for all $v\in
S_i$.
\end{lemma}
\begin{proof}
\label{proof:}
From Equation
\ref{eq:mean-codes-optimality}, we have
\begin{align*}
    \mathbf{u}_{iv}^* &= \frac{1}{\left|P_{iv}\right|}\sum_{k\in P_{iv}}^{}
    \overline{\mathbf{a}}^{(k)}  = 
    \frac{1}{\left|P_{iv}\right|}\sum_{k\in P_{iv}}^{}\l(\mathbf{a}^{(k)} -
    \sum_{l\neq i}^{}\mathbf{u}_{lv_l^{(k)}}^*\r)
    \\
    &= \frac{1}{\left|P_{iv}\right|}\sum_{k\in P_{iv}}^{}\mathbf{a}^{(k)} - 
    \frac{1}{\left|P_{iv}\right|}\sum_{k\in P_{iv}}^{}\sum_{l\neq i}^{} 
    \mathbf{u}_{lv_l^{(k)}}^*
\end{align*}
The first term converges to $\mathbb{E}\left[\mathbf{a}| a_i(p) = v\right]$.
The second term is a sum of terms of the form 
\begin{equation}
\label{eq:mean-codes-proof}
    \frac{1}{\left|P_{iv}\right|}\sum_{k\in P_{iv}}^{} \mathbf{u}_{l
    v_l^{(k)}}^* = \frac{1}{\left|P_{iv}\right|}\sum_{v_l\in
    S_l}^{}\mathbf{u}_{lv_l}^* \left|\{ k \text{ s.t. } a_i(p_k)=v, a_l(p_k)=v_l \}
    \right|
\end{equation}
for $l\neq i$. Since we are assuming $a_i$ is uncorrelated with $a_l$, in the
limit of the size $N$ of the dataset $\mathbf{a}^{(1)}, \mathbf{a}^{(2)},
\ldots, \mathbf{a}^{(N)}$ going to infinity, $\left|\{ k \text{ s.t. }
a_i(p_k)=v, a_l(p_k)=v_l \} \right|$ will approach $\left|P_{iv}\right|\mathbb{E}_{p\sim
\mathcal{D}}\left[\mathbf{1}_{a_l(p)=v_l}\right]$. Moreover, note that in the
closed-form solution $U^* = \l(C^\top C\r)^{+}C^\top A = \l(\frac{C^\top
C}{N}\r)^+ \frac{C^\top}{N}A$ from Equation \ref{eq:mse-codes-closed-form}, the
matrix $\frac{1}{N}C^TC$ converges to some limit $\Sigma\in \mathbb{R}^{S\times
S}$ as $N\to\infty$, and the matrix $\frac{1}{N}C^\top A$ similarly converges to
some limit $M\in \mathbb{R}^{S\times d}$ by the assumption that all conditional
means for all attributes exist. Thus, the optimal feature dictionaries $\mathbf{u}_{iv}^*$ will
also converge as $N\to\infty$. So we see that the sum in Equation
\ref{eq:mean-codes-proof} will converge to a value that is independent of the
value $v$ for the attribute $a_i$. 

Thus, if the conditional means $\mathbb{E}\left[\mathbf{a}| a_i(p) = v\right]$
are all equal to the overall mean $\mathbb{E}\left[\mathbf{a}\right]$, we get
that $\mathbf{u}_{iv}^*$ is independent of $v$; conversely, if
$\mathbf{u}_{iv}^*$ is independent of $v$, we get that the conditional means are
all equal to the overall mean. This completes the proof.

\end{proof}

\subsection{Feature-level mechanistic analyses for Section \ref{section:manual-sae}}
\label{app:feature-level-mech-methodology}

Since each activation is approximated as the sum of several vectors from a
finite set, it becomes possible to decompose the model's internal operations in
terms of elementary interactions between the learned vectors themselves. In the
current paper, we are particularly interested in attention heads, as they are
the building blocks of the IOI circuit. We consider the following subspace-level
analyses:
\begin{itemize}
\item \textbf{Attention scores}: The attention mechanism is considered to be a crucial
reason for the success of LLMs \citep{vaswani2017attention}, but a
subspace-level understanding of it is mostly lacking (but see
\citet{lieberum2023does}). How do the features in the keys and queries of
attention heads combine to produce the attention scores? Which feature pairs are
most important for the head's behavior?
\item \textbf{Head composition}: If we are to understand a circuit on the
subspace level, we need to develop a subspace-level account of how the outputs
of one attention head compose with the queries, keys and values of a downstream
head in the circuit. Each head adds its output to the residual stream, and
downstream heads' query/key/value matrices read from the residual stream. We can
thus examine the contribution, or \emph{direct effect}, of a head's output to
another head's queries/keys/values. We can decompose this direct effect in terms
of the features of the source head to calculate contributions of each feature to
the direct effect. 
\end{itemize}
Implementation details for these analyses follow.

\textbf{Attention scores.} Given feature dictionary reconstructions for the keys
and queries of an attention head at certain positions
\begin{align*}
    \mathbf{k}\approx \sum_{i\in I}\mathbf{u}_{a_i(\cdot)=a_i(p)},
    \quad \mathbf{q}\approx \sum_{i\in I}\mathbf{v}_{a_i(\cdot)=a_i(p)}
\end{align*}
we can decompose the attention scores as a sum of pairwise dot products between
the dictionary features
\begin{align*}
    \mathbf{q}^T \mathbf{k}/\sqrt{d_{head}} \approx
    \sum_{i,j\in I}\mathbf{v}_{a_i(\cdot)=a_i(p)}^T\mathbf{u}_{a_j(\cdot)=a_j(p)}/\sqrt{d_{head}}
\end{align*}
where $d_{head}$ is the dimension of the attention head.  This allows us to
examine which feature combinations are most important for the head's attention
according to the learned dictionaries. Variants of this decomposition can also
be applied to e.g. the difference in attention scores at two different token
positions.

\textbf{Head composition.} Following the terminology and results from \citet{elhage2021mathematical}, the
residual stream $\mathbf{r}_{l,t}$ of a transformer at a given layer $l$ and
token position $t$ is the sum of the input embedding and the outputs of all
earlier MLP and attention layers at this position. The residual stream is in
turn the input to the next attention layer; so, for example, we can write the
query vector for the $h$-th head at layer $l$ and token $t$ as
\begin{align*}
    \mathbf{q}_{l, t, h} = W_{l, h}^Q \operatorname{LayerNorm}\l(\mathbf{r}_{l, t}\r)
    = W_{l, h}^Q \operatorname{LayerNorm}\l(\overline{\mathbf{r}}_{l,t} +
    W^O_{l',h'}\mathbf{z}_{l',t, h'}\r)
\end{align*}
where $\mathbf{z}_{l',t, h'}$ is the attention-weighted sum of values of the $h'$-th
head at layer $l'<l$ and token $t$, $\overline{\mathbf{r}}_{l,t}$ is the remainder
of the residual stream after removing the contribution of this head, and LayerNorm is
the model's layer normalization operation \citep{Ba2016LayerN} before the
attention block in layer $l$. By treating the layer normalization as an approximately linear
operation (taking
the scale from an average over the dataset\footnote{This is justified by the
empirical observation that the layer normalization scales across the dataset are
well concentrated around their mean.}), we can derive an approximation of
the \emph{(counterfactual) direct effect} of the output of the $h'$-th head at
layer $l'$ and token $t$ on the query vector of the $h$-th head at layer $l$ and
token $t$: 
\begin{align*}
    \mathbf{q}_{l,t,h} \approx W_{l,h}^Q \l(\gamma_{l}\odot\frac{\mathbf{r}_{l, t} - \mu_{l,t}}{\sqrt{\widehat{\sigma}_l^2+\eps}} + \beta_l\r)
\end{align*}
where $\gamma_l,\beta_l$ are the learned scale and shift parameters of the LN
operation, $\mu_{l,t}$ is the average of the vector $\mathbf{r}_{l, t}$ over its
coordinates, and $\widehat{\sigma}_l$ is an average over the dataset of the
standard deviation of the residual stream at this position. Alternatively, we
can use the exact layernorm scale from the forward pass over a large sample to
compute the statistics of the exact direct effect over observed data.  

With either approach, we obtain a decomposition
\begin{align*}
    \mathbf{q}_{l,t,h} \approx \sum_{l'<l, h'}^{}\mathbf{u}_{t,(l',h')\to(l,h)} + \overline{r}_{l,t}
\end{align*}
of direct contributions from the outputs of earlier heads at this position, plus
some residual terms $\overline{r}_{l,t}$ (which are the contributions of all
previous MLP layers and the input embedding to the query vector). We can then
further decompose $\mathbf{u}_{t,(l',h')\to(l,h)}$ by replacing it with its
reconstruction from our feature dictionary.

For either way to treat the layer normalization, we can use the learned feature
dictionaries for the outputs, keys, queries and values of attention heads in a
number of ways to decompose the direct effect further. Here, we consider \textbf{head-and-feature attribution}: fixing the head $(l, h)$, we can vary the head $(l',
h')$ and break down the direct effects (projected on the query vector) by
feature.

\textbf{Results.} An interesting location to examine is the attention of
the name mover heads from END to the IO and S1 positions, where (according to
the analysis in \citet{wang2022interpretability}) the signal from the
S-Inhibition heads effectively removes the \textbf{S} name from the attention of
these heads. 

We show the results for the head L10H0 in Figure
\ref{fig:attention-decomposition}. Crucially, we observe that most interactions
are tightly clustered around zero, which suggests that these feature
dictionaries provide a sparse and interpretable account of the attention
mechanism. The only significantly nonzero interactions are (1) between the \textbf{S}
features in the query and the key at the S1 position; (2) between the \textbf{Pos}
features in the query and the key at both positions; and (3) between the \textbf{IO}
features in the query and the key at the IO position. The first two interactions
are expected given the findings of \citet{wang2022interpretability}. More
interesting is the third interaction, which is negative, suggesting that the
L10H0 head inhibits both the \textbf{S} and \textbf{IO} name tokens, and
effectively relies only on the \textbf{Pos} attribute to distinguish between the
two names. Notably, this is in contrast with the other two name movers L9H6 and
L9H9 (for which analogous plots are shown in Figure
\ref{fig:attention-decomposition-L9H6} and
\ref{fig:attention-decomposition-L9H9}), where the inhibition of the \textbf{IO}
attribute is absent. We found that using other methods to compute feature
dictionaries result in less sparse and interpretable patterns.

\begin{figure}[ht]
    \centering
    \begin{subfigure}{.48\textwidth}
        \centering
        \includegraphics[width=\linewidth]{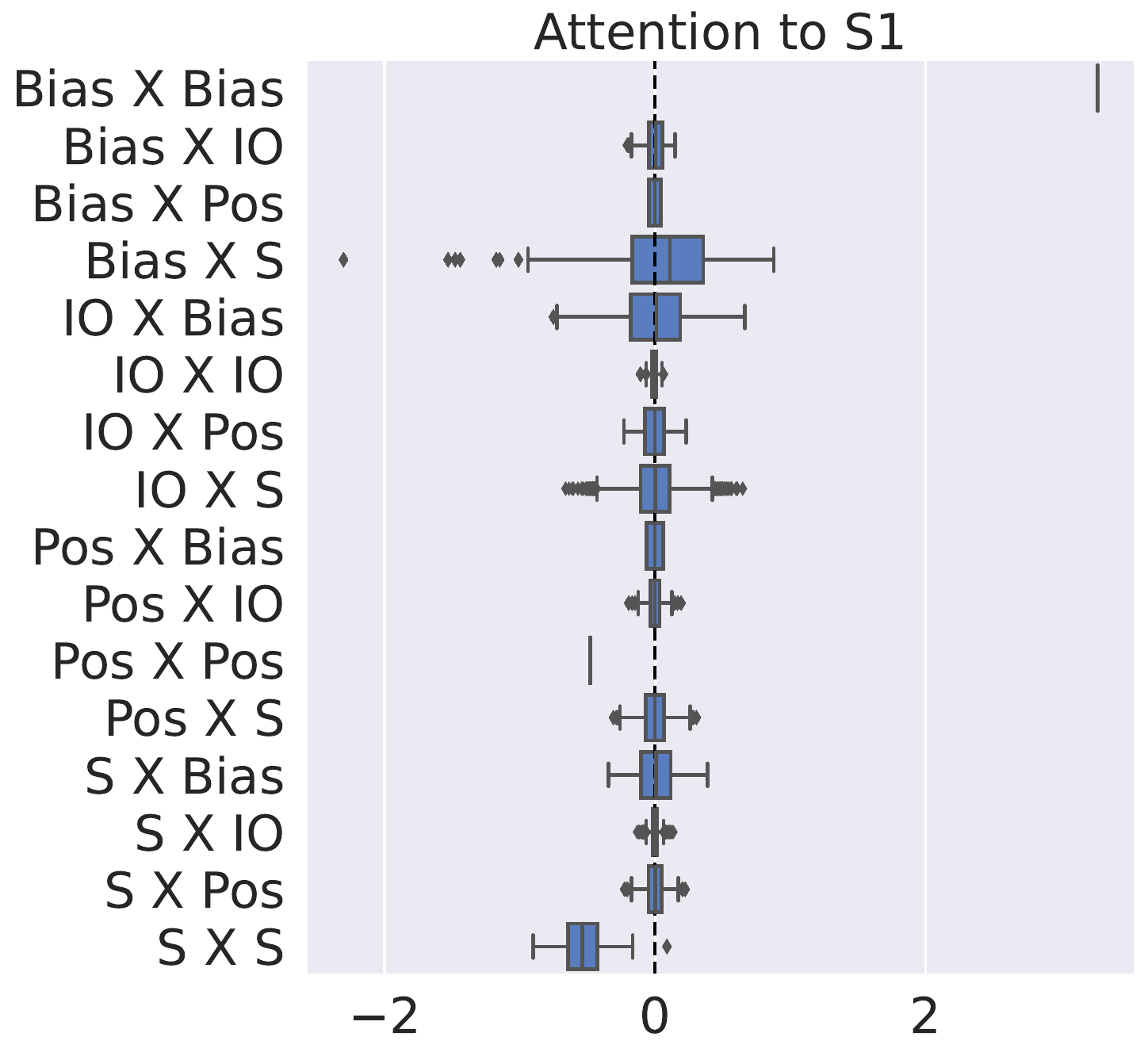}
        \label{subfig:attention-decomposition}
    \end{subfigure}%
    \begin{subfigure}{.36\textwidth}
        \centering
        \includegraphics[width=\linewidth]{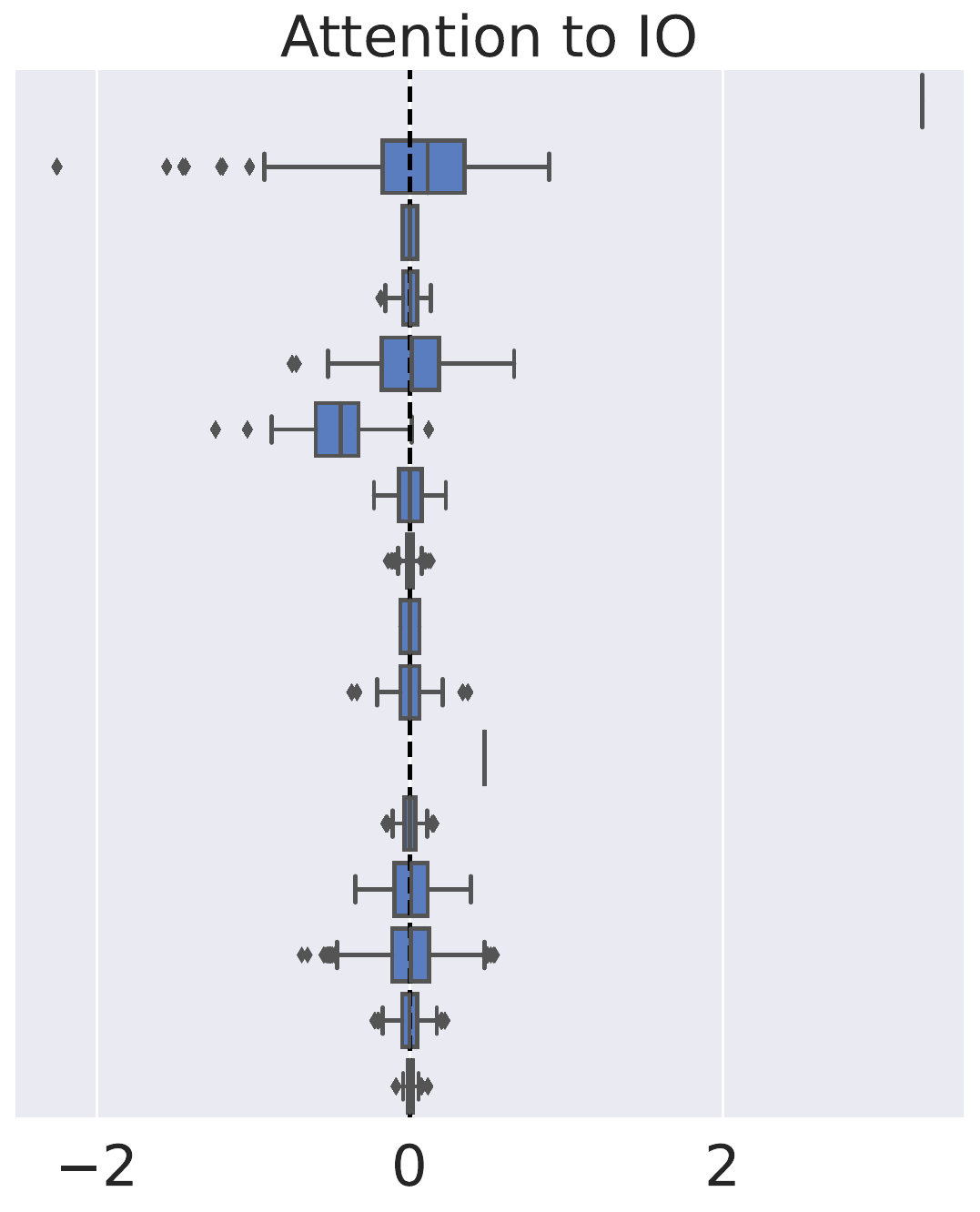}
        \label{subfig:head-feature-composition}
    \end{subfigure}%
    \caption{Decomposing the attention scores of the name mover head L10H0 from
    END to the S1 (left) and IO (right) positions. The y-axis ranges over the
    combinations of features from the query (first element) and the key (second
    element). The boxplots show the distribution of dot products between the
    corresponding feature vectors. The interaction between the bias terms (i.e.,
    the means of the respective queries/keys) provides a sense of the scale of
    the effects.}
    \label{fig:attention-decomposition}
\end{figure}

We further investigated the queries of the L10H0 head, by looking at which
features from upstream head outputs at the END token have a large direct effect
on these queries. Following the methodology detailed in Appendix
\ref{app:feature-level-mech-methodology}, we plot the direct effect from the
outputs of the S-Inhibition heads, as well as the two name mover heads L9H6 and
L9H9 in layer 9 in Figure \ref{fig:head-feature-attribution}. We find that the
S-Inhibition heads' \textbf{IO} features have no significant contribution to the
queries, but the \textbf{IO} features from the two name mover heads in layer 9
have a significant direct effect (aligned with the overall centered query
vector). This suggests that, having already computed a representation of the
\textbf{IO} attribute, these heads transmit it to the next layer, where it gets
picked up by the L10H0 head's query. Conversely, the S-Inhibition heads
contribute significantly with their \textbf{Pos} and \textbf{S} features,
whereas the name mover heads in layer 9 do not.

\begin{figure}[ht]
    \centering
    \includegraphics[width=\linewidth]{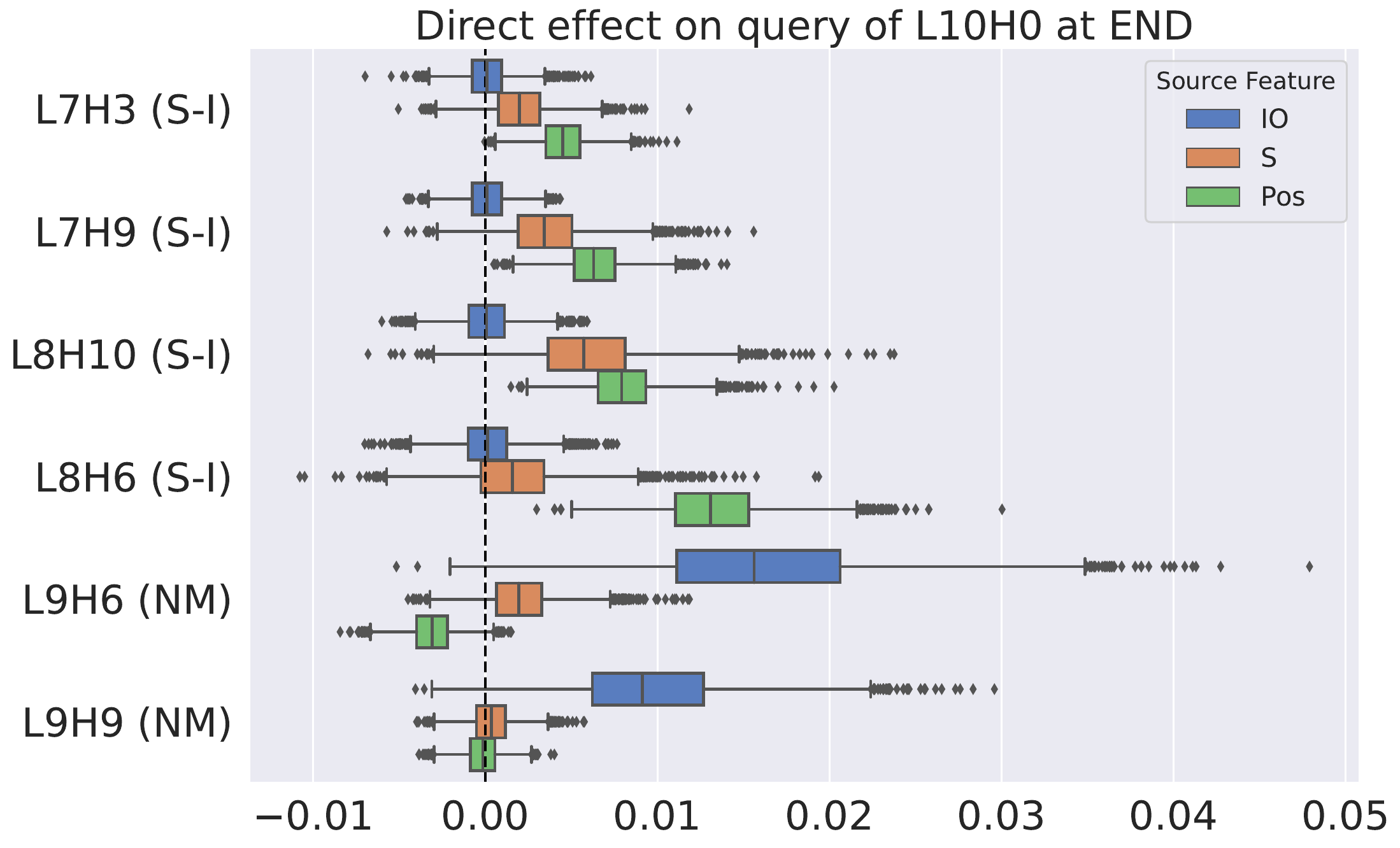}
    \caption{Direct effect of supervised features in the output of S-Inhibition
    heads, and Name Mover heads in layer 9, on the queries of the L10H0 name
    mover head at the END token.}
    \label{fig:head-feature-attribution}
\end{figure}

\subsection{Alternative parametrizations for the IOI task}
\label{app:alternative-ioi-parametrizations}
We mostly experimented with two possible parametrizations of prompts via
attributes:
\begin{itemize}
\item \textbf{independent parametrization}: we use the three independently
varying attributes -- \textbf{S}, \textbf{IO} and \textbf{Pos} -- to describe
each prompt. This is the parametrization used in the main text.
\item \textbf{coupled parametrization}: we couple position with name, and use
the two attributes (\textbf{S}, \textbf{Pos}) and (\textbf{IO}, \textbf{Pos}) to
describe each prompt. This parametrization is more expressive than the 
independent one, as it allows for different features for the same name depending
on whether it comes first or second in the sentence. At the same time, the
coupled parametrization can express the independent one as a special case
(Appendix \ref{app:coupled-vs-independent}).
\end{itemize}

We find that these parametrizations arrive at highly similar activation
reconstructions $\widehat{\mathbf{a}}$. In fact, we find an even stronger
property: the coupled parametrization essentially simulates the features in the
independent one; details are given in Appendix
\ref{app:coupled-simulates-independent}. 

Finally, we note that the fact that we find parametrizations that result in good
approximation is not trivial. Not every `natural-seeming' parametrization will
lead to a good approximation of model behavior; we show an example of this with
a `names' parametrization in Appendix Figure
\ref{fig:reconstructed-logitdiff-recovered-all}, where we instead use an
attribute for the 1st, 2nd and 3rd name in the sentence.

\paragraph{What about other parametrizations?}
Note that there are combinatorially many possible ways to pick attributes to
disentangle the activations into, and a priori any specific choice is arbitrary.
We justify our choice of parametrizations in several ways: (1) they pass our
tests for model approximation, control and interpretability given later in this
section; (2) they correspond to the internal states of the IOI circuit
identified in \citet{wang2022interpretability}; (3) we experimented and/or
considered other parametrizations, but found they either perform the same or
worse on our tests. In Appendix \ref{app:possible-decompositions}, we provide a
more detailed discussion of different possible parametrizations in the IOI task
and their relative strengths and weaknesses.

\subsection{Comparing the coupled and independent parametrizations}
\label{app:coupled-vs-independent}

\subsubsection{The coupled parametrization captures the independent one}
\label{app:coupled-simulates-independent}
\textbf{Idealized model.} We first
note that the coupled parametrization can express all reconstructions
expressible by the independent parametrization.  Suppose we have an IOI
distribution using a set of available names $S_{names}$, and let
$\mathbf{pos}_{ABB}, \mathbf{pos}_{BAB}, \{\mathbf{io}_a\}_{a\in S_{names}},
\{\mathbf{s}_a\}_{a\in S_{names}}$ be feature dictionaries for the independent parametrization
at some model activation. Then, we can define the following feature dictionaries for the
coupled parametrization:
\begin{align*}
    \mathbf{io}_{a,ABB} &= \mathbf{io}_{a} + \frac{1}{2}\mathbf{pos}_{ABB}, \quad
    \mathbf{io}_{a,BAB} = \mathbf{io}_{a} + \frac{1}{2}\mathbf{pos}_{BAB}, \quad
    \\
    \mathbf{s}_{a,ABB} &= \mathbf{s}_{a} + \frac{1}{2}\mathbf{pos}_{ABB}, \quad
    \mathbf{s}_{a,BAB} = \mathbf{s}_{a} + \frac{1}{2}\mathbf{pos}_{BAB}
\end{align*}
Then for a prompt $p$ of the form ABB (the BAB case is analogous), with the
\textbf{IO} name being $a$ and the \textbf{S} name being $b$, we have that the
reconstruction of an activation $\mathbf{a}$ using the independent
parametrization is
\begin{align*}
    \widehat{\mathbf{a}}_{independent} = \mathbf{io}_{a} + \mathbf{s}_{b} + \mathbf{pos}_{ABB}
\end{align*}
and the reconstruction using our coupled parametrization is
\begin{align*}
    \widehat{\mathbf{a}}_{coupled} &= \mathbf{io}_{a,ABB} + \mathbf{s}_{b,ABB} = \mathbf{io}_{a} + \mathbf{s}_{b} + \frac{1}{2}\mathbf{pos}_{ABB} + \frac{1}{2}\mathbf{pos}_{ABB} 
    \\
    &= \mathbf{io}_{a} + \mathbf{s}_{b} + \mathbf{pos}_{ABB} = \widehat{\mathbf{a}}_{independent}
\end{align*}

\textbf{Empirical evaluation.}
We evaluated whether this occurs empirically with the MSE features for the two
parametrizations. First, we plot the fraction of variance explained by the
reconstructions using the independent parametrization in the reconstructions
using the coupled parametrization. We find very high agreement (Figure
\ref{fig:coupled-vs-independent-var-explained}); results in the other direction
are almost identical and are not shown here for brevity.  Next, we check if the
coupled parametrization essentially simulates the independent one as described
analytically above. We do this by measuring the cosine similarity between the
vector $\mathbf{pos}_{ABB} - \mathbf{pos}_{BAB}$ from the independent
parametrization and vectors $\mathbf{io}_{a,ABB} - \mathbf{io}_{a,BAB}$ and
$\mathbf{s}_{a,ABB} - \mathbf{s}_{a,BAB}$ from the coupled parametrization. In
our idealized simulation of the independent parametrization using the coupled
one, these values would be exactly $1$ for all names. We find that in all
circuit locations that represent both \textbf{IO} and \textbf{Pos}, the
similarities w.r.t the \textbf{IO} differences are significant; similarly, in all
circuit locations that represent both \textbf{S} and \textbf{Pos}, the
similarities w.r.t the \textbf{S} differences are significant (Figure
\ref{fig:coupled-vs-independent-cosine-similarity}). For reference, in a space
of this dimensionality (64), the expected magnitude of the cosine similarity
between two random vectors is $1/8$.

\begin{figure}[h]
    \centering
    \includegraphics[width=0.75\textwidth]{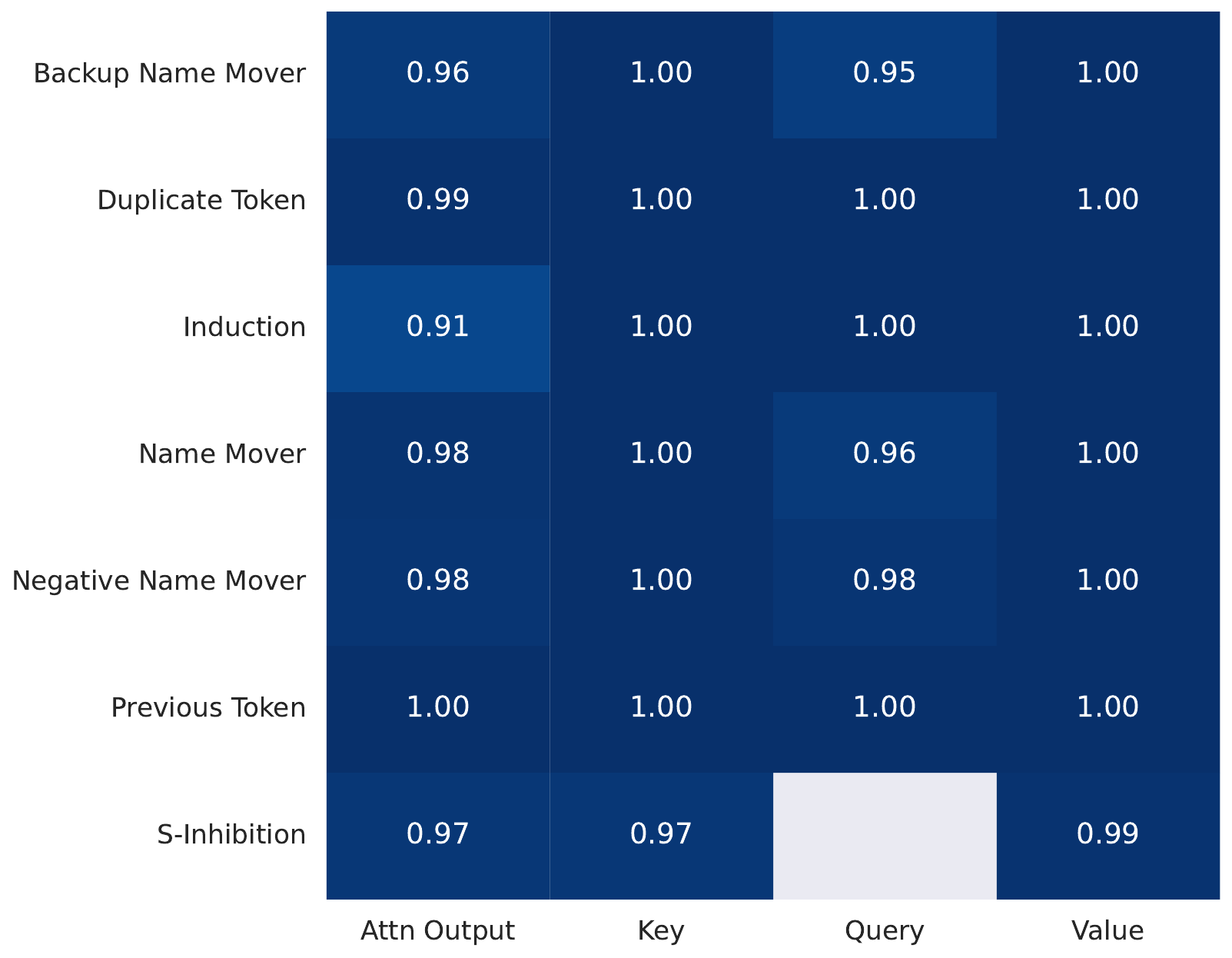}
    \caption{Variance explained by the reconstructions using the independent
    parametrization, with respect to the reconstructions using the coupled
    parametrization, averaged over combinations of class of heads in the circuit and activation types.}
    \label{fig:coupled-vs-independent-var-explained}
\end{figure}

\begin{figure}[h]
    \centering
    \includegraphics[width=0.65\textwidth]{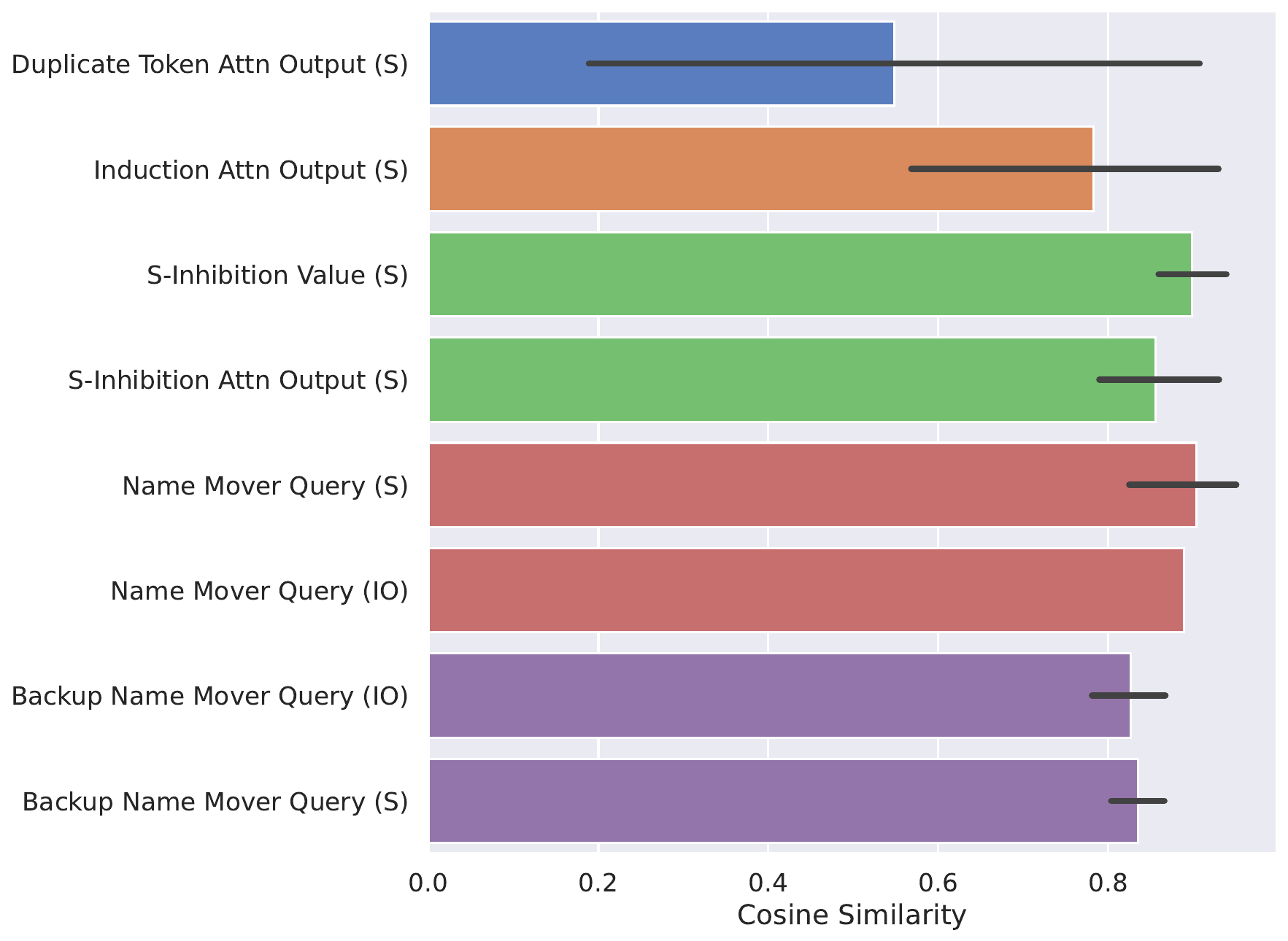}
    \caption{Cosine similarity between the vector $\mathbf{pos}_{ABB} -
    \mathbf{pos}_{BAB}$ from the independent parametrization and vectors
    $\mathbf{io}_{a,ABB} - \mathbf{io}_{a,BAB}$ and $\mathbf{s}_{a,ABB} -
    \mathbf{s}_{a,BAB}$ from the coupled parametrization, averaged over several
    classes of circuit locations. When evaluating similarity for the $\mathbf{io}_{a,ABB}-\mathbf{io}_{BAB}$ vectors, we only include locations where both the
    \textbf{IO} and \textbf{Pos} attributes are represented; and similarly for
    the $\mathbf{s}$-vectors.}
    \label{fig:coupled-vs-independent-cosine-similarity}
\end{figure}

\subsubsection{Other parametrizations expressible by the coupled parametrization.}
\label{app:coupled-simulates-other-parametrizations}
Consider the parametrization with attributes (\textbf{S}, \textbf{Pos}) and
\textbf{IO}. Again, let $\{\textbf{io}_a\}_{a\in S_{names}},
\{\textbf{s}_{a,ABB}\}_{a\in S_{names}}, \{\textbf{s}_{a,BAB}\}_{a\in
S_{names}}$ be feature dictionaries for this parametrization at some model activation.
Then, we can define the following feature dictionaries for the coupled parametrization:
\begin{align*}
    \mathbf{io}'_{a,ABB} = \mathbf{io}'_{a,BAB} = \mathbf{io}_a, \quad
    \mathbf{s}'_{a,ABB} = \mathbf{s}_{a,ABB}, \quad \mathbf{s}'_{a,BAB} = \mathbf{s}_{a,BAB}
\end{align*}
Then for a prompt $p$ of the form ABB (the BAB case is analogous), with the 
\textbf{IO} name being $a$ and the \textbf{S} name being $b$, we have that the
reconstruction of an activation $\mathbf{a}$ using the (\textbf{S},
\textbf{Pos}) + \textbf{IO} parametrization is
\begin{align*}
    \widehat{\mathbf{a}}_{(\textbf{S}, \textbf{Pos}) + \textbf{IO}} = \mathbf{io}_a + \mathbf{s}_{b,ABB} = \mathbf{io}'_{a,ABB} + \mathbf{s}'_{b,ABB} = \widehat{\mathbf{a}}_{coupled}
\end{align*}
and so the coupled parametrization can express all reconstructions expressible
by the (\textbf{S}, \textbf{Pos}) + \textbf{IO} parametrization. Analogously, it
can express all reconstructions expressible by the (\textbf{IO}, \textbf{Pos}) +
\textbf{S} parametrization.

\subsubsection{Editing methodology with the coupled parametrization.}
\label{app:coupled-editing-methodology}

For each activation, we may choose to edit one or several of the \textbf{IO},
\textbf{S} and \textbf{Pos} properties of the prompt. With the independent
parametrization, this is straightforward, since the attributes match these
properties. With the coupled parametrization, suppose we are given a prompt of
the form ABB (the BAB case is analosous) with the \textbf{IO} name being $a$ and
the \textbf{S} name being $b$. Given an activation $\mathbf{a}$ with
corresponding reconstruction under the coupled parametrization
\begin{align*}
    \widehat{\mathbf{a}} = \mathbf{io}_{a,ABB} + \mathbf{s}_{b,BAB}
\end{align*}
we can perform edits as follows:
\begin{itemize}
\item to change the \textbf{IO} name from $a$ to $a'$: $\mathbf{a}_{edit} =
\mathbf{a} - \mathbf{io}_{a,ABB} + \mathbf{io}_{a',ABB}$
\item to change the \textbf{S} name from $b$ to $b'$: $\mathbf{a}_{edit} = 
\mathbf{a} - \mathbf{s}_{b,ABB} + \mathbf{s}_{b',ABB}$
\item to change the \textbf{Pos} property from ABB to BAB: $\mathbf{a}_{edit} =
\mathbf{a} - \mathbf{io}_{a,ABB} - \mathbf{s}_{b,ABB} + \mathbf{io}_{a,BAB} +
\mathbf{s}_{b,BAB}$
\end{itemize}

\subsection{On Possible Feature Dictionaries for the IOI Task}
\label{app:possible-decompositions}

In this section, we compare the properties of several \emph{a priori} possible
ways in which the activations of the IOI circuit could be disentangled via
sparse feature dictionaries. The main goal is to illustrate that different
feature dictionaries can have similar usefulness in terms of model control and
interpretability, even if they fail natural tests that directly look for similar
features in the two dictionaries. This motivates evaluations that are agnostic
to the specific features in a dictionary, as long as the features can
parsimoniously disentangle the model's internal computations.

\textbf{The independent parametrization from the main text.} It
is worth explicitly describing the properties of the supervised feature
decomposition we constructed in Section
\ref{subsection:manual-ioi-implementation}, which uses the \textbf{IO},
\textbf{S} and \textbf{Pos} attributes to describe the prompts; it serves as an
idealized example against which to compare other possible decompositions. In
this decomposition, we can approximate an activation $\mathbf{a}$ for a prompt
$p$ where the \textbf{IO} name is $a$ and the \textbf{S} name is $b$, with the
\textbf{IO} name appearing first, as follows:
\begin{align*}
    \mathbf{a} \approx \mathbf{io}_a + \mathbf{s}_b + \mathbf{pos}_{ABB}
\end{align*}
where the vector $\mathbf{io}_a$ is the feature for the \textbf{IO} name $a$,
the vector $\mathbf{s}_b$ is the feature for the \textbf{S} name $b$, and the
vector $\mathbf{pos}_{ABB}$ is the feature for the \textbf{Pos} attribute when
the \textbf{IO} name appears first in the sentence (and analogously for
$\mathbf{pos}_{BAB}$). 

Imagine now that we are given an `unlabeled' feature dictionary that corresponds
to this decomposition (i.e., we don't know which attribute each feature
corresponds to). We want to evaluate the usefulness of this decomposition for
model control and interpretability relative to the `human-legible' attributes
\textbf{IO}, \textbf{S} and \textbf{Pos}. This will be tautologically successful:
\begin{itemize}
\item the reconstructions are a faithful and complete representation of the
model's internal computations;
\item the dictionary can (by definition) express edits to the \textbf{IO},
\textbf{S} and \textbf{Pos} attributes very efficiently, as we only need to
change a single feature vector to change the corresponding attribute's value.  
\item the features are fairly interpretable: we can understand the meaning of
each feature in terms of the attribute it represents. 
\item moreover, using metrics such as recall and precision (following the
evaluation methodology of \citet{bricken2023monosemanticity}) will readily
surface the features that are most important for each attribute.
\end{itemize}

\textbf{Using per-gender vectors to describe names.} Another possible
decomposition is
\begin{align*}
    \mathbf{a} \approx \mathbf{io}_a + \mathbf{io\_gender}_{g(a)} + \mathbf{s}_b
    + \mathbf{s\_gender}_{g(b)} + \mathbf{pos}_{ABB}
\end{align*}
where $g:\mathbf{Names}\to \{M, F\}$ is some labeling function that roughly
classifies names according to how they are typically gendered\footnote{We
experimented with this decomposition in our supervised framework, but did not
find it to confer additional benefits for the purposes of our tests.}.
Here, we hypothesize that the model may use features of high norm that sort
names into genders (which may be useful to the model for various reasons), and
then add a small-norm per-name correction to obtain a name-specific representation.
In particular we imagine that $\mathbf{io}_a + \mathbf{io\_gender}_{g(a)}$ in
this representation would correspond to $\mathbf{io}_a$ in the supervised
decomposition, and similarly for the \textbf{S} name.
\begin{itemize}
\item This decomposition is also fairly sparse and interpretable, and it can
express edits almost as parsimoniously as the supervised decomposition (we only
need to change \emph{two} feature vectors to edit a name, and one to edit
\textbf{Pos}).
\item Moreover, metrics such as recall and precision will pick up on the 
per-name features;
\item If we compare the prompts for which a feature activates for this
dictionary and our `independent parametrization' feature dictionary (discussed
above), we will easily discover the per-name features that correspond to the
\textbf{IO} and \textbf{S} names.
\item \textbf{However}, if we instead directly use cosine similarity to the
supervised features $\mathbf{io}_a, \mathbf{s}_b$ as a metric, we may be misled,
because if the per-gender features have sufficiently higher norm than the
per-name corrections, the cosine similarity may be low.
\end{itemize}

\textbf{Features for small, overlapping subsets of names.} Going further, we can
imagine a decomposition where we have features that correspond to pairs of
names, such that each name is in exactly two pairs (this can be achieved by
partitioning all names into pairs along a cycle). We can express a name as a sum
of the features for the two pairs it is in, with some superposition. Note that
more sophisticated constructions with more features per name / more names per
feature are possible by e.g.  picking subsets at random or using expander graphs
\citep{hoory2006expander}, as they will `spread out' the superposition more
evenly. 
\begin{itemize}
\item This decomposition is \emph{somewhat} sparse and interpretable, and can
likely be used for feature editing in a reasonable way, as long as the sets of
features associated with each name are not too large. Even though we would need
to change several feature vectors to edit a name, there should also be a fair
amount of disentanglement so that we don't also need to throw away all the 
features active in an example to change a single attribute.
\item \textbf{However}, comparing our supervised decomposition against this one using
geometric metrics such as cosine similarity may be misleading, because while a
sum of a few feature vectors associated with the same name may point in the same
direction as our supervised feature, any individual feature may not.
\item \textbf{Furthermore}, it also has significantly reduced precision for the
features, because each of the few features associated with a name will also
activate for several other names. This can make directly looking for features
whose activation patterns resemble the ones in our supervised decomposition
misleading.
\end{itemize}

Our experiments suggest that both task-specific and full-distribution SAEs
trained on IOI circuit activations learn a decomposition resembling this
abstract construction more than any other decomposition discussed here.

\textbf{Overfitting dictionaries.} Finally, a worst-case decomposition would be
to have a single feature for each possible set of values of the \textbf{S},
\textbf{IO} and \textbf{Pos} attributes.

\begin{itemize}
\item This decomposition is not interpretable, and it is not editable in any
non-trivial way: to change a single attribute, the entire decomposition must be
replaced;
\item Features of this form will have maximum precision, but very low
recall for the attributes.
\end{itemize}

\subsection{Details for training Sparse Autoencoders}
\label{app:sae-training-details}

\textbf{Task SAEs.}
We followed the methodology of \citet{bricken2023monosemanticity}, with the
exception that our neuron re-initialization method is not as sophisticated as
theirs: we simply re-initialize the encoder bias, encoder weights and decoder
weights for the dead neurons every 500 training epochs.

Training SAEs on the IOI distribution alone allows us to do a more extensive
hyperparameter search. Importantly, we normalized SAE inputs across attention
heads so that activations have an $\ell_2$ norm of $1$ on average in order to
make it easier for the same set of hyperparameters to work well across different
heads. In our main experiments, we use SAEs that were trained with a $16\times$
hidden expansion factor, (effective) $\ell_1$ regularization coefficient in
$(0.01, 0.05, 0.1, 0.3)$, batch size of $1024$, and learning rate of $0.001$.

We evaluated two key test-set metrics across training epochs: the average number
of active features per example (i.e. the average $\ell_0$ norm of activations),
and the fraction of the logit difference recovered when using the SAE
reconstructions at the given model location instead of the original activations,
scaled against a mean-ablation baseline (which is more stringent than the
zero-ablation baseline employed in most other work). We chose the regularization
coefficient and training checkpoint for each node that provided a good trade-off
between these two metrics; in particular, for almost all nodes, we recover logit
difference to within $20\%$ with respect to the mean-ablation baseline, and
there are $<25$ active features per example. We provide the results of this
sweep in Figure \ref{fig:sae-sweep}. While most SAEs seemed adequate, some
still have poor approximation as measured by the reconstructed logit difference.

We use a training set of 20,000 examples and an evaluation set of 8,000 examples
(for the purposes of automatic interpretability scoring, we need a large enough
evaluation sample so that each property in the distribution appears a
significant number of times). Since our training regime is significantly
distinct from that of \citet{bricken2023monosemanticity} (we use a much smaller
dataset), we first experimented extensively with different hyperparameters,
focusing on training SAEs on the queries of the name mover heads. We observed
that the most important hyperparameters are the dictionary size and the
effective $\ell_1$ regularization coefficient. We found that the batch size did
not influence the eventual quality of the learned features, only the speed of
convergence, and that a learning rate of $10^{-3}$ (as in
\citet{bricken2023monosemanticity}) was a good choice throughout. The runs
reported here used a dictionary size of $1024$ (a $16\times$ increase over the
dimensionality of attention head activations in GPT-2 Small), an effective
$\ell_1$ regularization between $0.05$ and $0.3$, and a batch size of $1024$.

We normalized activations across the circuit to make it
easier for the same range of hyperparameters to give good results, and ran a
sweep over $\ell_1$ regularization coefficients in $(0.01, 0.05, 0.1, 0.3)$. 

\textbf{Full-distribution SAEs.}
We trained full-distribution SAEs on every IOI component using
\textsc{OpenWebText} as training data. We mostly followed the method outlined in
\citet{bricken2023monosemanticity}. We added a
standardization procedure to be able to train SAEs on components with different
activation scales using the same l1-coefficient. Before training, we calculated
the mean and the mean l2-norm over 10 million activations. These values were
then frozen and used to standardize all activations as a preprocessing step and
to rescale the SAE reconstructions to match the original scale as a
post-processing step. We generated the training dataset by extracting a buffer
of 10 million activation vectors from the shuffled \textsc{OpenWebText} dataset
at a time with a maximal context window of $512$ tokens. We then trained the SAEs
for 250 million activation vectors and resampled dead neurons after 50000 steps
(around 100 million activations) as outlined in
\cite{bricken2023monosemanticity}. We used a batch size of 2048 and 8192
features per SAE. Post-training, we excluded dead and ultra-low frequency
neurons that we define as neurons who activate less than once per million
activations. The amount of dead neurons varies across SAEs between 20 and 90\%. 
We plot the fraction of dead neurons versus $\ell_0$ loss in Figure
\ref{fig:l0-vs-dead-neurons}, and the loss recovered versus $\ell_0$ loss in
Figure \ref{fig:l0-vs-loss-recovered}.

We used an $\ell_1$ coefficient of $0.006$ initially for all SAEs, and retrained SAEs
with a different $\ell_1$ coefficient for crosssections whose SAE metrics were
undesired (e.g. very low $\ell_0$ / bad reconstruction or very high $\ell_0$).
For the name mover outputs, we used $0.025$, and for S-Inhibition keys we used
$0.005$. The test $\ell_0$ and loss-recovered metrics were calculated on $81920$
unseen activation vectors.

We trained fewer SAEs on the full pre-training distribution compared to the IOI
distribution, as the computational cost is higher. We observe that SAEs with a
lower number of active features per example generally perform better for
IOI-related tests, even if their other metrics (such as loss recovered on
\textsc{OpenWebText}) are worse.

\textbf{SAE training loss metrics.} The most important loss metrics to track
during SAE training are the $\ell_0$ loss (measuring the average number of
active features per activation) and the language model loss recovered when using
the learned features to reconstruct the model's logits
\citep{bricken2023monosemanticity}. To turn the loss recovered into a meaningful
quantity, it is rescaled against a baseline; both zero ablation and mean
ablation have been used for this purpose in the literature
\citep{bricken2023monosemanticity,gpt2_attention_saes}. In this work, we used
mean ablation, as it is a more strict test of the quality of approximation.

\subsection{Additional notes on methodology for SAE interpretability}
\label{app:sae-interp-methodology}

\textbf{Interpretations considered.} Let \textbf{Names} be
the set of names in our IOI dataset.  We consider the following binary
predicates over prompts as possible interpretations for SAE features in the
activations at the S2 and END tokens:
\begin{itemize}
\item \textbf{IO is 2nd name}: the \textbf{Pos} attribute having
value corresponding to BAB-type prompts;
\item \textbf{IO is 1st name}: the \textbf{Pos} attribute having value
corresponding to ABB-type prompts;
\item \textbf{S is $<$name$>$}: the \textbf{S} attribute has a certain value in \textbf{Names};
\item \textbf{S is $<$name$>$ and at 1st position}: same as above, but also the
\textbf{S} name is at 1st position in the prompt (i.e., this is a BAB-type
prompt);
\item \textbf{S is $<$name$>$ and at 2nd position}: same as above, but for
ABB-type prompts;
\item \textbf{IO is $<$name$>$}: the \textbf{IO} attribute has a certain value in \textbf{Names};
\item \textbf{IO is $<$name$>$ and at 1st position}: same as above, but also the
\textbf{IO} name is at 1st position in the prompt (i.e., this is a ABB-type
prompt);
\item \textbf{IO is $<$name$>$ and at 2nd position}: same as above, but for
BAB-type prompts;
\item \textbf{S is male}: the \textbf{S} name is labeled as a male name under
our labeling of \textbf{Names} provided by GPT-4;
\item \textbf{S is female}: same as above for female names;
\item \textbf{$<$name$>$ is in sentence}: a certain name in \textbf{Names};
\item \textbf{$<$name$>$ is at 1st position}: same as above, but the name is the
first name in the sentence;
\item \textbf{$<$name$>$ is at 2nd position}: same as above, but the name is the
second name in the sentence;
\end{itemize}

The next several interpretations are only defined for the keys and values of the
name mover heads. We collect together activations for the keys according to
their role in the IOI circuit as opposed to absolute position: we group together
all activations at the IO token position (these are the \textbf{IO}
keys/values), even though they come from different absolute positions across IOI
prompts, because in ABB prompts the \textbf{IO} name comes first, while in BAB
prompts it comes second. The same applies for gathering the \textbf{S}
keys/values.

The key/value activations described above have not yet `seen' the repeated name
in the sentence, so there is no meaningful concept of \textbf{IO} and \textbf{S}
for them. Instead,  the only potentially task-relevant information contained in
them is about the name(s) seen so far in the sentence, and the position (1st
name or 2nd name) where the activation is taken from. Accordingly, the
applicable interpretations for features contained in these activations are
different:
\begin{itemize}
\item \textbf{current token is $<$name$>$}: the token from which the activations
are taken holds a certain value in \textbf{Names}.
\item \textbf{token is $<$name$>$ and at 1st position}: the activation was taken
from a token with a certain value in \textbf{Names}, and in addition it comes
from the first name in the sentence;
\item \textbf{token is $<$name$>$ and at 2nd position}: same as above, but
activation is from the second name in the sentence;
\item \textbf{current token is at 1st position}: the activation is from the first name in the sentence;
\item \textbf{current token is at 2nd position}: same as above, but from second
name;
\item \textbf{current token is female}: the token the activation is from is
female under our labeling of \textbf{Names}.
\end{itemize}

\emph{Unions of interpretations}. In addition, for each type of predicate that
has a free parameter in \textbf{Names}, we considered unions of up to 10 such
predicates (recall that we have a total of 216 names in our dataset). We ordered
the individual predicates according to their $F_1$ score, and chose the union of
the first $k\leq 10$ predicates with the highest $F_1$ score as a possible
interpretation. Note that the $F_1$ score is not in general a monotone function of $k$ in this setup; indeed, we find that for many features the highest-$F_1$-score explanation uses $k<10$ features.

\textbf{Sufficiency/necessity of interpretable features.} We take the
interpretations of the features described above and their respective $F_1$
scores, and for each threshold $t\in [0,1]$ over $F_1$ scores consider two
interventions:
\begin{itemize}
\item \textbf{sufficiency}: we subtract from the respective activation all
active features with $F_1$ score $<t$;
\item \textbf{necessity}: we subtract from the respective activation all 
active features with $F_1$ score $\geq t$;
\end{itemize}

\textbf{Interpretation-aware sparse control.} The goal of this experiment is to
evaluate the usefulness of our feature interpretations for editing the
attributes \textbf{IO}, \textbf{S} and \textbf{Pos} we have chosen to describe
prompts with. In particular, this is a different goal from our exploratory
interpretability experiment, where we were concerned with assigning
interpretations to each feature in a way agnostic to whether the features
correspond 1-to-1 with the attributes. 

Correspondingly, for this experiment we use the $F_1$ score with respect to each
attribute as a guide for the potential relevance of a feature to this attribute.
Specifically, recall from Subsection \ref{subsection:test-3-interpretability}
that we can assign an $F_1$ score to each combination of a feature and a
possible value of one of the attributes \textbf{IO}, \textbf{S} and
\textbf{Pos}. 

To perform editing, let our SAE have a dictionary of decoder vectors
$\{\mathbf{u}_j\}_{j=1}^m$, and the original and counterfactual prompts
$p_s,p_t$ have respective activations $\mathbf{a}_s, \mathbf{a}_{t}$ with
reconstructions
\begin{align*}
    \widehat{\mathbf{a}}_s = \sum_{i\in S} \alpha_i\mathbf{u}_i + \mathbf{b}_{dec}, \quad 
    \widehat{\mathbf{a}}_{t} = \sum_{i\in T} \beta_i\mathbf{u}_i + \mathbf{b}_{dec}
\end{align*}
for $S, T \subset \{1, \ldots, m\}$ and $\alpha_i, \beta_i > 0$. Suppose the
original and counterfactual prompts differ only in the value of the attribute
$a$ we wish to edit, with $a(p_s) = v_s, a(p_t) = v_t$.
Using a test set, we estimate the $F_1$ score of each possible value $v$ of the
attribute $a$, and for each $v$ we order the SAE features in decreasing order of
the $F_1$ score, obtaining a list $\operatorname{top-features}\l(v\r) =
\left[j_1, j_2, \ldots\right]$ with $1\leq j_k\leq m$. 

Then, fixing some value of $k$, we compute the edited activation as
\begin{align*}
    \mathbf{a}_{edited} = \mathbf{a}_s - \sum_{j\in (S\cap\operatorname{top-features}(v_s))\left[:k\right]}^{}\alpha_j\mathbf{u}_j + \sum_{j\in (T\cap\operatorname{top-features}(v_t))[:k]}^{}\beta_j\mathbf{u}_j.
\end{align*}
Note that this is a somewhat \emph{less expressive} intervention than our
per-prompt agnostic feature editing, because here we require that there is some
fixed ordering from which we choose features, instead of being free to pick
features for each prompt independently.

\subsection{Additional observations on feature interpretations}
\label{app:interp-extra-observations}

\textbf{Interpretable features agree with head roles identified in the IOI
circuit by \citet{wang2022interpretability}.}
For example, duplicate token heads attend to a previous occurrence of the
previous token and write information about this to the residual stream.
Consistent with this, we found that SAEs trained on them contain features that
indicate the duplicated name, the subject in case of IOI. This information is
then used by the induction heads that determine whether the subject is the first
or second name. The subject position is subsequently copied to the END position
by the S-Inhibition heads, where it will query the name movers to \emph{not}
attend to the subject, and to copy and predict the IO name. Indeed, we find
features of the outputs of induction heads, in the outputs and values of the S-I
heads, and in the queries of the name movers that inform about the position of
the indirect object. Lastly, the name movers attend to the IO position and copy
the name to predict it. As anticipated, the name mover values and outputs
contain features that specify the concrete IO name. In summary, the type of
features detected informs well about the function of the head on a certain task.

\textbf{New insights from the feature dictionaries.} 

\begin{itemize}
    \item The first layer DT-heads almost exclusively contain \textbf{S}
    features but the third layer duplicate token head also has positional
    features, suggesting more sophisticated text processing happening there.
    \item Induction head encompass different positional features. L6H9 features
    inform about what name is at the second position, while L5H5 and L5H8
    activate when the \textbf{IO} is at the first position. L5H9 is comprised of
    different features including positional features that don't inform about the
    role (\textbf{IO} vs \textbf{S}) of the name.
    \item L7H3 is the only head that contains a significant amount of gender
    features.
    \item S-Inhibition heads include primarily features that are a combination
    of names and \textbf{S}, while the name mover queries seem to only contain
    \textbf{Pos} features.
    \item The keys and values of name movers both inform about the token at the
    current position, but there is an important difference in the type of
    features: while the values primarily contain features that contain the name
    (that later gets moved to the END position), the keys consist of positional
    features and combinations of position and name. This hints at an important
    mechanism where the name mover query contains positional information about
    the \textbf{S} name that gets matched with the corresponding key,
    effectively shifting the attention towards the IO token such that the value
    with the \textbf{IO} gets copied to the END position.
\end{itemize}

\textbf{Editing interpretable features}

While the feature descriptions generated through our automatic scoring predict well when a feature is active, it is still unclear whether they also have an interpretable causal role, i.e. whether activating or deactivating a certain feature leads to a change in output logits that would be expected from the feature's description. To test this, we propose two experiments to judge the faithfulness and completeness of our interpreted features that involve patching activations from a counterfactual prompt and calculating the effect on the model's output:

\begin{itemize}
    \item \textbf{Estimating Faithfulness:}
    To estimate faithfulness, we construct SAE activations where we fix features with a test F1-score smaller than a threshold and patch activations of features with a high F1-score from the counterfactual prompt. We then calculate reconstructions of this new SAE activation vector, patch cross-sections, and record whether the model successfully predicts the correct counterfactual name.
    \item \textbf{Estimating Completeness:}
    To estimate completeness, we propose a similar experiment where we fix features with a high test F1-score and patch all remaining features. This intervention should not change the output logits if our features are complete.
\end{itemize}

We run this experiment on cross-sections of name mover outputs and repeat this experiment for different thresholds. We observe that for a threshold F1-score of 0.6, the SAE features are both faithful and complete to a high degree. We observe that the faithfulness metric significantly decreases for higher F1-scores of 0.7 and 0.8 but remarkably, we also observe that only fixing features with a very high F1-score of $>0.8$ while patching all other features from the counterfactual prompt is sufficient to keep the model predicting the base prompt's output.

\textbf{Generalization of features to the full distribution.}
We sample prompts from OpenWebText and visualize prompts that highly activate a
feature. We do this for the name mover queries and outputs. For the outputs, we
calculate the direct-feature-attribution (DFA) metric 
first proposed by \citet{gpt2_attention_saes_3}
that calculates per position how large the contribution of its
values to activating the feature is. Thus, it informs what position led to the
feature being activated.

We find that features mostly generalize. For example, we investigated a feature
that activated if the IO name starts with the letter "E" and we found that on
full distribution, this feature fires at tokens preceeding words starting with
"E". DFA suggests that previous tokens starting with "E" activate this feature,
and calculating the unembed $W_{dec}[j] W_O W_U$, denoted with
"Positive Logits" and "Negative Logits" in Figure
\ref{fig:full-sae-example-prompts} shows that activating the feature increases
logits for words starting with "E". Together, this hints at a general name
moving mechanisms to predict words that previously occurred in the context that
drives in-context learning where the head's QK-circuit drives attention to the
position of the word to predict, and the OV circuit copies the token from the
previous to the current position to predict it.

\begin{figure}[ht]
    \centering
    \includegraphics[width=\textwidth]{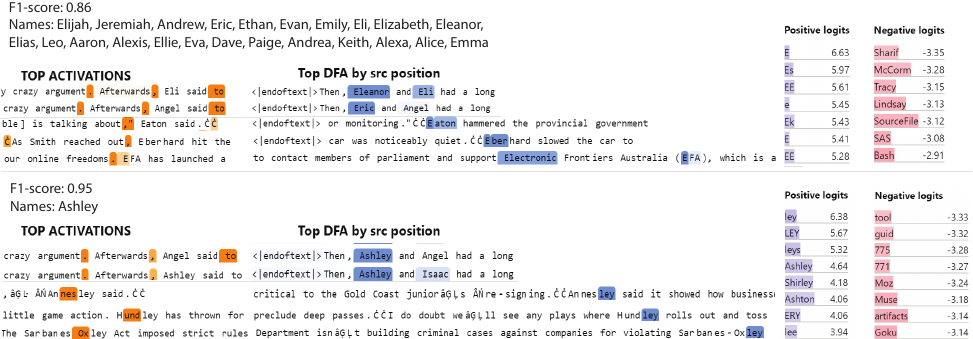}
    \caption{Two representative features discovered in the output name mover SAE
    L9H9 to illustrate the features behavior or webtext. Both features are IO
    name features, the upper one containing 23 names, the lower one only a
    single name. Left: Feature activation per position; Middle: Direct Feature
    Attribution (DFA) that tracks the position whose values contribute most to
    activating the given feature; Right: The output tokens whose logits get
    increased (positive) or decreased (negative) when the feature is active,
    calculated by $\textbf{v} \textbf{W}_O \textbf{W}_U$ with $\textbf{v}$ being
    the decoder weight vector of the feature of interest, $\textbf{W}_O$ the
    output weight of the head and $\textbf{W}_U$ the unembed}
    \label{fig:full-sae-example-prompts}
\end{figure}

\subsection{Feature occlusion details}
\label{app:occlusion}

\textbf{Training an exhaustive set of SAEs.} Focusing on the L10H0 queries, we
found that our SAEs consistently find a single feature for almost all
\textbf{IO} names, but fail to find a significant number of features for
individual \textbf{S} names. 

To investigate further whether this is a result of poor hyperparameter
choices, we trained SAEs on the queries of L10H0 over a wide grid of
hyperparameters, so that performance deteriorates/plateaus at the edges of the
grid.  We pushed the dictionary size, training dataset size, $\ell_1$
regularization coefficient and number of training epochs significantly beyond
the values used in our other experiments. Specifically, we used a training set
of 100,000 examples (this is more datapoints than all $\sim93,000$ possible
combinations of \textbf{S}, \textbf{IO} and \textbf{Pos} in our data); we
trained dictionaries of up to $128*64=8192$ features (our supervised feature
dictionaries contain $\sim500$ features); we varied the effective $\ell_1$
regularization coefficient across two orders of magnitude; and we trained for up
to $6,000$ epochs. 

Results on the number of \textbf{IO}/\textbf{S} features with $F_1$ score $>0.5$
are reported in Figure \ref{fig:occlusion-main} (left). We observe that
across hyperparameters, we often find as many \textbf{IO} features as there are
names in our dataset; however, the number of \textbf{S} features is consistently
low, never exceeding $22$.

\textbf{Magnitudes of the IO and S features.} As a first step, we investigate 
the distribution of the norms of the features for \textbf{IO} and \textbf{S}
across the names in the IOI dataset in the L10H0 queries in our supervised
feature dictionaries from Section \ref{section:manual-sae}; results are shown in
Figure \ref{fig:occlusion-extra} (left). We observe that the norms of the
\textbf{IO} features are significantly (but not by much) higher than those of
the \textbf{S} features.

\begin{figure}[ht]
    \centering
    \begin{subfigure}{.45\textwidth}
    \centering
    \includegraphics[width=\textwidth]{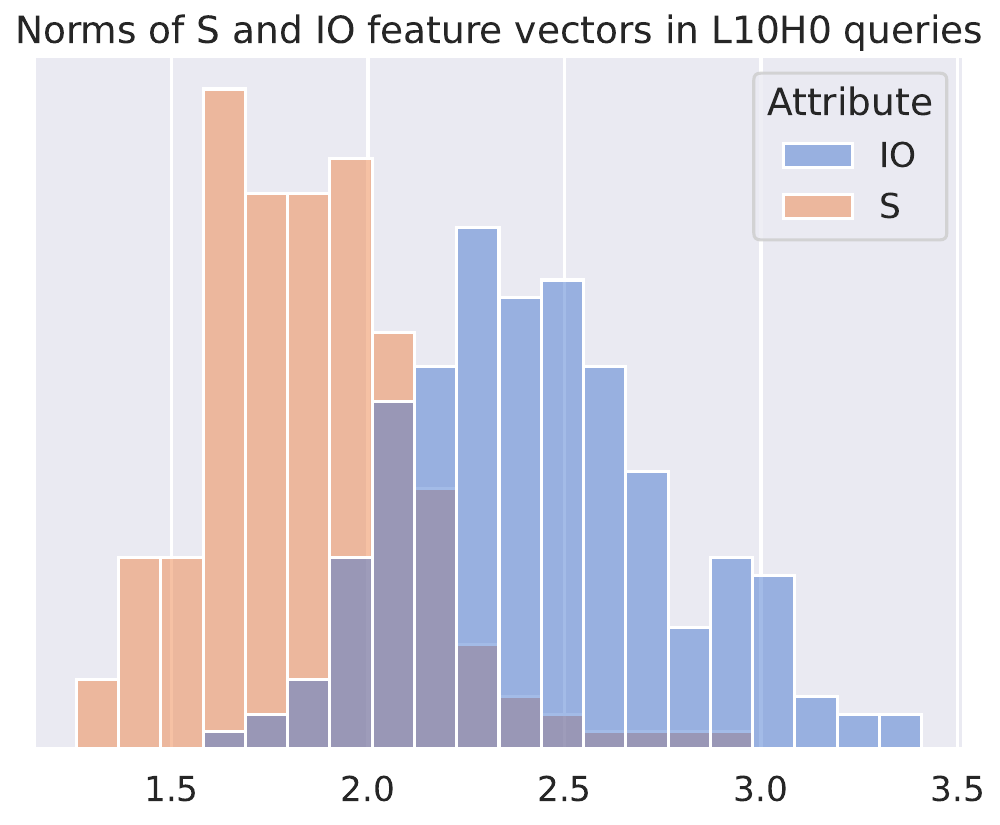}
    \end{subfigure}%
    \hfill
    \begin{subfigure}{.5\textwidth}
        \centering
        \includegraphics[width=\linewidth]{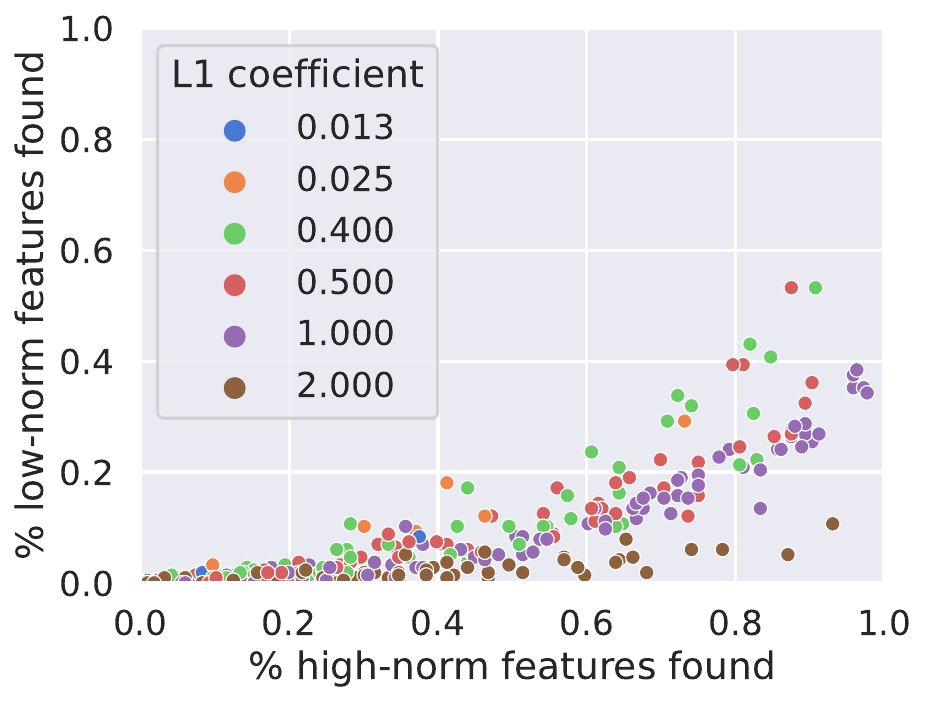}
    \end{subfigure}%
    \caption{
        \textbf{Left}: distributions of the $\ell_2$ norms of the feature
    vectors for the \textbf{IO} and \textbf{S} attributes from our supervised
    feature dictionary for the queries of the L10H0 name mover.
    \textbf{Right}: the results of the toy model experiment, where we
    investigate whether a disparity in feature magnitudes alone can lead to the
    occlusion phenomenon. The $x$-axis shows the fraction of high-magnitude
    ground-truth features for which we find an SAE feature with $F_1$ score
    $>0.9$; the $y$-axis shows the same for the low-magnitude ground-truth
    features.}
    \label{fig:occlusion-extra}
\end{figure}

\textbf{Surgically reducing the magnitude of \textbf{IO} features.} To examine
the role of feature magnitude in the occlusion phenomenon, we continuously
reduce the magnitude of \textbf{IO} features. Namely, given our supervised
feature dictionary with features $\{\mathbf{io}_a\}_{a\in \text{Names}}, 
\{\mathbf{s}_b\}_{b\in \text{Names}}, \{\mathbf{pos}_{v}\}_{v\in \{\text{ABB},
\text{BAB}\}}$ and an activation $\mathbf{a}(p)$ for a prompt $p$ where the 
\textbf{IO} name is $a$, we construct a new activation 
\begin{align*}
    \mathbf{a}_{\alpha}\l(p\r) = \mathbf{a}\l(p\r) - \alpha \mathbf{io}_a
\end{align*}
for $\alpha\in[0,1]$. We find that, applying this intervention without any
hyperparameter tuning (with a modest dictionary size of $1024$, and $\ell_1$
regularization coefficient of $0.2$), increasing $\alpha$ from $0$ to $1$
gradually makes the number of \textbf{IO} features with $F_1$ score $>0.5$
to decrease, while the number of \textbf{S} features with $F_1$ score $>0.5$
increases; results are shown in Figure \ref{fig:occlusion-main} (right).

\textbf{Reproducing the occlusion phenomenon in a toy model.} Finally, we wanted
to know if a disparity in feature magnitudes alone could lead to the occlusion
phenomenon. We constructed a simple toy model closely based on the empirical
setup in the queries of the L10H0 head to test this hypothesis. 

We form synthetic activations $\mathbf{a} = \mathbf{u}_i + \mathbf{v}_j$ for
random pairs $i, j \in \{1, \ldots, \left|\text{Names}\right|\}$, where
$\mathbf{u}_i, \mathbf{v}_j\in \mathbb{R}^{d_{head}}$ are sampled independently
from a standard normal distribution centered at zero, and then $\mathbf{u}_i$
are rescaled so that their mean norm matches the mean norm of \textbf{IO}
features in the L10H0 queries, and similarly for $\mathbf{v}_j$ and \textbf{S}
features. We train SAEs on these activations over a wide grid of hyperparmeters:
dictionary sizes in $(512, 1024, 2048)$, $\ell_1$ regularization in $(0.0125,
0.025, 0.4, 0.5, 1.0, 2.0)$, batch size in $(256, 1024)$ and learning rate in
$(0.001, 0.003, 0.0003)$. We trained for $1000$ epochs, saving checkpoints in a
geometric progression of epochs. Results for the number of high- and
low-magnitude features with $F_1$ score $\geq 0.9$ discovered are shown in
Figure \ref{fig:occlusion-extra} (right); we observe that we easily find one SAE
feature per each high-magnitude ground-truth feature, but it is more difficult
to find an SAE feature for each low-magnitude feature. 

However, we note that with lower $F_1$ thresholds, this effect is less
pronounced and eventually disappears.

\subsection{Feature over-splitting in a mixture of gaussians toy model}
\label{app:over-splitting}

\begin{figure}[ht]
    \centering
    \includegraphics[width=0.75\textwidth]{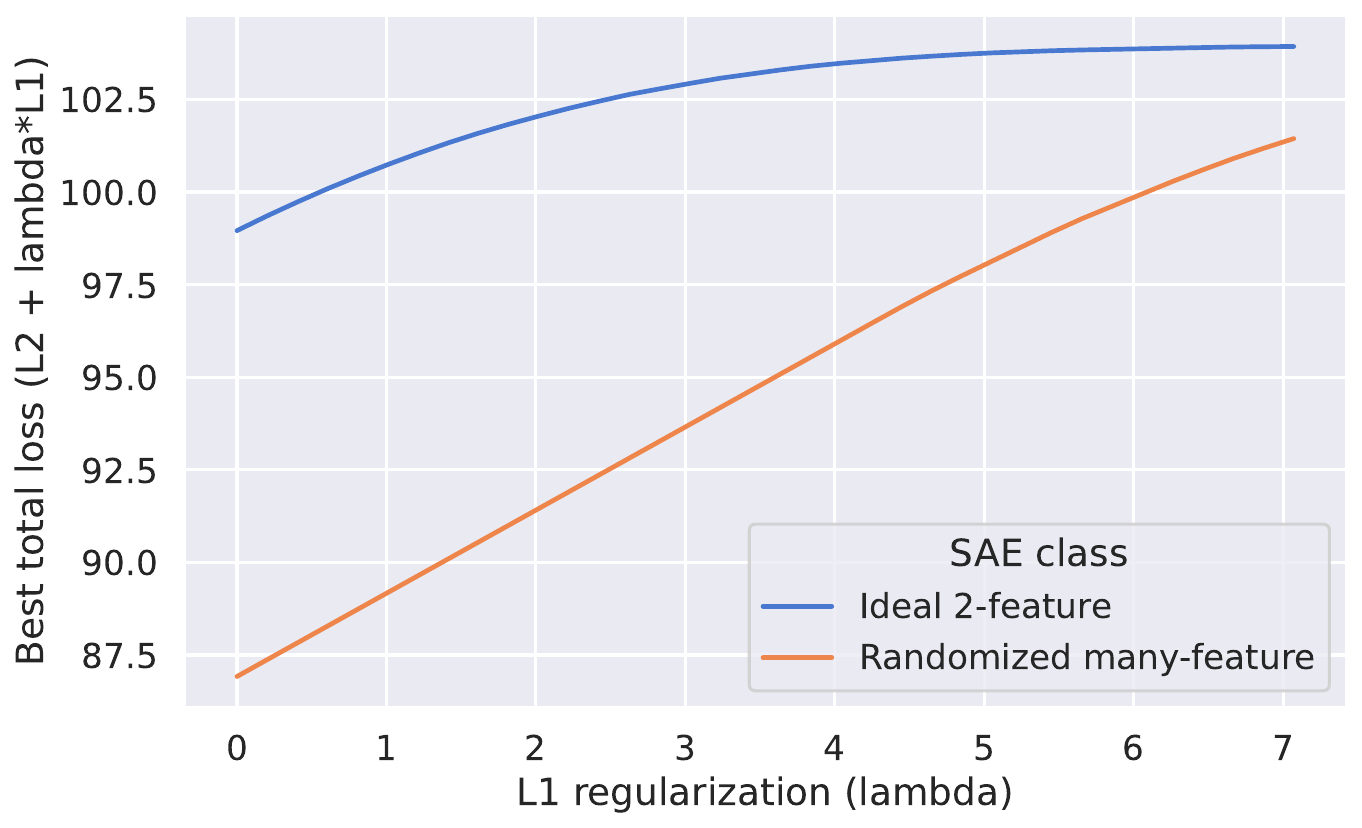}
    \caption{Main experiment for our toy model of feature oversplitting. The
    data distribution is a uniform mixture of two standard multivariate Gaussian
    random variables in $100$ dimensions. \textbf{Blue}: the (approximate) best
    possible total loss (in the infinite data limit) achieved by a class of
    `ideal' SAEs that use two features pointing towards the means of the two
    components of the mixture. \textbf{Orange}: an approximate upper bound on the 
    best possible total loss achieved by an SAE with $m=1,000$ hidden features
    (in the infinite data limit). The $x$-axis is the $\ell_1$ regularization
    coefficient $\lambda$. The cutoff on the $x$-axis is chosen so that the
    idealized solution activates only for a vanishing fraction ($<2\%$) of the
    examples in the mixture.} 
    \label{fig:full-sae-example-prompts}
\end{figure}

Our goal in this section is to demonstrate that there \emph{exist} setups where
an SAE with a large number of hidden features $m\gg 2$, when trained on a uniform
mixture of two isotropic gaussian variables, will prefer a solution with $\gg 2$
features (as opposed to the `ideal' solution with only two features, one per
component of the mixture), for \emph{any} value of the $\ell_1$ regularization
coefficient $\lambda$ and \emph{any} amount of training data.

\textbf{Setup.}
We consider a simple toy model where activations are
distributed according to a uniform mixture $\mathcal{D}_{toy}$ of two isotropic
gaussians $\pm \bm{\mu} + \mathcal{N}\l(0, \mathbf{I}_d\r)$ in $\mathbb{R}^d$
(i.e., we first flip a fair coin to determine the sign of $\bm{\mu}$, and then
sample an extra additive term from $\mathcal{N}(0, \mathbf{I}_d)$). We sample an
i.i.d.\ dataset from this mixture. We used $d=100, \l\|\bm{\mu}\r\|_2=2$ in the
experiments below; this choice guarantees that the two components of the mixture
are separable with high ($>95\%$) probability.

\textbf{The `ideal' solution with 2 features.}
We might hope that, with the right $\ell_1$ regularization, an SAE trained on
such a distribution will discover only two interpretable features, one for each
component in the mixture, with the encoder/decoder vectors aligned with
$\bm{\mu}$ and $-\bm{\mu}$ respectively, and each feature activating for
(approximately) the examples in its corresponding component. 

We find this is the case empirically when we train an SAE with only two hidden
features on this toy distribution. Specifically, there is a range of $\lambda$
values $1\leq \lambda \leq 3$ where the SAE reliably approximately recovers this
ideal solution, with the encoder bias controlling the trade-off between the two
loss terms: when $\lambda$ increases, the encoder bias changes so that fewer
examples in each component of the mixture are activated (and the active examples
have a lower $\ell_1$ loss). Beyond this range, the SAE often fails to activate
any feature on any examples ($\lambda > 3+\varepsilon$), or activates both
features on almost all examples ($\lambda < 1-\varepsilon$).

\textbf{Analyzing the ideal solution across $\ell_1$ coefficients.}
To study the properties of the `ideal' solution analytically, we make the
following assumptions (borne out empirically with 2-feature SAEs) using only
symmetries of the data distribution:
\begin{itemize}
\item the decoder vectors are $\pm \bm{\mu}$, normalized to have
unit $\ell_2$ norm (by symmetry of each component around its mean);
\item the respective encoder vectors are $\pm k\bm{\mu}$ for some $k>0$ (again
by symmetry of each component around its mean);
\item the decoder bias is zero (by symmetry of the mixture around zero);
\item both encoder biases are set to $-\gamma$ for some $\gamma>0$ (again by
symmetry of the mixture around zero).
\end{itemize}
This leaves only two parameters to tune: the encoder bias $\gamma$ and the
encoder scale $k$. We can thus use the following strategy to analytically
approximate the best loss of this class of solutions for a given $\lambda$:
\begin{itemize}
\item approximate the expected $\ell_1$ and $\ell_2$ losses over a fine grid of
values for $\gamma$ and $k$, for a large dataset of samples from the mixture;
\item given a $\lambda$ value, find the point in the grid that minimizes the
total loss $\ell_2 + \lambda \ell_1$.
\end{itemize}
We implemented this using $10^5$ samples, with a grid of 100 values for $\gamma$
in $[0, 5]$ and 20 values for $k$ in $[0, 2]$, over a grid of 100 values of
$\lambda$ in $[0, 20]$.  We verify that the best values chosen for each
$\lambda$ are not on the edges of the grid; the resulting curve of best total
loss values versus $\lambda$ is shown in Figure \ref{fig:full-sae-example-prompts} (blue), cut off at
$\lambda\approx 7$, beyond which the selected SAE activates for $<2\%$ of
examples in the components of the mixture.

\textbf{SAEs with $m\gg 2$ features prefer other solutions even with infinite data.} 
Next, we want to show that with enough features, SAEs will prefer solutions
different from the class of 2-feature solutions described above.
How can we give an \emph{empirical} argument that applies to \emph{any} amount
of training data and any $\lambda$? After all, a \emph{trained} SAE is a
function of the data it is trained on, so no experiment on datasets of bounded
size can establish properties of SAEs trained on arbitrarily large datasets. 

We get around this by defining a class of SAEs that is competitive with the
class of ideal 2-feature SAEs \emph{upfront}, independent of the training
sample, by using a randomized construction that works w.h.p., and then
estimating the expected loss of this SAE in the infinite data limit empirically.
To give evidence of our result for arbitrary $\lambda$, we consider a fine
enough grid of $\lambda$ values, and for each $\lambda$ we construct an SAE that
is competitive with the best ideal 2-feature SAE.

Our randomized SAE construction proceeds as follows:
\begin{itemize}
\item Sample $m$ encoder vectors $W_{enc}\in \mathbb{R}^{m\times d}$ from
$\beta * \mathcal{D}_{toy}$ (i.e. a version of $\mathcal{D}_{toy}$ scaled by
$\beta$) where $\beta>0$ is a hyperparameter that we will tune;
\item Set the decoder vectors $W_{dec}\in \mathbb{R}^{d\times m}$ to be the same
as the encoder vectors $W_{enc}^\top$, but normalized so that each column of
$W_{dec}$ has unit $\ell_2$ norm;
\item Set all encoder biases to $-\gamma$ (where $\gamma>0$ is a hyperparameter
that we will tune), and the decoder bias to $\mathbf{0}\in \mathbb{R}^d$.
\end{itemize}

We used $m=1,000$, a grid of 100 values for $\beta$ in $[0, 0.1]$ to search for
the best $\beta$ for each $\lambda$, and fixed $\gamma\approx 2.35$.

\textbf{Empirical confirmation.} Finally, we actually trained SAEs with many
hidden features on $\mathcal{D}_{toy}$, and observed that these SAEs reliably
learned solutions with many active features across $\lambda$ values.

\subsection{Additional details for Section \ref{section:sae-evaluation}}
\label{app:sae-evaluation-details}

\textbf{Feature weights are mostly in the interval $[0, 1]$.}
Recall that given a reconstruction $\widehat{\mathbf{a}} =
\sum_{i}^{}\mathbf{u}_i$, we defined the feature weight for the $i$-th feature as
\begin{align*}
    \operatorname{weight}\l(i\r) = \mathbf{u}_i^\top \widehat{\mathbf{a}} / \l\|\widehat{\mathbf{a}}\r\|_2^2.
\end{align*}
For our supervised feature dictionaries, we find that $10\%$ of all weights are
negative, and that the average value of all negative weights across all nodes in
the IOI circuit and all three attributes is $-0.037$.  Similarly, for our
task-specific SAE feature dictionaries, even though $31\%$ of all weights are
negative, the average value of all negative weights is $-0.002$. The number and
magnitude of weights higher than $1$ are even smaller.

\textbf{Causal evaluation using interpretability.} 
While the feature descriptions generated through our automatic scoring predict
well when a feature is active, it is still unclear whether they also have an
interpretable causal role, i.e. whether activating or deactivating a certain
feature leads to a change in output logits that would be expected from the
feature's description. To test this, we propose two experiments to judge the
sufficiency and necessity of our interpreted features that involve patching
activations from a counterfactual prompt and calculating the effect on the
model's output:

\begin{itemize}
    \item \textbf{Estimating sufficiency:}
    To estimate sufficiency, we construct SAE activations where we fix features with a test F1-score smaller than a threshold and patch activations of features with a high F1-score from the counterfactual prompt. We then calculate reconstructions of this new SAE activation vector, patch cross-sections, and record whether the model successfully predicts the correct counterfactual name.
    \item \textbf{Estimating necessity:}
    To estimate necessity, we propose a similar experiment where we fix features with a high test F1-score and patch all remaining features. This intervention should not change the output logits if our features are complete.
\end{itemize}

We run this experiment on cross-sections of name mover outputs and repeat this
experiment for different thresholds. We observe that for a threshold F1-score of
0.6, the SAE features are both faithful and complete to a high degree. We
observe that the faithfulness metric significantly decreases for higher
F1-scores of 0.7 and 0.8 but remarkably, we also observe that only fixing
features with a very high F1-score of >0.8 while patching all other features
from the counterfactual prompt is sufficient to keep the model predicting the
base prompt's output.

\subsection{Additional figures}
\label{app:additional-circuit-reconstruction}

\begin{figure}[ht]
    \centering
    \includegraphics[width=\linewidth]{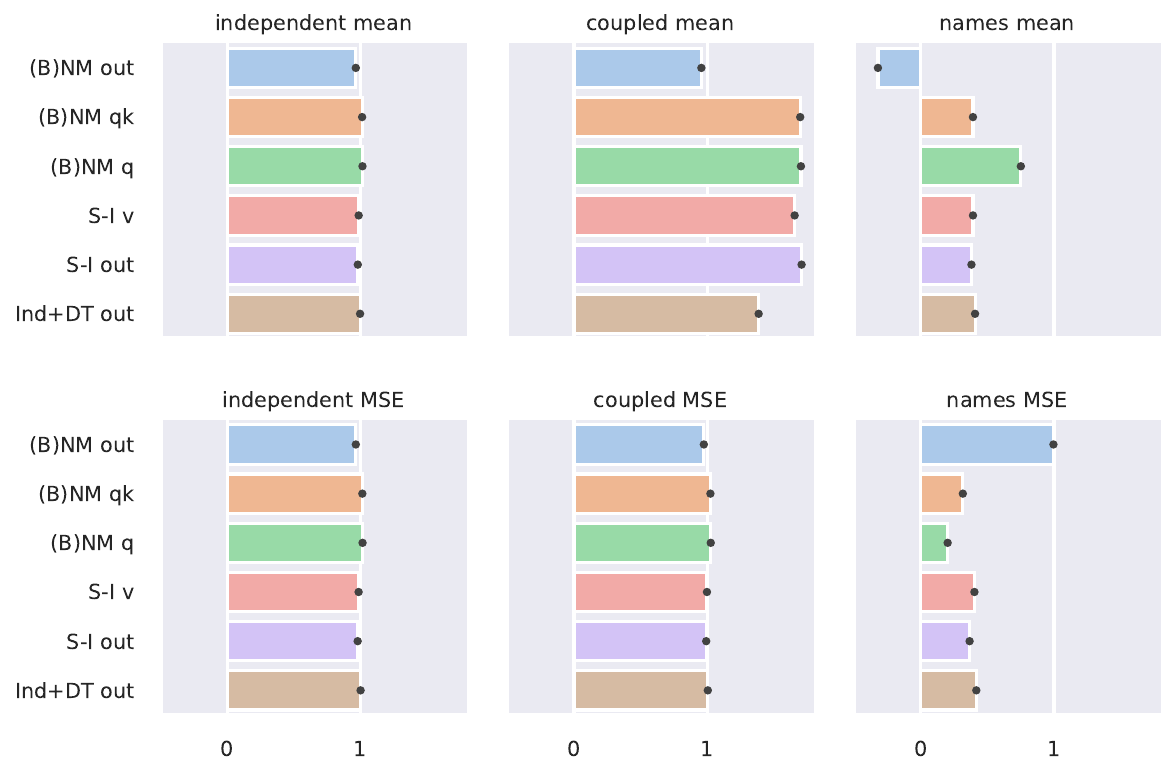}
    \caption{Fraction of recovered logit difference for several different
    methods to compute feature dictionaries, across cross-sections of the
    circuit. For a definition of the `names' parametrization, see Appendix
    \ref{app:possible-decompositions}.}
    \label{fig:reconstructed-logitdiff-recovered-all}
\end{figure}

\begin{figure}[ht]
    \centering
    \begin{subfigure}{.48\textwidth}
        \centering
        \includegraphics[width=\linewidth]{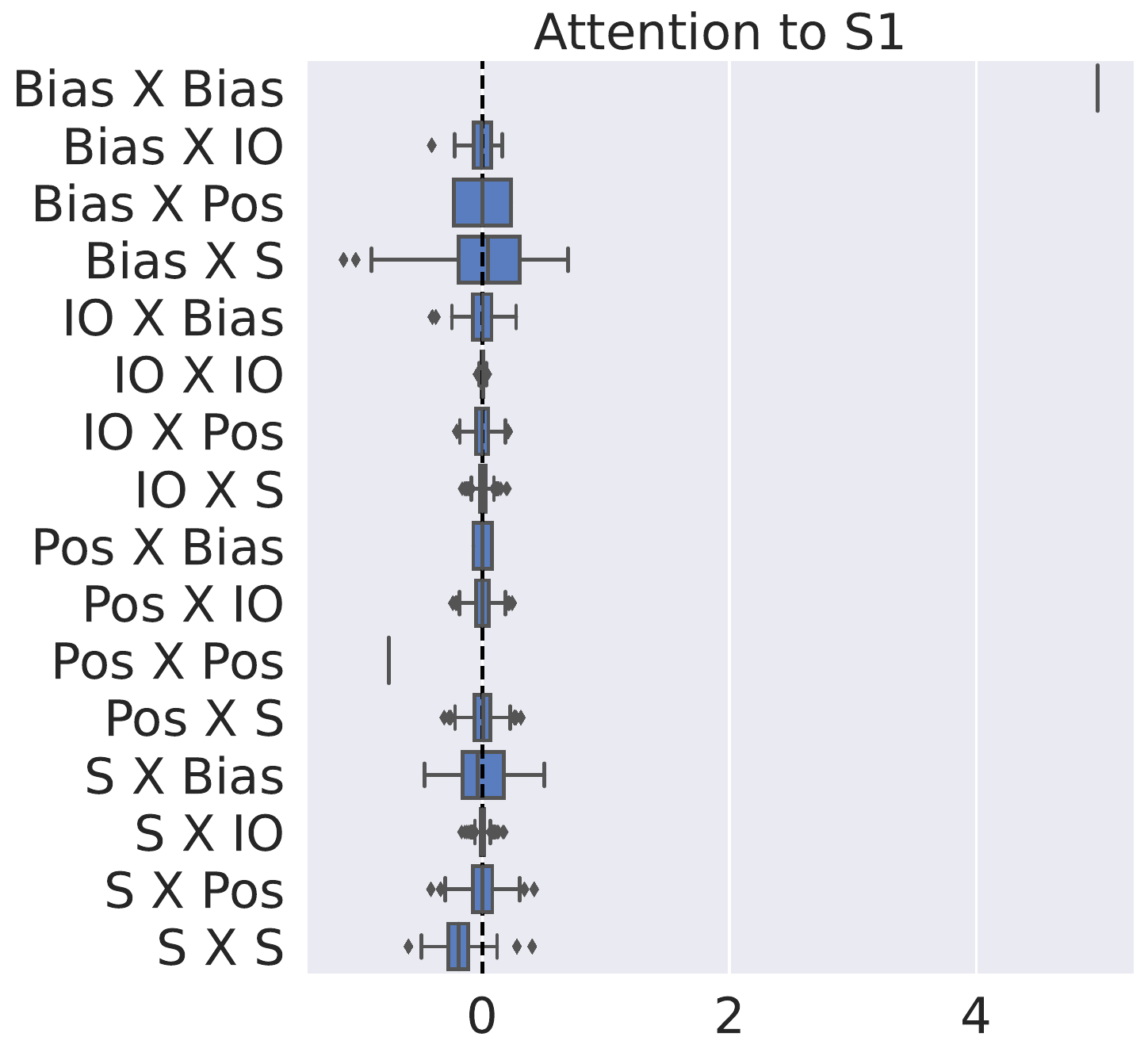}
    \end{subfigure}%
    \begin{subfigure}{.36\textwidth}
        \centering
        \includegraphics[width=\linewidth]{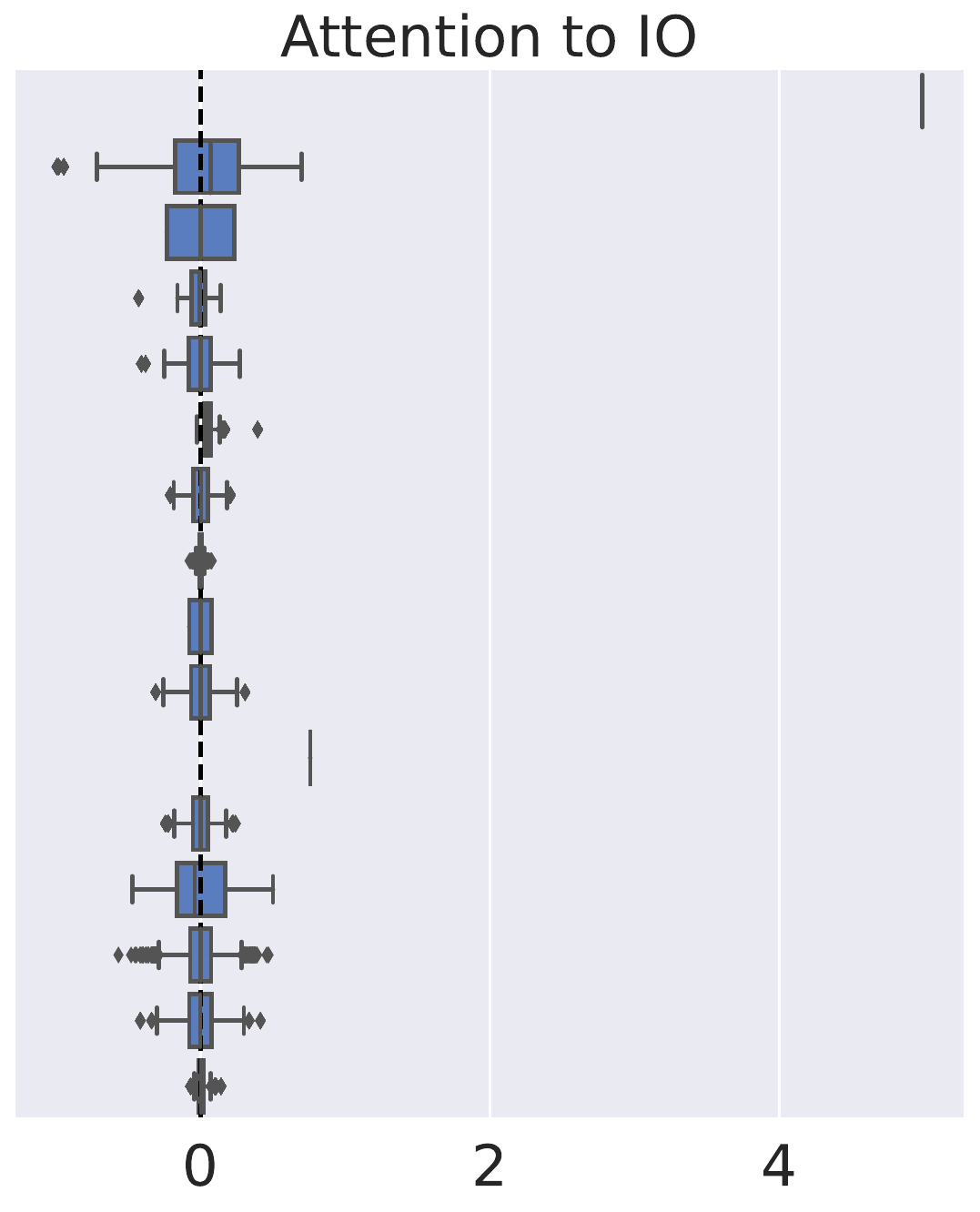}
    \end{subfigure}%
    \caption{Attention score decomposition for the L9H6 name mover (see Figure 
    \ref{fig:attention-decomposition} for explanation). Notice that, in contrast
    with L10H0 attention socres, there is no significant (inhibitory)
    interaction between the IO features in the query and key.}
    \label{fig:attention-decomposition-L9H6}
\end{figure}

\begin{figure}[ht]
    \centering
    \begin{subfigure}{.48\textwidth}
        \centering
        \includegraphics[width=\linewidth]{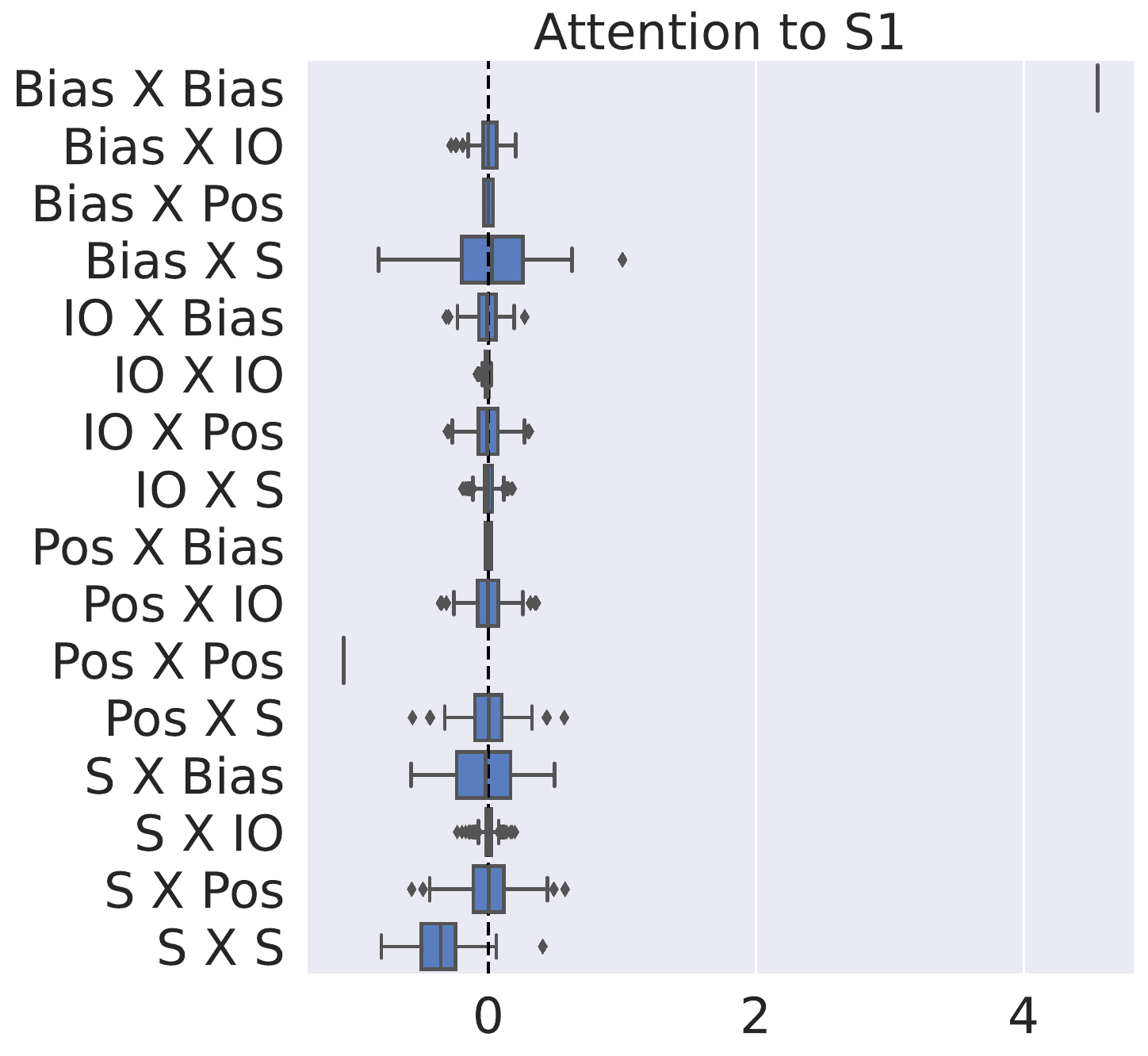}
    \end{subfigure}%
    \begin{subfigure}{.36\textwidth}
        \centering
        \includegraphics[width=\linewidth]{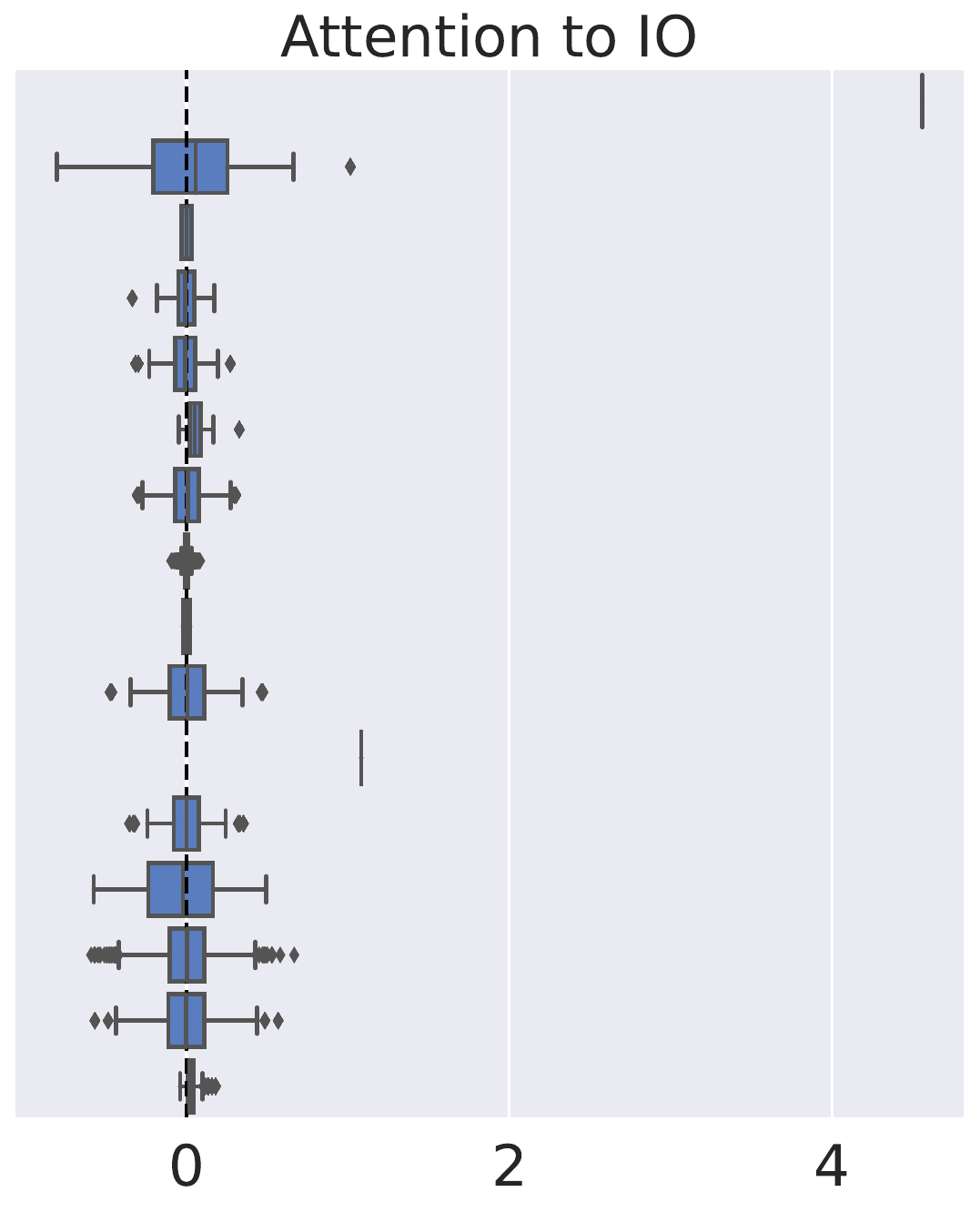}
    \end{subfigure}%
    \caption{Attention score decomposition for the L9H9 name mover (see Figure 
    \ref{fig:attention-decomposition} for explanation).
    }
    \label{fig:attention-decomposition-L9H9}
\end{figure}

\begin{figure}[ht]
    \centering
    \includegraphics[width=\linewidth]{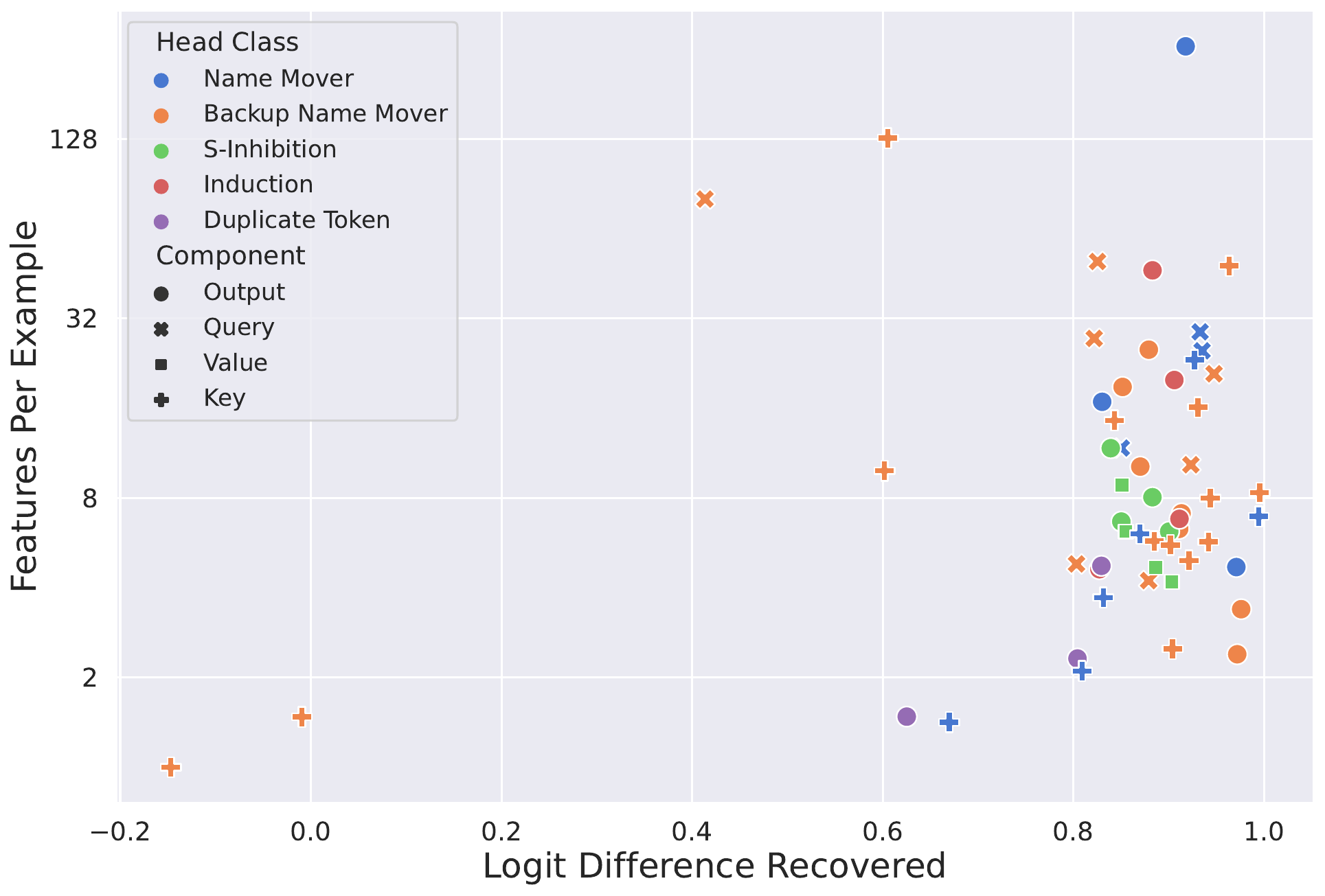}
    \caption{Metrics for our chosen task-specific SAEs for each relevant node in
    the IOI circuit. The x-axis shows the absolute value of the difference in
    logit differences between a clean run of the model, and a run where the
    activations at the given node are replaced by the SAE's reconstructions,
    normalized by the difference between the clean logit difference and the
    logit difference when the node is mean-ablated instead. The y-axis shows the
    average number of features active per prompt.}
    \label{fig:sae-sweep}
\end{figure}

\begin{figure}[ht]
    \centering
    \includegraphics[width=\linewidth]{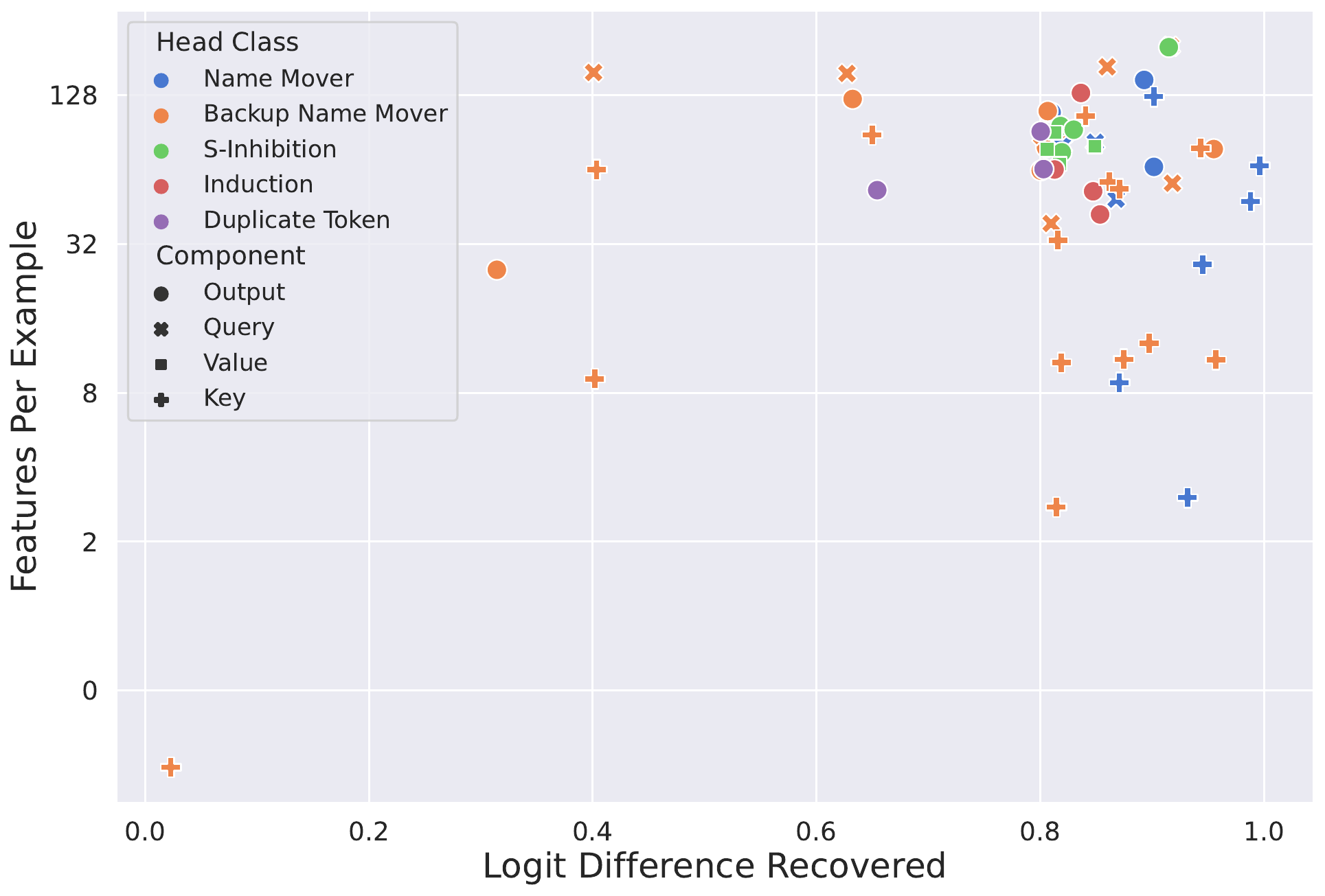}
    \caption{Counterpart to Figure \ref{fig:sae-sweep} where the decoder vectors are frozen during SAE training.}
    \label{fig:sae-sweep-freeze-decoder}
\end{figure}

\begin{figure}[ht]
    \centering
    \includegraphics[width=\linewidth]{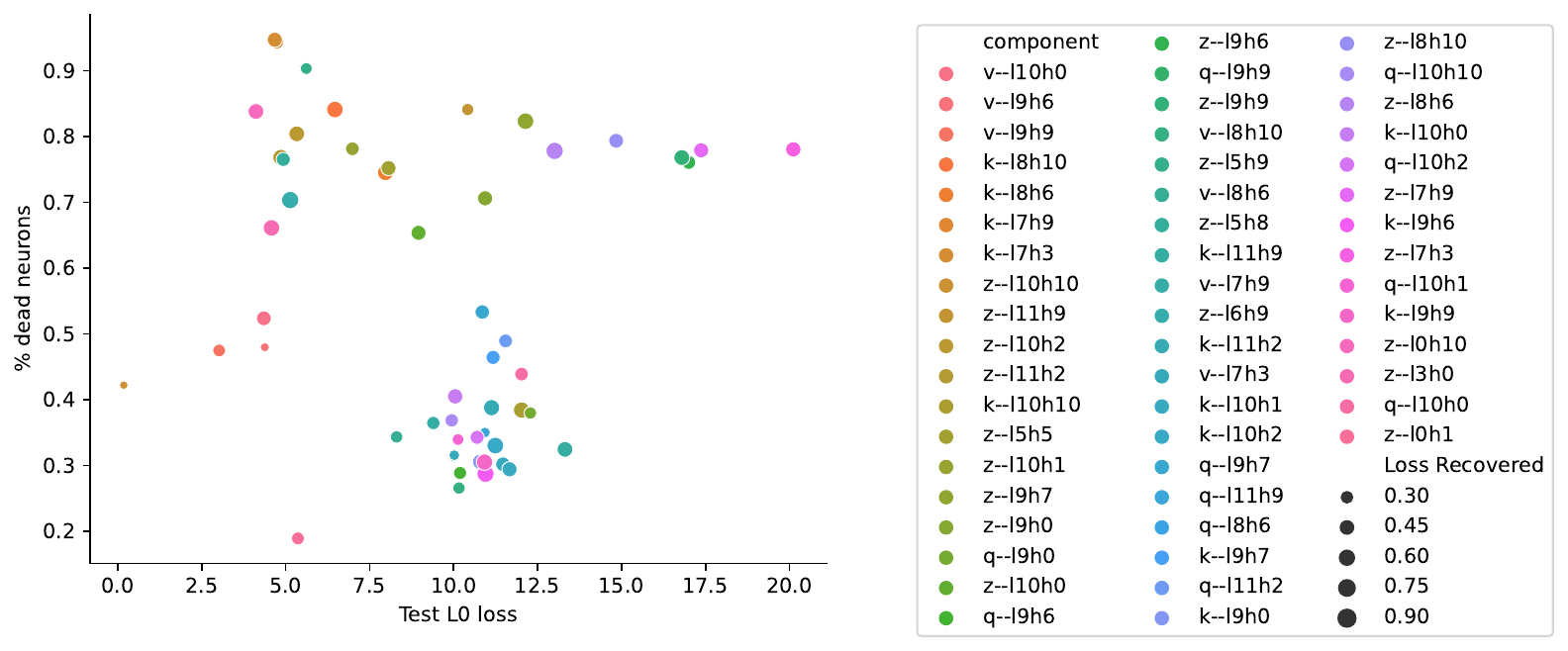}
    \caption{$\ell_0$ loss versus fraction of dead neurons for our full-distribution SAEs.}
    \label{fig:l0-vs-dead-neurons}
\end{figure}

\begin{figure}[ht]
    \centering
    \includegraphics[width=\linewidth]{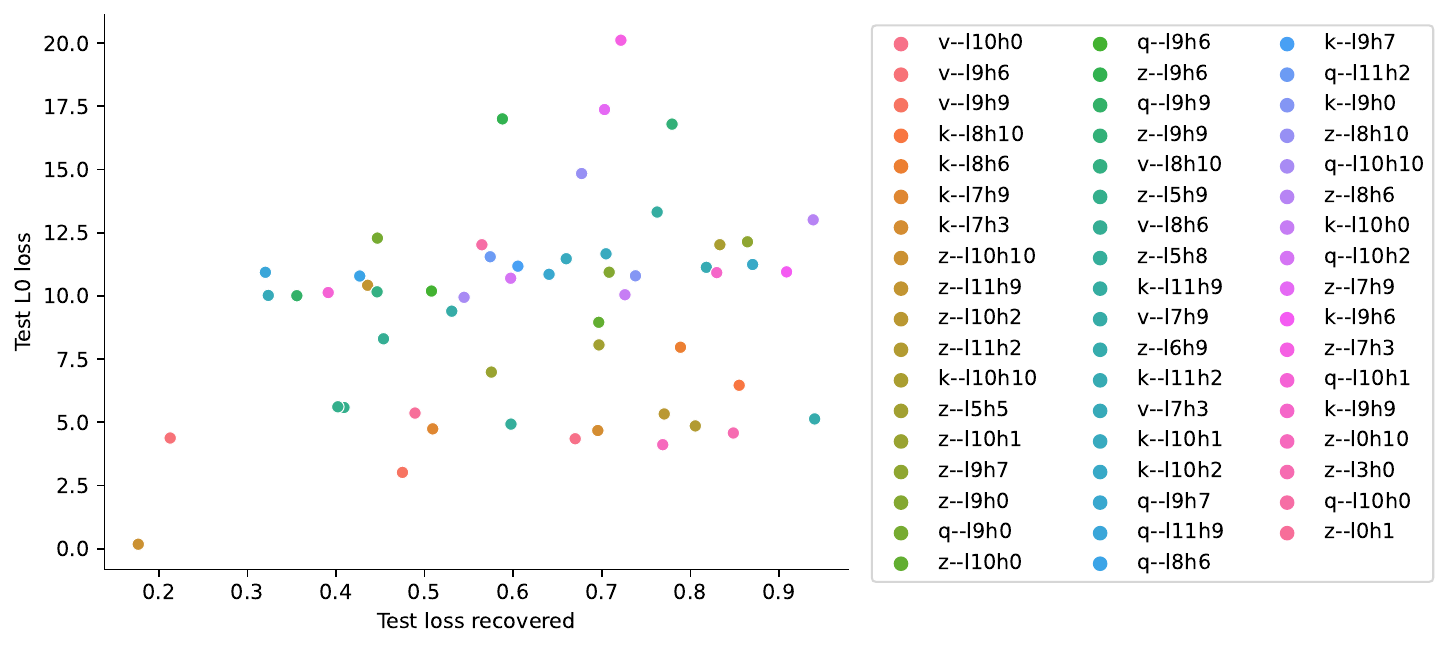}
    \caption{$\ell_0$ loss versus loss recovered (against a mean ablation) for
    our full-distribution SAEs.}
    \label{fig:l0-vs-loss-recovered}
\end{figure}

\begin{figure}[ht]
    \centering
    \includegraphics[width=\linewidth]{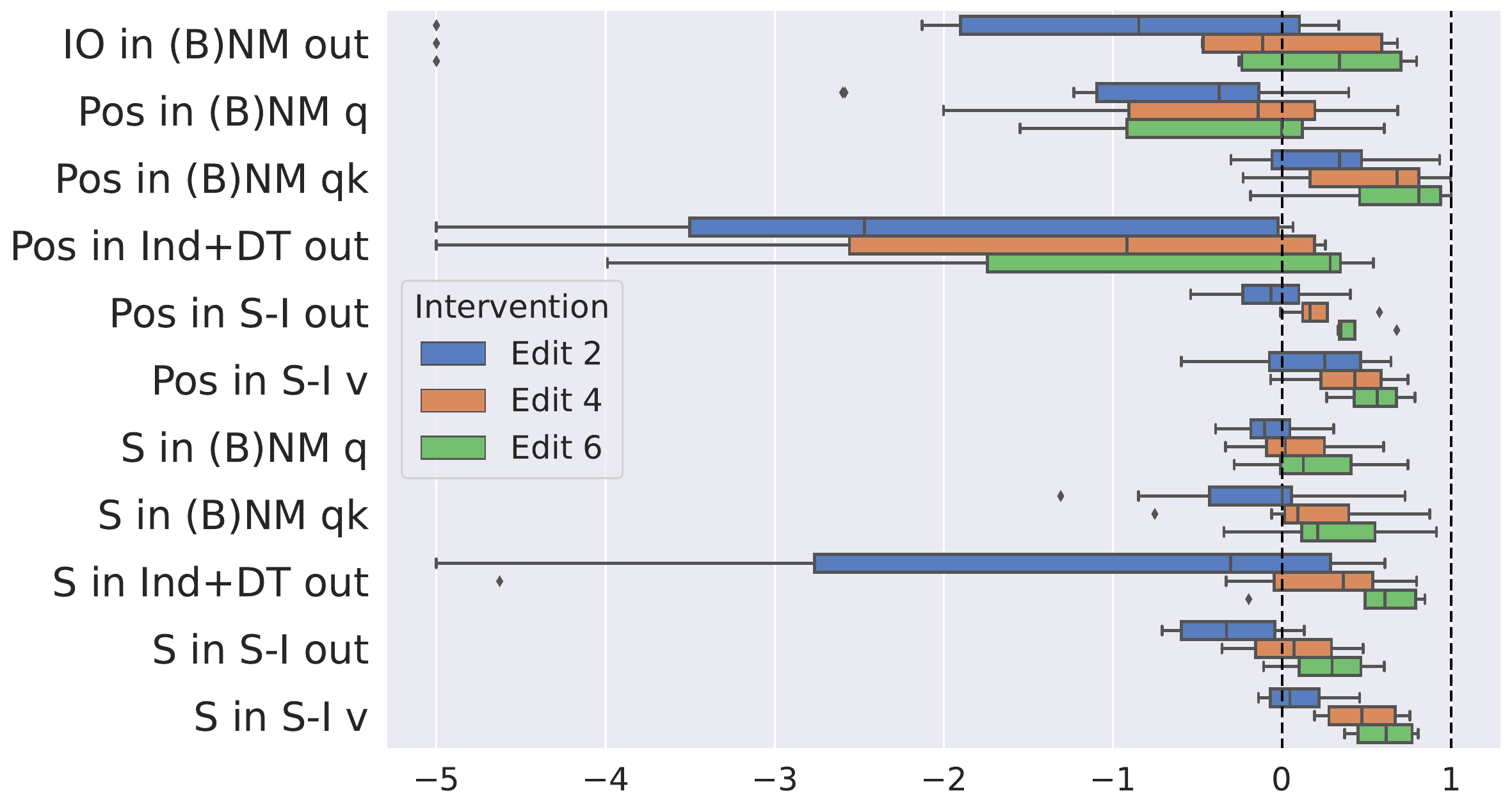}
    \caption{Distribution of the average feature weight removed by our
    interpretation-agnostic edits using SAE feature dictionaries, over all
    locations in a given cross-section, normalized by the corresponding weight
    for the supervised feature dictionaries. Weight removed values are
    transformed linearly so that a value of $0$ indicates that the weight
    removed by the edit equals the weight removed by the corresponding `ground
    truth' supervised edit; and a value of $1$ indicates that the edit removed a
    total weight of $1$, meaning that the edit essentially overwrites all SAE
    features present in the activation. Negative values are clipped at $-5$ to preserve readability.}
    \label{fig:removed-weight-sae}
\end{figure}

\begin{figure}[ht]
    \centering
    \includegraphics[width=\linewidth]{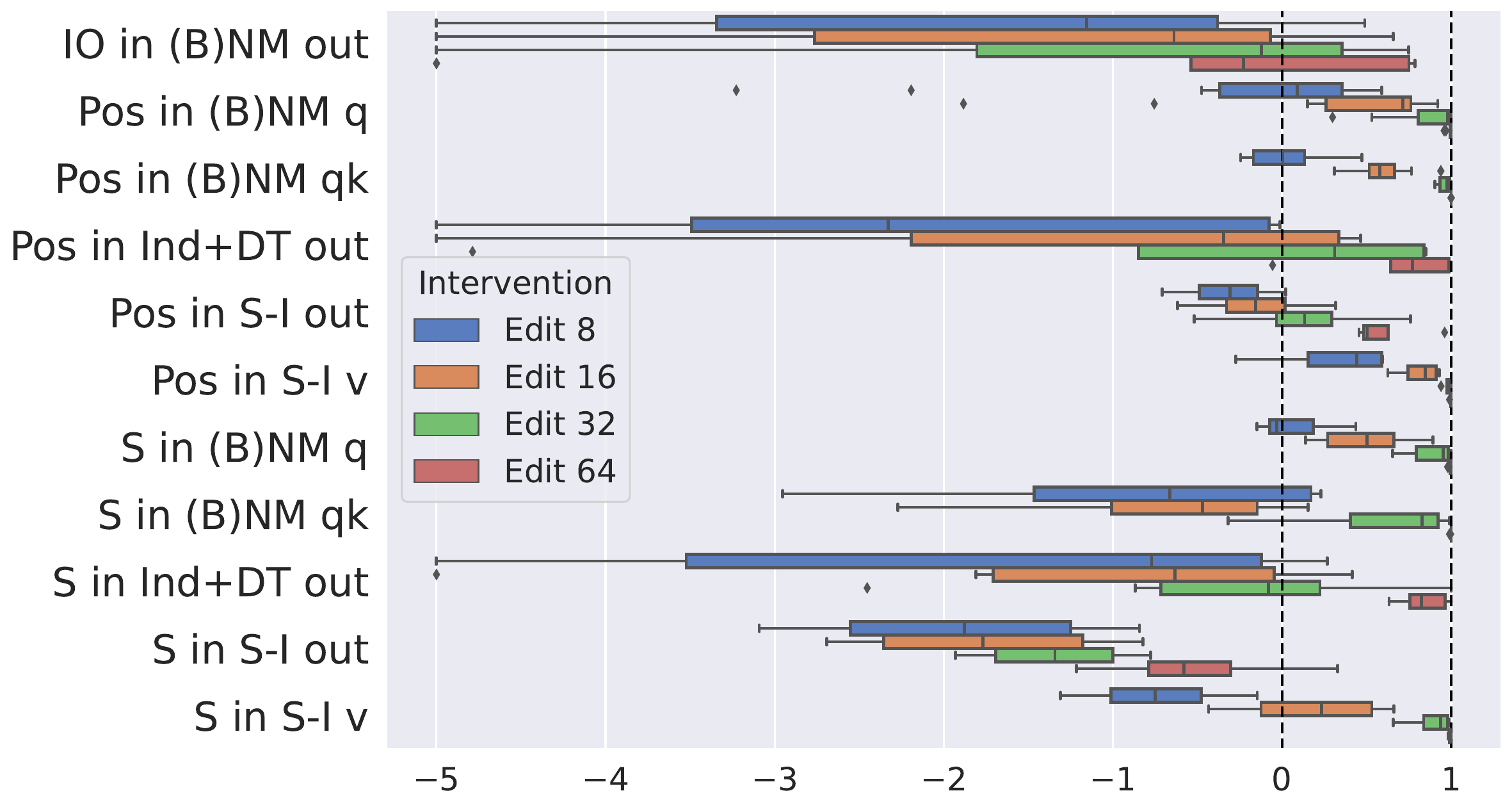}
    \caption{Counterpart of Figure \ref{fig:removed-weight-sae} for full-distribution SAEs, when editing using features with high F1 score for the attribute}
    \label{fig:removed-weight-sae-full}
\end{figure}

\begin{figure}[ht]
    \centering
    \includegraphics[width=\linewidth]{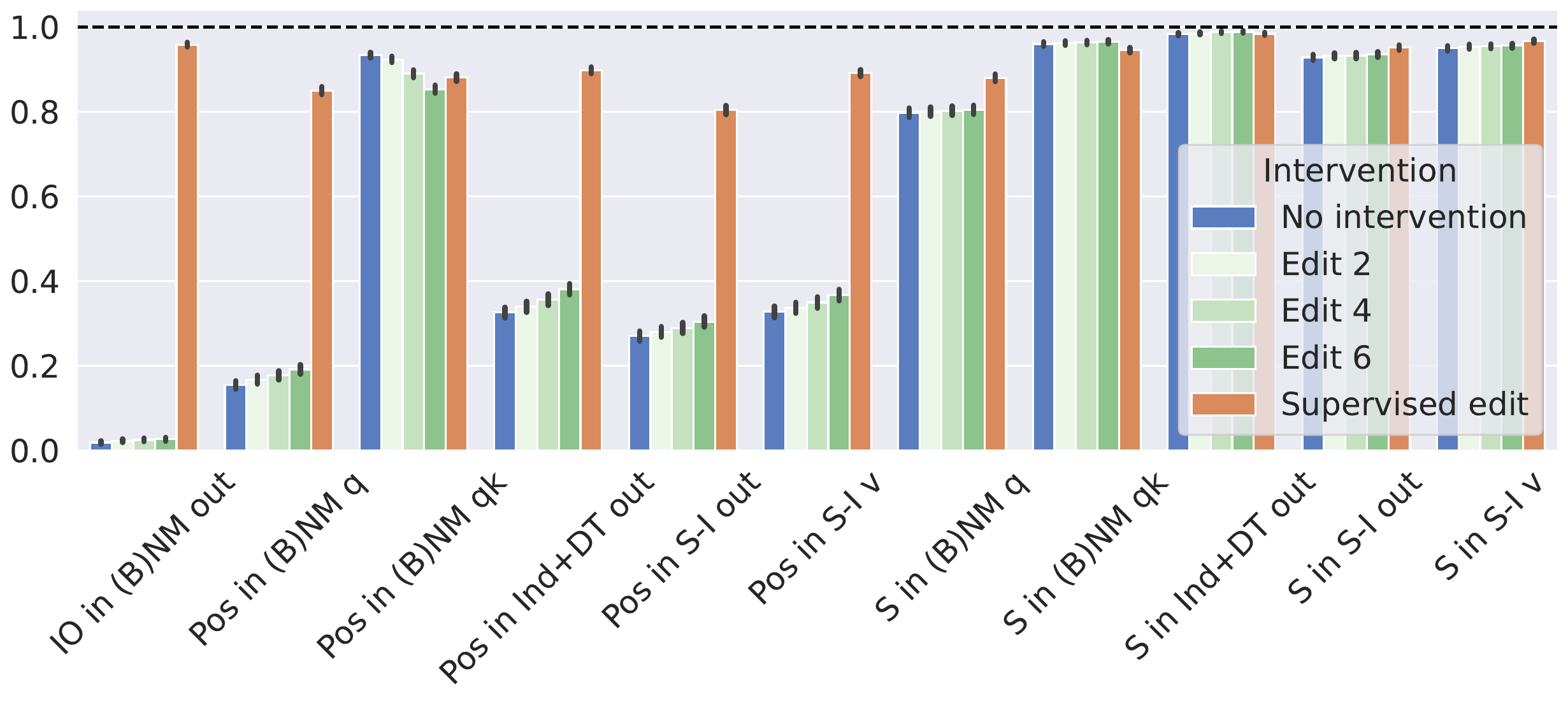}
    \caption{Counterpart to Figure \ref{fig:edit-accuracy}, where the
    decoder vectors are frozen during SAE training.}
    \label{fig:agnostic-editing-accuracy-freeze-decoder}
\end{figure}

\begin{figure}[ht]
    \centering
    \includegraphics[width=\linewidth]{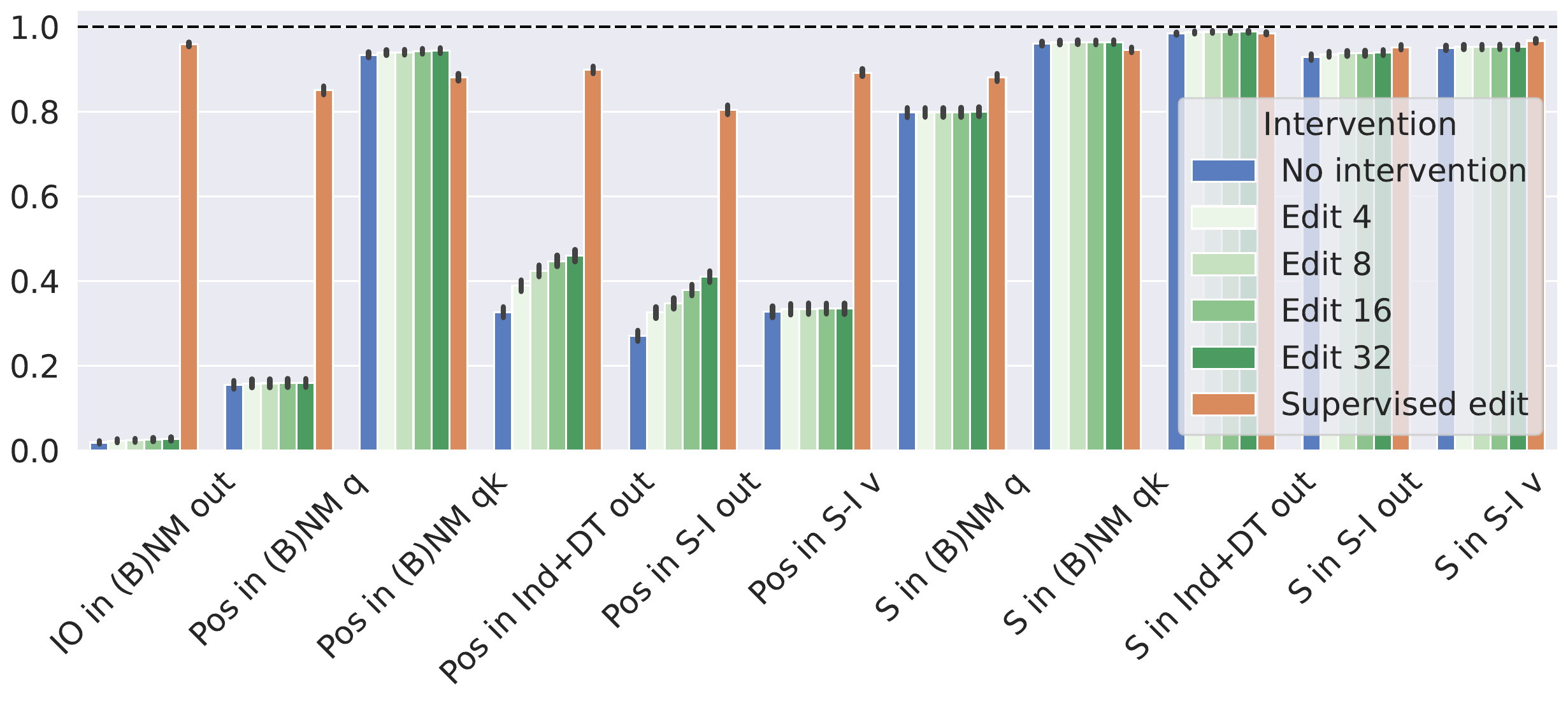}
    \caption{Counterpart to Figure \ref{fig:edit-accuracy}, where we use
    full-distribution SAEs instead. Here, we need to change a much higher number of features in order to have a noticeable effect (and sometimes editing even 32 features fails)}
    \label{fig:agnostic-editing-accuracy-webtext}
\end{figure}

\begin{figure}[ht]
    \centering
    \includegraphics[width=\textwidth]{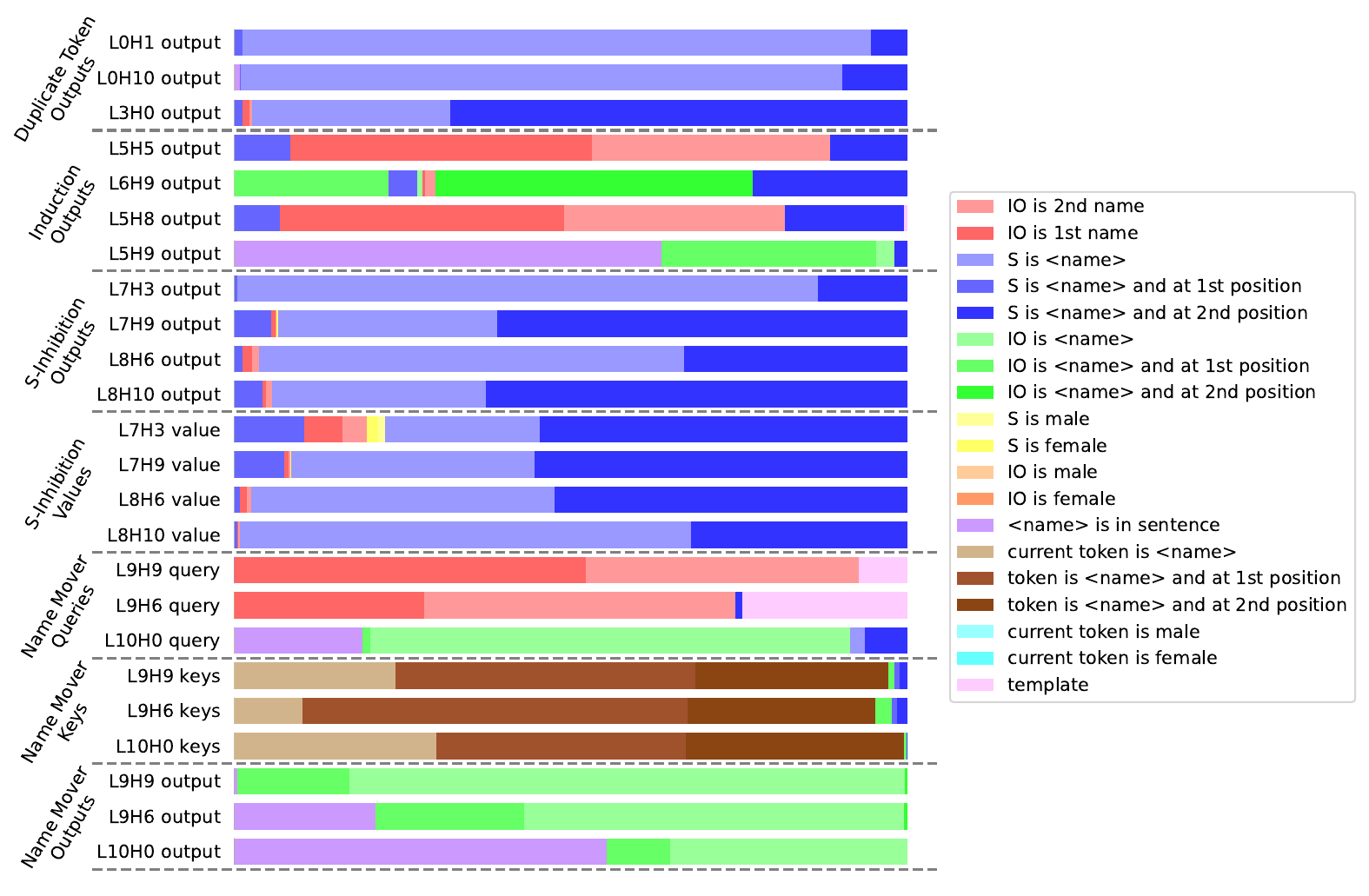}
    \caption{Interpreting the features learned by the task SAEs. For each node in the main IOI circuit (without backup/negative name movers), we show the distribution of the features which have an explanation with $F_1$ score above a threshold. The SAE chosen at each node is the one with the most interpretable features out of all SAEs trained on this node during our hyperparameter sweep.}
    \label{fig:task-sae-interp-most-interp}
\end{figure}

\begin{figure}[ht]
    \centering
    \includegraphics[width=\textwidth]{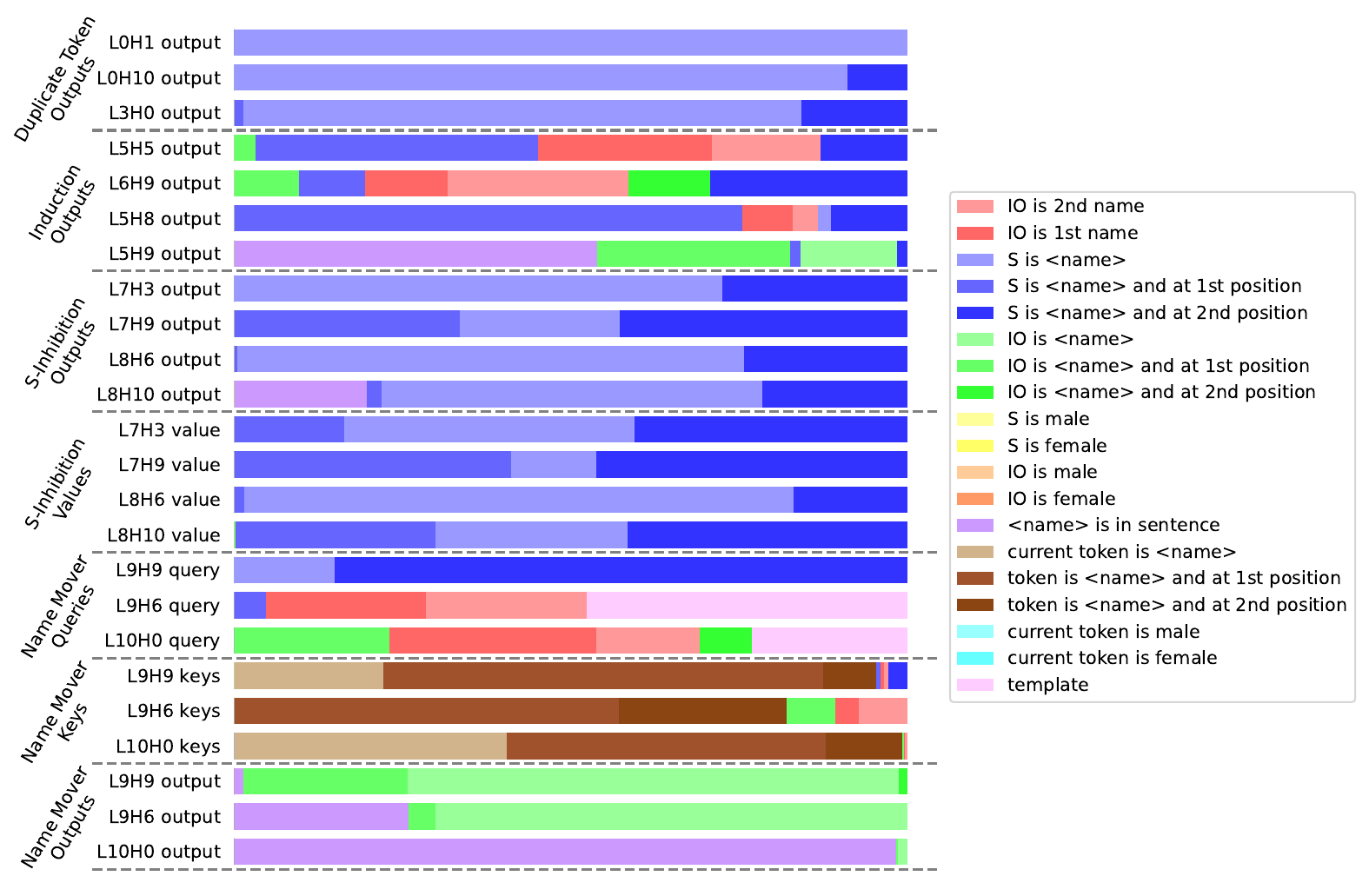}
    \caption{Interpreting the features learned by the task SAEs. This is a counterpart to Figure \ref{fig:task-sae-interp-most-interp} for the SAEs chosen based on the $\ell_0$ and logit difference recovered metrics.}
    \label{fig:task-sae-interp-best-metrics}
\end{figure}

\begin{figure}[ht]
    \centering
    \includegraphics[width=\textwidth]{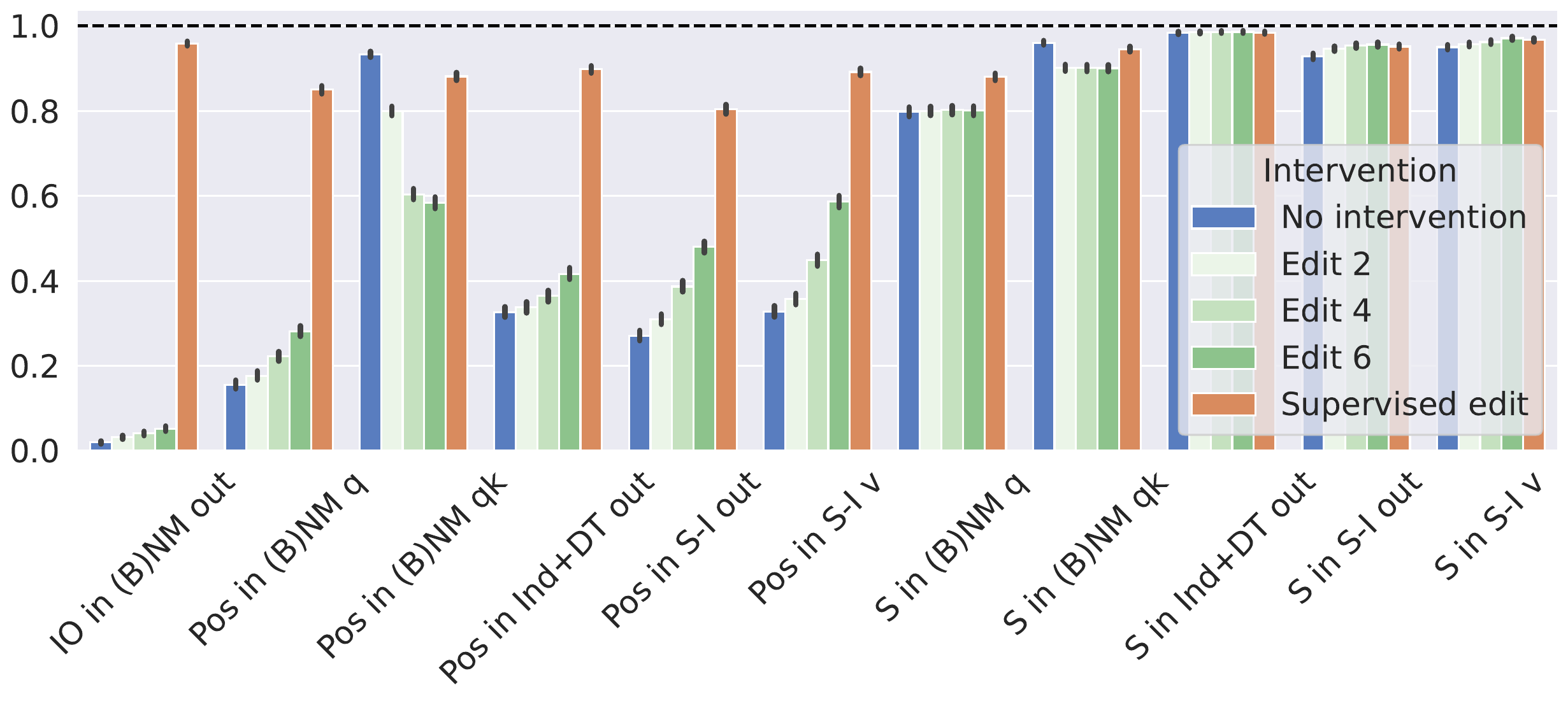}
    \caption{Interpretation-aware sparse control, using task SAE features with
    the highest $F_1$ score with respect to the given attribute.}
    \label{fig:editing-accuracy-sae-interp}
\end{figure}

\begin{figure}[ht]
    \centering
    \includegraphics[width=\textwidth]{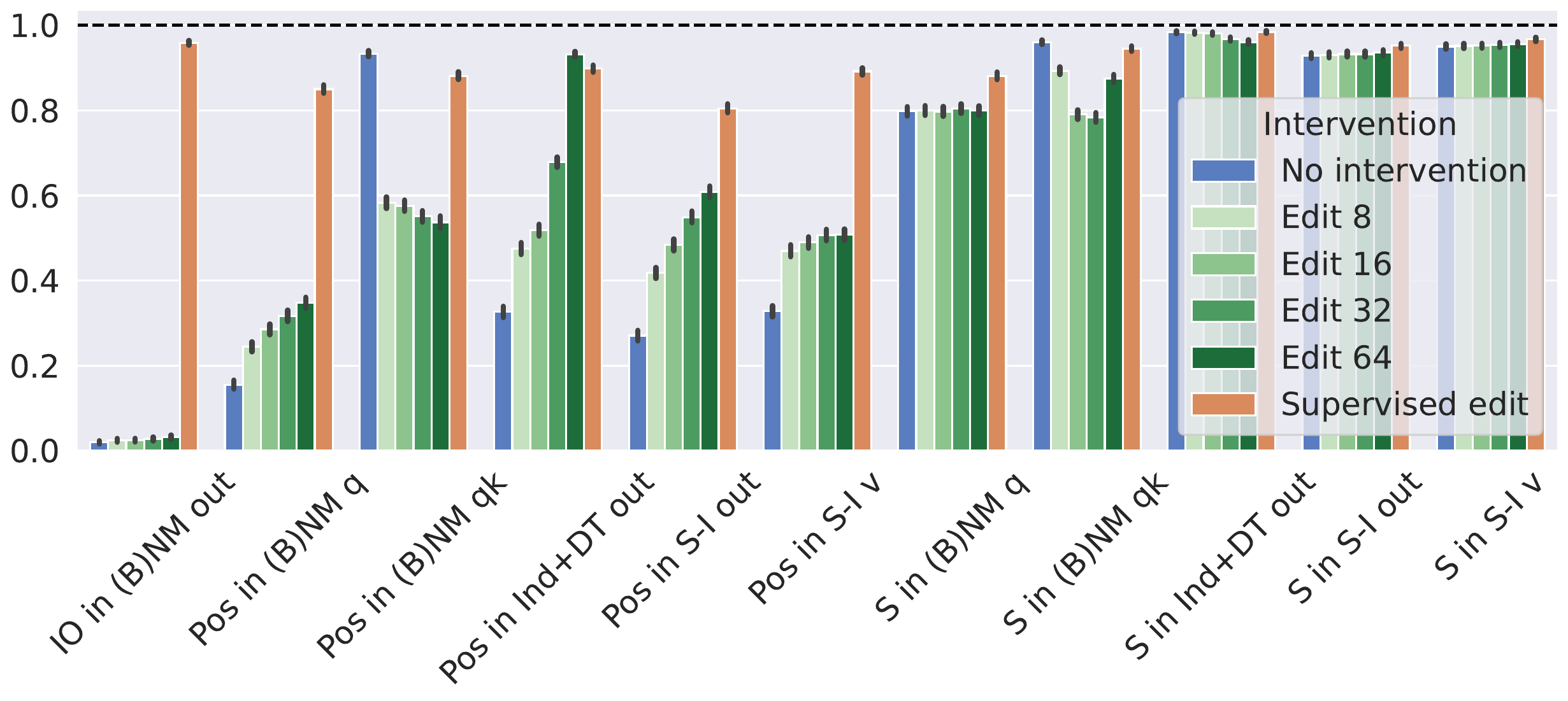}
    \caption{Counterpart of Figure \ref{fig:editing-accuracy-sae-interp} with full-distribution SAEs.}
    \label{fig:editing-accuracy-sae-interp-webtext}
\end{figure}

\begin{figure}[ht]
    \centering
    \includegraphics[width=\textwidth]{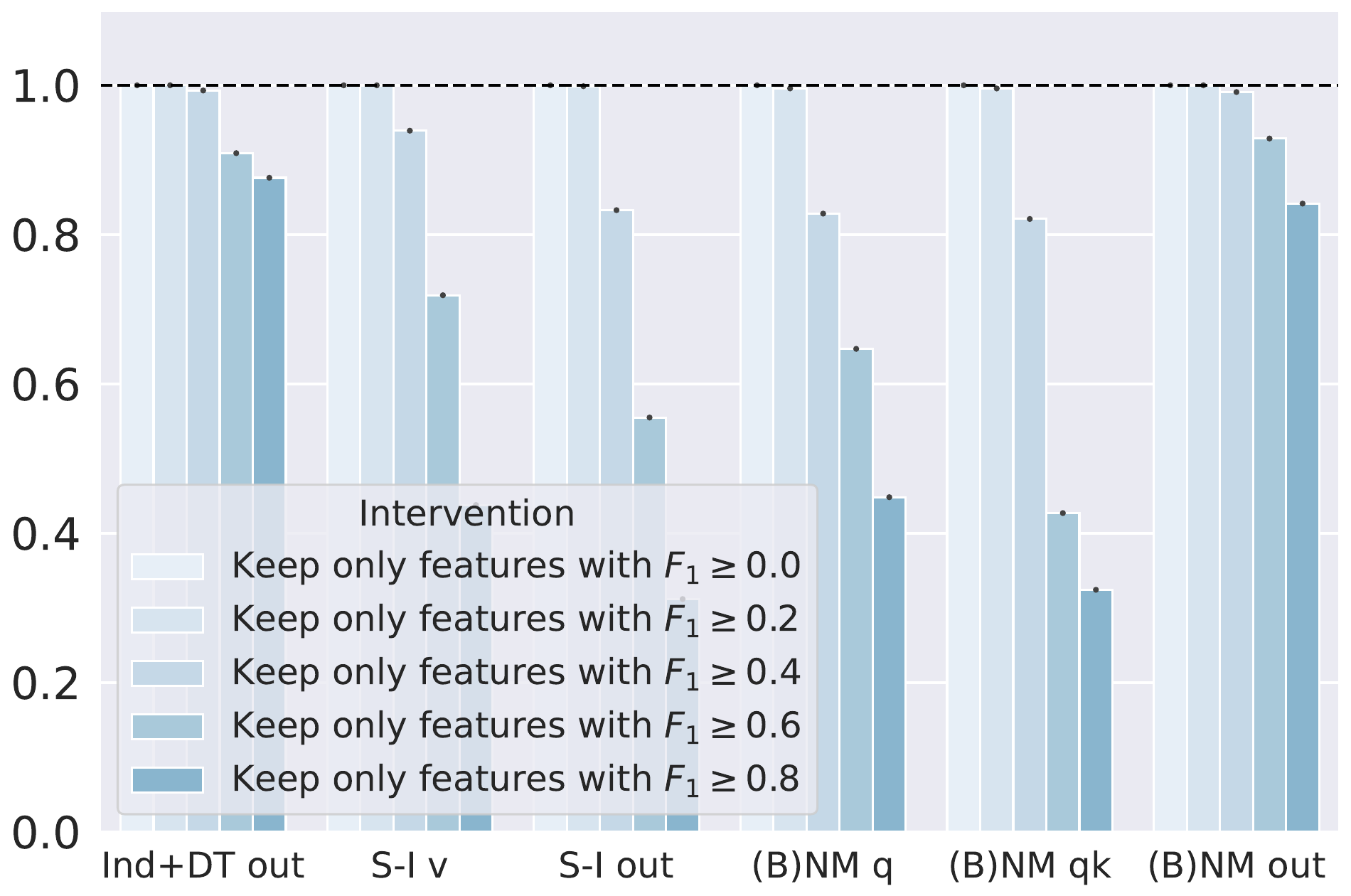}
    \caption{Measuring the \textbf{sufficiency} of interpretable features for task SAEs:
    effect of subtracting features with the lowest $F_1$ score
    from activations on logit difference. A value of 1 is best.}
    \label{fig:interp-sufficiency}
\end{figure}

\begin{figure}[ht]
    \centering
    \includegraphics[width=\textwidth]{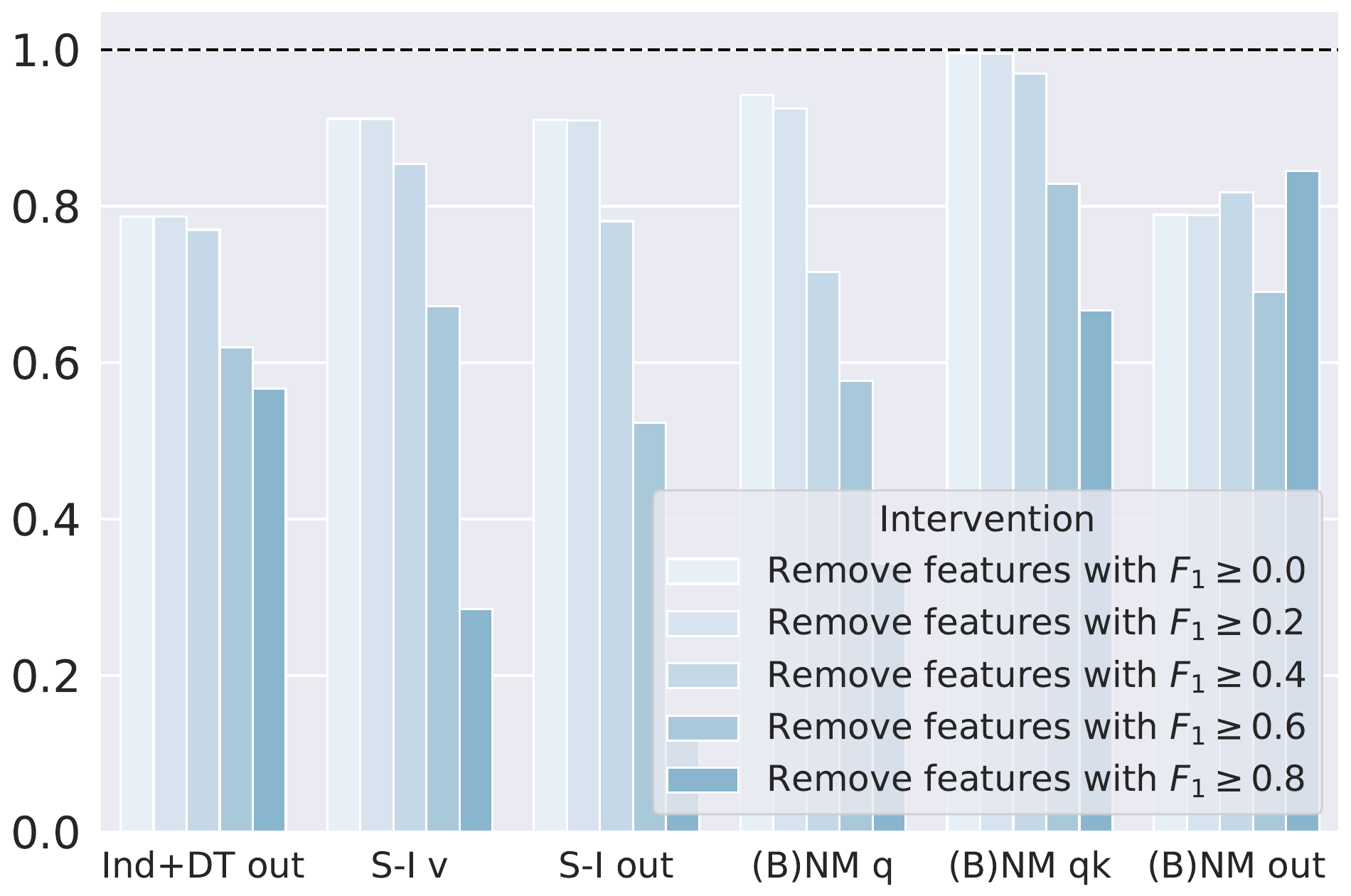}
    \caption{Measuring the \textbf{necessity} of interpretable features for task SAEs:
    effect of removing features with the highest $F_1$ score from
    activations on the logit difference. Values are rescaled linearly so that a value of
    1 corresponds to perfect recovery of the logit difference achieved by mean
    ablation (i.e., ideal intervention removing all features). A value of 1 is best.}
    \label{fig:interp-necessity}
\end{figure}

\end{document}